\documentclass{article}
\usepackage{times,natbib}
\setcitestyle{authoryear,round,citesep={;},aysep={,},yysep={;}}

\usepackage{color}
\usepackage{graphicx}
\graphicspath{{../}}
\usepackage{subcaption}
\usepackage{amsthm}
\usepackage{amsbsy}
\usepackage{booktabs}
\usepackage{amsmath,amssymb,amsfonts,mathtools}
\usepackage{algorithm}
\usepackage{enumerate}
\usepackage{algpseudocode}
\usepackage{placeins}
\usepackage{bm}
\usepackage{fullpage}
\newcommand{\mA}{\mathbf{A}}
\newcommand{\mB}{\mathbf{B}}
\newcommand{\mC}{\mathbf{C}}
\newcommand{\mD}{\mathbf{D}}

\newcommand{\mF}{\mathbf{F}}
\newcommand{\mG}{\mathbf{G}}
\newcommand{\mI}{\mathbf{I}}
\newcommand{\mL}{\mathbf{L}}
\newcommand{\mM}{\mathbf{M}}
\newcommand{\mN}{\mathbf{N}}
\newcommand{\mP}{\mathbf{P}}
\newcommand{\mQ}{\mathbf{Q}}
\newcommand{\mR}{\mathbf{R}}
\newcommand{\mS}{\mathbf{S}}

\newcommand{\mU}{\mathbf{U}}
\newcommand{\mV}{\mathbf{V}}
\newcommand{\mW}{\mathbf{W}}
\newcommand{\mX}{\mathbf{X}}
\newcommand{\mY}{\mathbf{Y}}
\newcommand{\mZ}{\mathbf{Z}}
\newcommand{\mSigma}{\bm{\Sigma}}

\newcommand{\vg}{\bm{g}}
\newcommand{\vr}{\bm{r}}
\newcommand{\vc}{\bm{c}}
\newcommand{\va}{\bm{a}}
\newcommand{\vb}{\bm{b}}
\newcommand{\vv}{\bm{v}}
\newcommand{\vu}{\bm{u}}

\newcommand{\vm}{\bm{m}}
\newcommand{\vw}{\bm{w}}

\newcommand{\vone}{\mathbf{1}}
\newcommand{\R}{\mathbb{R}}
\newcommand{\St}{\mathrm{St}}
\newcommand{\Polar}{\mathrm{Polar}}

\newcommand{\MN}{\mathrm{MN}}
\newcommand{\Normal}{\mathcal{N}}
\newcommand{\IW}{\mathrm{IW}}
\newcommand{\diag}{\operatorname{diag}}
\newcommand{\tr}{\operatorname{tr}}
\renewcommand{\vec}{\operatorname{vec}}

\newcommand{\KL}{\operatorname{KL}}

\providecommand{\eg}{e.g.,\ }

\DeclareMathOperator{\argmininn}{argmin} 
\DeclareMathOperator{\argmaxinn}{argmax} 
\newcommand{\argmin}[1]{ \underset{#1}{\argmininn} \;}
\newcommand{\argmax}[1]{ \underset{#1}{\argmaxinn}\;}

\newcommand{\E}[1]{\mathbb{E}\left[#1\right] }

\usepackage{hyperref}
\usepackage{url}
\usepackage{mdframed} 
\usepackage{thmtools}

\definecolor{theoremblue}{RGB}{245,249,255}
\definecolor{theoremlineblue}{RGB}{199,218,242}
\declaretheoremstyle[
headfont=\normalfont\bfseries,
notefont=\mdseries, notebraces={(}{)},
bodyfont=\normalfont,
postheadspace=0.5em,
spaceabove=3pt,
spacebelow=-3pt,
mdframed={
  skipabove=8pt,
  skipbelow=8pt,
  roundcorner=3pt,
  hidealllines=true,
  backgroundcolor={theoremblue},
  innerleftmargin=8pt,
  innerrightmargin=8pt,
  innertopmargin=6pt,
  innerbottommargin=6pt}
]{shaded}

\declaretheorem[style=shaded,within=section]{definition}

\declaretheorem[style=shaded,sibling=definition]{proposition}
\declaretheorem[style=shaded,numbered=no,name=Proposition]{proposition*}

\declaretheorem[style=shaded,sibling=definition]{lemma}
\declaretheorem[style=shaded,sibling=definition]{remark}

\usepackage[nameinlink]{cleveref}
\makeatletter
\providecommand{\theHALG@line}{}
\renewcommand{\theHALG@line}{\thealgorithm.\arabic{ALG@line}}
\makeatother

\hypersetup{colorlinks, citecolor=blue, filecolor=black, urlcolor=blue}

\title{Stochastic optimization methods through statistical whitening}
\author{Tetiana Parshakova$^{*}$ \and Robert M. Gower$^{*}$ \and Paulo Orenstein}

\begin{document}

\maketitle
{\renewcommand{\thefootnote}{\fnsymbol{footnote}}%
\footnotetext[1]{Joint first authors.}%
}

\begin{abstract}
We provide new statistical insights into designing stochastic optimization methods.
 By modeling the gradient via a normal distribution, we argue that the stochastic gradient should be whitened using its covariance. 
We give an online method for estimating the mean and variance of the gradients based on regularized likelihood or, equivalently,
through a maximum a posteriori estimate. We then show how to use structured approximations to the covariance 
within this same framework, which includes diagonal and Kronecker-factored approximations. The diagonal approximation results in a
 method that is very closely related to Adam and performs as well in practice. The Kronecker approximation results in
 the new MN method, which is a variant
of Shampoo and KL-Shampoo. 
We also recover Muon by applying this same statistical framework to the gradient direction rather than the full gradient. 
By combining this with a whitening step, we arrive at MN-Muon, a novel preconditioned version of Muon. 
Our experiments on GPT-2-type models at Chinchilla scale indicate that MND, the diagonal implementation of MN, 
outperforms both the Distributed Shampoo and the KL-Shampoo implementations in terms of stability, final validation loss and time taken.
MND-Muon, the diagonal implementation of MN-Muon, improves on Muon in 21 of 24 paired GPT-small runs and achieves a lower best validation loss on GPT-medium,
 where it remains stable over the six tested learning rates. However, we found that NorMuon was comparable to MND-Muon and achieved a slightly 
lower validation loss on GPT-medium and at GPT-small batch 512 on FineWeb.
\end{abstract}

\section{Introduction}

Here we present a new viewpoint for modeling stochastic pretraining methods based on whitening. 
Stochastic gradients are noisy estimates of the full-batch gradient, with some components 
having high variance and others low variance. We need to take this variance into account; 
otherwise, we might accidentally take a large step in a high-variance direction, leading 
the method astray. 
% By modelling the gradient with a normal distribution, I will argue that 
% gradient estimates should be whitened using their covariance.

To explain this whitening, let $\ell(\vw)$ be a loss we wish to minimize with parameters $\vw\in \R^d$, and
let $\vg_t \in \R^d$ be a stochastic gradient evaluated at a given $\vw_t\in\R^d$. Thus  $\E{\vg_t\mid \vw_t}= \nabla \ell(\vw_t)$.
With a sampled stochastic gradient, we may update our parameters using 
$$\vw_{t+1} = \vw_t -\eta \vg_t,$$
where $\eta>0$ is a step size. 
Suppose now that we fit a normal model to our stochastic gradients, and assume $\vg_t~\sim~\Normal(\vm_t,\mSigma_t).$ Instead of using $\vg_t$ as our update direction, we could use an
estimate $\hat{\vm}_t$ of the mean and update according to
$$\vw_{t+1} = \vw_t -\eta \hat{\vm}_t.$$
This approach corresponds to using the momentum algorithm~\citep{schaipp2025tracking}. But this momentum update ignores the variance in the gradients. Some of the coordinates of the gradient may have a very high variance, in which case we should be cautious and take a smaller step in these coordinates. For coordinates that have low variance, we can safely take larger steps. We can implement this idea by using the covariance to \emph{whiten} our estimate of the mean. 
That is, the whitened gradient $\mSigma_t^{-1/2}\vg_t\sim\Normal(\mSigma_t^{-1/2}\vm_t,\mI)$ has independent coordinates, each with variance one.
% That is, note that $\mSigma_t^{-1/2}\vg_t\sim \cN(\mSigma_t^{-1/2}m_t, \mI)$ has independent coordinates each with an equal variance of one. 
Taking a step using this whitened stochastic gradient corresponds to
\begin{equation}\label{eq:whiten-steps-gt}
     \vw_{t+1} = \vw_t -\eta \mSigma_t^{-1/2} \vg_t.
\end{equation}
To illustrate, see~\Cref{fig:sgd-precon}, where we compare SGD to Whitened SGD~\eqref{eq:whiten-steps-gt} on a simple quadratic loss with Gaussian noise. The left and middle panels of~\Cref{fig:sgd-precon} show the trajectories of the two methods, each with the best tuned base learning rate (over five seeds).
% , which happened to be the same base learning rate.
 SGD takes large, noisy steps in the high-variance direction, while Whitened SGD rescales those steps so it can make steadier progress toward the minimum.
\begin{figure}
    \centering
    \includegraphics[width=\linewidth]{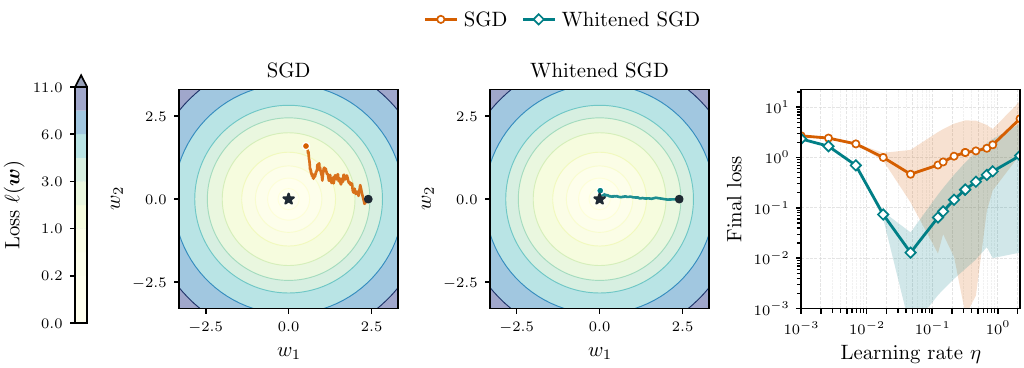}
\caption{SGD and Whitened SGD on $\ell(\vw)=\frac{1}{2}\|\vw\|^2$  with stochastic gradients $\vg_t=\vw_t+\bm{\xi}_t$, where $\bm{\xi}_t\sim\Normal(0,\mSigma)$ and $\mSigma=\diag(0.1,100)$.  Right: final loss after 300 iterations versus base learning rate over five seeds. Left and middle: 300-step trajectories of SGD and Whitened SGD each with best base learning rate chosen from the right plot. }
    \label{fig:sgd-precon}
\end{figure}

In practice we will use an estimate of the mean whitened gradient and update according to
\begin{equation}\label{eq:whiten-steps}
     \vw_{t+1} = \vw_t -\eta \hat{\mSigma}_t^{-1/2} \hat{\vm}_t,
\end{equation}
where $\hat{\mSigma}_t$ and $\hat{\vm}_t$ are estimates of the covariance and mean. We refer to methods that fit the format~\eqref{eq:whiten-steps} as the \emph{Whitened Momentum methods}.
% \[\nabla_{\vw} f_{x_t}(\vw_t) = \vg_t\sim\cN(\vm_t,\mSigma_t)\]

% Our starting point is in~\Cref{sec:Matrix Normal-warmup} where we show that by modeling the gradient matrices (of linear layers) using a Matrix Normal distribution we arrive at variants of Shampoo. The Matrix Normal model uses a Kronecker factored covariance, and we show that the original Shampoo method arises from an empirical estimate of the these Kronecker factors in~\Cref{sec:shampoo-empirical}. 
% We then argue in~\Cref{sec:likelihood-mean-covariance} that a better method for estimating the Kronecker factors and the mean (momentum) simultaneously is to maximize the posterior of the Matrix Normal distribution. We choose a reasonable prior such that the posterior has a well defined and unique global minimal and admits an efficient alternating minimization algorithm for computing the factors. This despite the fact that the posterior is a non-convex function. The result of using two steps of this alternating minimization  on the posterior to compute the mean and covariance factors gives our MN (Matrix Normal model) method.
Our starting point is in~\Cref{sec:bayesian-and-optimizer-agree} where we show that by modeling the gradient 
through a Bayesian recursion or, equivalently, by solving a regularized likelihood optimization problem, we arrive at an online method
for estimating both the variance $\hat{\mSigma}_t$ and mean $\hat{\vm}_t$ of the gradients. Our method can be seen as a full 
covariance matrix variant of Adam~\citep{adam}.

In practice we cannot store the full covariance matrix or approximate its inverse square root. To arrive at a practical method, 
in \Cref{sec:structured} we show how to use structured approximations to the full covariance matrix, such as diagonal matrices or Kronecker-factored matrices with full or diagonal factors.
 
We then show that using a diagonal approximation in~\Cref{sec:Adam} we recover the update of the first and second moment estimates of Adam as a special case~\citep{orvieto2025adamsecretsauce}.
However, the Adam update does not exactly fit the format of a Whitened Momentum method, because Adam uses the uncentered second
 moment to whiten momentum, instead of the centered second moment (covariance) as is used in the Whitened Momentum method~\eqref{eq:whiten-steps}. 
Combining this diagonal covariance with the Whitened Momentum step~\eqref{eq:whiten-steps} yields Centered Adam, a variant of both Adam and AdaBelief~\citep{zhuang2020adabelief}.
We do several experiments showing that Centered Adam 
 performs as well as the original Adam method. Thus, explaining the practical
effectiveness of Adam-like methods may not require an exact account of
Adam itself: the closely related method Centered Adam achieves comparable
performance and admits a simple interpretation as online gradient
whitening using a diagonal covariance approximation.

In~\Cref{sec:mn}, we move beyond diagonal covariance approximations
to the richer class of Kronecker-factored matrices. Within this class,
we show that square-root (as opposed to quarter-root) variants of
Shampoo fit the Whitened Momentum
format~\eqref{eq:whiten-steps}, giving them an interpretation as
gradient whitening with a Kronecker-factored covariance.
Applying our covariance estimation approach to these Kronecker
factors then yields new variants of
KL-Shampoo~\citep{lin2026klshampoo}, which we 
call the MN (Matrix Normal) method.
This name reflects the underlying probabilistic model: a Gaussian
distribution over matrices with a Kronecker-factored covariance
is known as a Matrix Normal distribution.

% We show our optimization problem has a remarkably simple solution through alternating minimization with strong guarantees, resulting in our MN (Matrix Normal model) method.

The MN method has storage and computational costs of the same order as Shampoo, 
but requires substantially more optimizer state storage than Muon. 
For layers with a $1{:}4$ aspect ratio, Shampoo and MN require $5.25$ 
times as much optimizer state storage as Muon, counting momentum and 
covariance factors at the same precision.
To arrive at an algorithm with storage comparable to that of Muon, we propose MND, 
which is a variant of our MN where the factors are diagonal matrices. 
Our experiments\footnote{Our open source implementation is available at:
{\color{blue}\url{https://github.com/parshakova/matrix_normal_optimizers}}} 
in~\Cref{sec:mn-experiments} compare the MND methods 
to the most well-established variants of Shampoo, namely  Distributed Shampoo~\citep{shi2023distributedshampoo}
 and KL-Shampoo on GPT-small and GPT-medium at Chinchilla scale. These experiments show that 
MND was the most robust method to the choice of learning rate and achieved a significantly lower final validation loss.
% Because of the efficacy of MND in these first experiments, we focus on MND for the remainder of the paper. 

% We find that MN was not as effective because of issues with numerical linear algebra routines.
%  Indeed, MN  requires both matrix inverses and inverse square roots, for which we use GPU-friendly 
%  variants of the Newton--Schulz method. We ran into issues around numerical stability when using lower precision 
%  (\text{BF16} and \text{FP32}) which we highlight in~\Cref{sec:Newtonschulz}. 

In~\Cref{sec:mn-muon} we go on to show that Muon also admits a statistical interpretation by modeling 
the behavior of the gradient direction, rather than the gradient itself.
Using this interpretation we derive a new preconditioned version of Muon that we call MND-Muon.
Benchmarking MND-Muon, we find it consistently improves upon Muon, but falls a little short of  NorMuon in \Cref{sec:mn-muon-experiments}.

\paragraph{Background.}
 Our work is largely motivated by the 
 empirical success of  Muon~\citep{muon}, Adam~\citep{adam}, Shampoo~\citep{gupta2018shampoo} 
  and its variants such as KL-Shampoo and Distributed Shampoo \citep{lin2026klshampoo,kasimbeg2025accelerating,shi2023distributedshampoo}. 
More recently, Muon~\citep{muon} has emerged as one of the most competitive methods, having already been used to train several frontier models~\citep{moonlight,glm2025glm45,kimi2025k2,deepseek2026v4}.
Our technique for estimating the mean and covariance of the gradients can be seen as an extension of~\cite{orvieto2025adamsecretsauce},
but where we unify methods through the lens of gradient whitening, 
consider other structured approximations to the covariance matrix, and introduce a complete Bayesian interpretation.
Another closely related line of work is by~\cite{lin2026klshampoo}, where the authors use a KL divergence projection of the uncentered gradient second moment
onto the space of Kronecker-factored matrices to derive a new variant of Shampoo. Our methods can also be seen as KL projections,
but where we jointly model the mean and variance, and show that the target mean and variance are the posterior mode of a Bayesian model.
See~\Cref{sec:extended-background} for a more detailed discussion of related work.

\section{A Gaussian model for the gradients}
\label{sec:bayesian-and-optimizer-agree}

%%\rob{Very confusing to have this here, since we dont explain these dimensions, nor mention them below.}
% For an $m\times n$ parameter matrix, we write $d=mn$ and stack gradient
% columns into $\vg_t\in\R^{mn}$.

Suppose we observe a stochastic gradient
 $\vg_t\sim\Normal(\vm,\mSigma)$ with unknown mean $\vm$ and
covariance $\mSigma$, and have estimates
$(\hat{\vm}_{t-1},\hat{\mSigma}_{t-1})$ computed from past gradients.
As the iterates $\vw_t$ change, we assume that the gradient
distribution changes slowly. Past estimates therefore provide
information about the current mean and covariance.

There are two ways we can use the above to design a method for updating our parameters.
An optimization approach such as that of~\cite{orvieto2025adamsecretsauce} would minimize the negative log-likelihood plus an additional regularizer that encourages the unknown $(\vm, \mSigma)$ to be close to $(\hat{\vm}_{t-1},\hat{\mSigma}_{t-1})$. We call this the optimization approach.
A Bayesian approach would choose a reasonable prior on $(\vm, \mSigma)$ to reflect our assumptions and then update our estimates using the posterior.  Our method admits both interpretations.

% Both the Bayesian and optimization approaches want to find updated estimates $(\hat{\vm}_t, \hat{\mSigma}_t)$ such that it is likely that the new gradient $\vg_t$ is drawn from $\Normal(\hat{\vm}_t, \hat{\mSigma}_t)$. But they differ on how to incorporate past information.

\paragraph{An optimization viewpoint.} The standard optimization approach encodes the importance of the past using a regularizer~\citep{orvieto2025adamsecretsauce}. That is, since we also expect the distribution of the gradients to change slowly across iterations,  we  prefer that $\Normal(\vm, \mSigma)$ remain close to $\Normal(\hat{\vm}_{t-1}, \hat{\mSigma}_{t-1})$ in distribution. For this, we use the KL divergence between these two normal distributions to encourage their difference to be small; see the second term of~\eqref{eq:full-optimization_problem}.

In addition, we also want $\vg_t$ to be a \emph{likely} observation; that is, we also want to maximize the likelihood of  $\vg_t$ being sampled from 
$\Normal(\vm, \mSigma)$.  To balance this likelihood and regularizer, we solve 
% \[ p(\vg_t \mid \vm, \mSigma) \;p(\vm, \mSigma\mid \hat{\vm}_{t-1}, \hat{\mSigma}_{t-1}) = \Normal(\vg_t \mid \vm, \mSigma) \]
\begin{align}
(\hat{\vm}_t, \hat{\mSigma}_t') =  \argmin{\vm, \mSigma\succ0} \Big\{ &-\log \Normal(\vg_t \mid \vm, \mSigma) + \frac{\beta}{1-\beta} \KL\!\left(\Normal(\hat{\vm}_{t-1}, \hat{\mSigma}_{t-1}) \,\middle\|\, \Normal(\vm, \mSigma)\right)\Big\}, \label{eq:full-optimization_problem}
\end{align}
where $\beta \in (0,1)$ is the hyperparameter used to control how much regularization we use.
The above objective is a natural trade-off between our goal of making $\vg_t$ likely and our assumption that the distribution of $\vg_t$ changes slowly.
% The regularization is controlled by the parameter $\beta \in (0,1)$.

\paragraph{A Bayesian viewpoint.}
 The Bayesian would encode this dependency on the past using a prior $p(\vm, \mSigma \mid \hat{\vm}_{t-1},\hat{\mSigma}_{t-1})$.
The standard \emph{conjugate} prior for a multivariate normal with unknown mean and covariance is the Normal-Inverse-Wishart distribution
% \[p(\vm_t, \mSigma_t \mid \hat{\vm}_{t-1},\hat{\mSigma}_{t-1})= \operatorname{Normal-Inverse-Wishart}\]
% \begin{equation}\label{eq:niw}
% p(\vm_t, \mSigma_t \mid \hat{\vm}_{t-1},\hat{\mSigma}_{t-1})= \mathcal{NIW}\!\left(\hat{\vm}_{t-1},\
% \frac{\beta}{1-\beta},\ \frac{\beta}{1-\beta}\hat{\mSigma}_{t-1},\
% \frac{\beta}{1-\beta}-(d+2)\right)
% \end{equation}
\begin{equation}\label{eq:niw}
p(\vm, \mSigma \mid \hat{\vm}_{t-1},\hat{\mSigma}_{t-1})=\mathcal{NIW}\!\left(\hat{\vm}_{t-1},\
\frac{\beta}{1-\beta},\ \frac{\beta}{1-\beta}\hat{\mSigma}_{t-1},\
\frac{\beta}{1-\beta}-(d+2)\right),
\end{equation}
where $\beta \in (0,1)$ is used to discount the past. 
In the usual notation,
$\mathcal{NIW}(\vm_0, \lambda, \bm{\Psi}, \nu)$ is defined by
$$\vm \mid \mSigma \sim \Normal(\vm_0, \mSigma/\lambda) \quad \mbox{ and } \quad 
\mSigma \sim \mbox{Inverse-Wishart}(\nu, \bm{\Psi}).$$
See~\Cref{sec:results-from-2-joint-prior} for the detailed definition. 
The Bayesian approach would then compute the posterior distribution. We will then use this posterior to obtain point estimates $(\hat{\vm}_t, \hat{\mSigma}_t')$ by maximizing the posterior. Using Bayes' rule, the posterior is proportional to the product of the likelihood and prior. Thus, our point estimates are given by
\begin{equation} \label{eq:bayes-up}
    (\hat{\vm}_t, \hat{\mSigma}_t')  =    \argmax{\vm, \mSigma \succ 0} p(\vg_t \mid \vm, \mSigma) \, 
    p(\vm, \mSigma \mid \hat{\vm}_{t-1},\hat{\mSigma}_{t-1}).
\end{equation}
When the prior is proper\footnote{The prior~\eqref{eq:niw} is proper when $\frac{\beta}{1-\beta}>2d+1$,
where $d$ is the dimension of the gradient with $\vg_t \in \R^d$.
If $\beta$ violates this inequality, then~\eqref{eq:niw} still defines an unnormalized
prior kernel, and thus the optimization interpretation~\eqref{eq:full-optimization_problem} and the mode updates below
remain valid.}, conjugacy gives us that the posterior is a Normal-Inverse-Wishart distribution;
see~\Cref{prop:prior-mode,prop:posterior}.
To prepare for iteration $(t+1)$, we raise this posterior density
to the power $\beta$ and normalize again. The resulting
distribution becomes the prior for the next gradient observation.

% Normal-Inverse-Wishart conjugacy gives a Normal-Inverse-Wishart posterior; see~\Cref{prop:prior-mode,prop:posterior}. 
% At iteration $t$, the likelihood of $\vg_t$ updates the prior over $(\vm_t,\mSigma_t)$. Raising the posterior density to the power $\beta$ and renormalizing gives the next prior. The same algebra updates the unnormalized kernels outside the proper-prior regime.

Fortunately, the Bayesian~\eqref{eq:bayes-up} and optimization~\eqref{eq:full-optimization_problem} approaches agree.

\begin{restatable}[Bayesian and optimization approaches agree]{proposition}{closdformfull} 
\label{prop:alternating}
Assume $0<\beta<1$ and $\hat{\mSigma}_{t-1}\succ0$.
Both~\eqref{eq:full-optimization_problem} and~\eqref{eq:bayes-up}  are equivalent to solving the projection problem 
\begin{equation} \label{eq:KL-viewpoint}
    (\hat{\vm}_t, \hat{\mSigma}_t')  =   \argmin{\vm, \mSigma \succ 0 }  \KL\!\left(
\Normal(\hat{\vm}_t,\hat{\mSigma}'_t)
\,\middle\|\,
\Normal(\vm,\mSigma)
\right),
\end{equation}
where the solution 
is given by
\begin{align}
  \hat{\vm}_{t} &= \beta \hat{\vm}_{t-1} + (1-\beta) \vg_t \label{eq:mean_update}\\
  \hat{\mSigma}_t' &= \beta \hat{\mSigma}_{t-1} + \beta(1-\beta) (\vg_t - \hat{\vm}_{t-1}) (\vg_t - \hat{\vm}_{t-1})^{\top}.\label{eq:covariance_update}
\end{align}
\end{restatable}
All proofs are given in the appendix. 

% To see why~\eqref{eq:full-optimization_problem} and~\eqref{eq:bayes-up} are equal, both have the same likelihood term on the left (up to a change of sign). The log prior is also equal, up to an additive constant, to the regularization term in~\eqref{eq:full-optimization_problem}. 

To illustrate, in \Cref{fig:white-mom} we compare the standard Momentum method to our Whitened Momentum~\eqref{eq:whiten-steps} when using \eqref{eq:mean_update}--\eqref{eq:covariance_update} on a simple two-dimensional quadratic. Here we can already see that the high variance of the Momentum method makes the method less stable and unable to achieve as low a loss as Whitened Momentum.

\begin{figure}
    \centering
    \includegraphics[width=\linewidth]{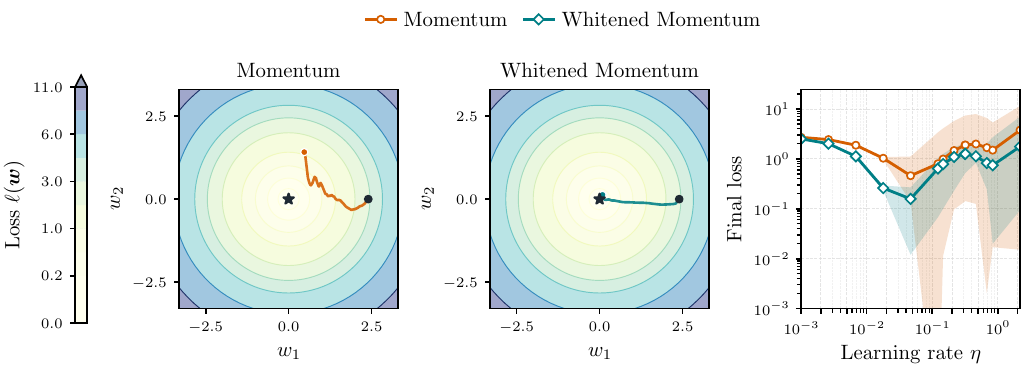}
\caption{Momentum and  Whitened Momentum using~\eqref{eq:mean_update}--\eqref{eq:covariance_update}   with stochastic gradients $\vg_t=\vw_t+\bm{\xi}_t$, where $\bm{\xi}_t\sim\Normal(0,\diag(0.1,100))$. 
Right: final loss after 300 iterations versus base learning rate over five seeds. Left and middle: 300-step trajectories of each method with best base learning rate chosen from the right plot.
}
    \label{fig:white-mom}
\end{figure}

%%\rob{Comment on why this holds. Not that interesting, and can be confusing} 
%  This modeling choice can be directly tied to the optimization problem above since by~\Cref{prop:prior-mode} the negative log density of the prior~\eqref{eq:niw} is equal to the regularization term in~\eqref{eq:full-optimization_problem}, up to an additive constant, that is
% \begin{equation} 
%   -\log p(\vm_t, \mSigma_t | \hat{\vm}_{t-1},\hat{\mSigma}_{t-1}) =  \frac{\beta}{1-\beta}\,
% \mathrm{KL}\!\left(
% \Normal (\hat{\vm}_{t-1},\hat{\mSigma}_{t-1})
% ~\middle\|~
% \Normal (\vm_t,\mSigma_t)
% \right)+\mbox{const}.
% \end{equation}
% The resulting optimization problem has exactly the objective in~\eqref{eq:full-optimization_problem}, and thus the MAP is also given by solving~\eqref{eq:full-optimization_problem}.

The resulting method
\eqref{eq:mean_update}--\eqref{eq:covariance_update} can be interpreted as a generalization of Adam's first and second moment estimates that take into account a full covariance, as opposed to a diagonal covariance~\citep{orvieto2025adamsecretsauce}.

In neural networks, storing this full covariance matrix is practically impossible. Indeed, even for the nanoGPT model~\citep{modded_nanogpt_2024}, which is a relatively small transformer language model with $124$ million parameters, storing the full covariance matrix would exceed $300$ terabytes.
To circumvent this cost, we will use structured approximations for the covariance matrix. 

Since both the Bayesian and optimization approaches boil down to solving a KL projection~\eqref{eq:KL-viewpoint}, we will from now on use this KL formulation to address this problem.

\section{Structured approximations for the covariance}
\label{sec:structured}
To reduce the costs of both storing and using the covariance matrix in the whitening step~\eqref{eq:whiten-steps}, we will impose that $\mSigma$ has some structure. Namely, $\mSigma$ is either diagonal or Kronecker-factored with full or diagonal factors. Let $\mathcal{S}$ denote the class of structured positive definite matrices we are considering. We will now choose $\mSigma \in \mathcal{S}$ that best solves the projection problem~\eqref{eq:KL-viewpoint}, that is
\begin{equation}\label{eq:KL-viewpoint-structured}
\begin{array}{ll}
\underset{\vm,\mSigma \in \mathcal{S}}{\mbox{minimize}}
&
\KL\!\left(
\Normal(\hat{\vm}_t,\hat{\mSigma}'_t)
\,\middle\|\,
\Normal(\vm,\mSigma)
\right).
\end{array}
\end{equation}
Regardless of the structure of $\mathcal{S}$, so long as the covariance is positive definite, solving for the mean always gives $\vm = \hat{\vm}_t$, which is the momentum update~\eqref{eq:mean_update}. Substituting back the solution $\hat{\vm}_t$  into the above gives 
\begin{equation}\label{eq:KL-viewpoint-cov-gen}
\begin{array}{ll}
\underset{\mSigma \in \mathcal{S}}{\mbox{minimize}}
&
\KL\!\left(
\Normal(0,\hat{\mSigma}'_t)
\,\middle\|\,
\Normal(0,\mSigma)
\right).
\end{array}
\end{equation}
Thus it only remains to solve for the covariance matrix.
After dropping multiplicative and additive constants, solving the above KL projection is equivalent to 
% \begin{equation} \label{eq:f-structured}
%   \hat{\mSigma}_t =  \argmin{\mSigma \in \mathcal{S}}  f(\mSigma), \qquad \mbox{where  } f(\mSigma) :=  \log|\mSigma| + \tr(\mSigma^{-1}\hat{\mSigma}'_t).
% \end{equation}
\begin{equation}\label{eq:f-structured}
\hat{\mSigma}_t =
\argmin{\mSigma \in \mathcal{S}}
f(\mSigma), \qquad f(\mSigma) := \log|\mSigma|
+ \tr(\mSigma^{-1}\hat{\mSigma}'_t).
\end{equation}
See~\Cref{prop:regularized-likelihood} for the details of the above.
The next two sections are dedicated to solving this \emph{covariance estimation} problem~\eqref{eq:f-structured} for different choices of structure $\mathcal{S},$ and testing the resulting Whitened Momentum method.

% \rob{Do we really want this next paragraph? It breaks the flow, and I'm not sure it adds much. It just says, this type of projections have been used by Bayesians' before.}

From the strict Bayesian viewpoint, this restriction to $\mathcal{S}$ breaks the exact Bayesian recursion,
 since the unconstrained MAP covariance $\hat{\mSigma}_t'$ given by~\eqref{eq:covariance_update} may not 
 admit the structure imposed by $\mathcal{S}$. But this does fit the \emph{update, then project} pattern used in 
 approximate sequential Bayesian inference~\citep{opper1999bayesian}; see~\Cref{rem:update-project} for more details.

This projection approach~\eqref{eq:KL-viewpoint-cov-gen} was also the main motivation for the KL-Shampoo method~\citep{lin2026klshampoo}, except that~\cite{lin2026klshampoo} use the uncentered gradient second moment
instead of $\hat{\mSigma}_t'$ as a target.

\section{Diagonal approximation and Adam}
\label{sec:Adam}

The simplest structure we can work with is a diagonal matrix. Choosing $\mathcal{S}$ to be the class of positive definite diagonal matrices, we have that the solution to our covariance estimation problem~\eqref{eq:f-structured} is given by $\hat{\mSigma}_t = \diag(\hat{\bm{\sigma}}_t)$, where
\begin{align}
  \hat{\bm{\sigma}}_t &= \beta  \hat{\bm{\sigma}}_{t-1} + \beta(1-\beta) (\vg_t - \hat{\vm}_{t-1}) \circ (\vg_t - \hat{\vm}_{t-1}),\label{eq:covariance_update-diag}
\end{align}
where $\circ$ denotes the elementwise product.
% such that $(\vv \circ \vz)_i = v_i z_i$ for $i=1,\ldots,d$, and $\vv,\vz \in \R^d$.
See~\Cref{sec:diagonal-projection-proof} for the derivation.
Though not immediately apparent, the update of Adam can be written with respect to this estimate of variance $\hat{\bm{\sigma}}_t$
~\eqref{eq:covariance_update-diag} and the mean estimate $\hat{\vm}_t$~\eqref{eq:mean_update}. Indeed, for $\beta_1=\beta_2=\beta$, $\varepsilon=0$, matching buffer initialization, and without bias correction, \cite{orvieto2025adamsecretsauce} showed that the Adam update is equivalent to
\begin{equation} \label{eq:adam-update-uncentered}
    \vw_{t+1} = \vw_t- \eta (\diag(\hat{\bm{\sigma}}_t+\hat{\vm}_t\circ \hat{\vm}_t))^{-1/2}\hat{\vm}_t .
\end{equation}
That is, instead of whitening the momentum using variance, the Adam update~\eqref{eq:adam-update-uncentered}  whitens using an estimate of the \emph{uncentered} second moment since $$\diag(\hat{\bm{\sigma}}_t+\hat{\vm}_t\circ \hat{\vm}_t )\approx \diag\!\left(\E{\vg_t\circ\vg_t}\right).$$
Consequently, Adam
does not fit exactly the Whitened Momentum step~\eqref{eq:whiten-steps}.

\noindent\begin{minipage}{0.45\textwidth}
\begin{figure}[H]
    \centering
    \includegraphics[width=\linewidth]{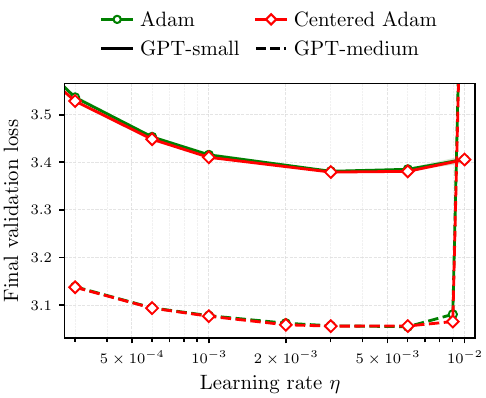}
    \caption{Final validation loss versus learning rate for Adam and Centered Adam on FineWeb. Solid lines: GPT-small; dashed lines: GPT-medium. }
    \label{fig:adam}
\end{figure}
\end{minipage}%
\hspace{0.05\textwidth}%
\begin{minipage}{0.5\textwidth}
The Whitened Momentum update, using $\hat{\bm{\sigma}}_t$
~\eqref{eq:covariance_update-diag} and the mean estimate~\eqref{eq:mean_update},  is given by
\begin{equation} \label{eq:adam-update-centered}
    \vw_{t+1} = \vw_t- \eta \diag(\hat{\bm{\sigma}}_t)^{-1/2}\hat{\vm}_t,
\end{equation}
which we will refer to as the  Centered Adam  method, since it uses the centered second moment (the covariance).

Despite this difference between Centered Adam and Adam, we  found that Centered Adam and Adam are very similar in practice.
\Cref{fig:adam} shows that the final validation losses achieved by running Adam
and Centered Adam at Chinchilla scale are almost the same across all learning rates; see~\Cref{app:adam-experiments} for the experimental details.
% \\[0.1em]

% Thus, although Adam itself does not admit a straightforward explanation, nor exactly fit into our framework, 
% we find that a very closely related method does admit a simple explanation. That is, Centered Adam is a Whitened Momentum update using a diagonal approximation.  
   
Adam has been the method of choice for over a decade, and yet has eluded a simple explanation. It might be that we need not explain Adam.
If Centered Adam  matches the performance of Adam on most problems of interest, then it would be enough to understand Centered Adam, which does admit a simple explanation as a Whitened Momentum method (with a diagonal covariance estimate).
It remains to be seen if Centered Adam performs as well as Adam at larger scales.
\end{minipage}

\paragraph{Relation to AdaBelief.}
AdaBelief~\citep{zhuang2020adabelief} accumulates squared
residuals around the updated momentum:
\[
\tilde{\bm{\sigma}}_t= \beta_2\tilde{\bm{\sigma}}_{t-1}
+(1-\beta_2)(\vg_t-\hat{\vm}_t)
\circ(\vg_t-\hat{\vm}_t),
\]
where $\beta_2 \in (0,1)$.
Since $\vg_t-\hat{\vm}_t
=\beta(\vg_t-\hat{\vm}_{t-1})$, setting $\beta_2=\beta$ gives
$\tilde{\bm{\sigma}}_t=\beta\hat{\bm{\sigma}}_t$ under matching
initialization. That is,  
without bias correction or numerical stabilizers,
Centered Adam 
with learning rate $\eta$ is equivalent to AdaBelief with $\beta_1=\beta_2=\beta$ and
learning rate $\sqrt{\beta}\eta$. Our framework provides a covariance-estimation derivation and whitening
interpretation of this update.

We now turn our attention to using a Kronecker-factored approximation to the covariance matrix, which will in turn result in a variant of Shampoo (more specifically KL-Shampoo), whose diagonal implementation outperforms the tested Shampoo implementations.

\section{A Matrix Normal model of the gradients (MN)} \label{sec:mn}
Most parameters of a neural network are matrices. Let
$\mW \in \R^{m \times n}$ be a matrix parameter and $\mG_t \in \R^{m \times n}$ the gradient with respect to $\mW$ at some given iterate $\mW_t$. 
Our method \eqref{eq:mean_update}--\eqref{eq:covariance_update} with a full covariance matrix is infeasible because storing $\hat{\mSigma}_t\in\R^{nm\times nm}$ would cost $\mathcal{O}(m^{2}n^{2})$ in storage. 
% Also, computing the inverse square root in~\eqref{eq:whiten-steps} would cost $\mathcal{O}(m^{3}n^{3})$
% operations with a direct solver.

However, once the covariance admits a Kronecker decomposition $\hat{\mSigma}_t = \hat{\mC}_t \otimes \hat{\mR}_t$, storage becomes feasible at $\mathcal{O}(n^{2}+m^{2})$. Thus to arrive at a practical algorithm we will now assume that \begin{equation}\label{eq:mat-norm-fact}
    \vg_t = \vec(\mG_t) \sim \Normal(\vm_t,  \mC_t \otimes \mR_t) .
\end{equation}  This distribution~\eqref{eq:mat-norm-fact} is also referred to as the Matrix Normal distribution which is denoted by
$$
\mG_t \sim \MN_{m,n}(\mM_t, \mR_t, \mC_t), \qquad  \vm_t = \vec(\mM_t).
$$  This factored form not only reduces storage, but can also be used to efficiently compute the whitened update~\eqref{eq:whiten-steps}. Indeed the whitened mean in~\eqref{eq:whiten-steps} is now
\begin{align} \label{eq:edzdzczecc}
(\mC_t\otimes\mR_t)^{-1/2}\vec(\mM_t) &=
(\mC_t^{-1/2}\otimes\mR_t^{-1/2})\vec(\mM_t)   = \vec(\mR_t^{-1/2} \mM_t \mC_t^{-1/2}), \end{align}
Thus we only need to approximate inverse square roots of an $n\times n$ and an $m \times m$ matrix, rather than an $nm \times nm$ matrix. 
Using~\eqref{eq:edzdzczecc}, a step along the whitened mean~\eqref{eq:whiten-steps} is 
\begin{equation}\label{eq:mn-sgd_update}
\mW_{t+1} = \mW_t - \eta \hat{\mR}_t^{-1/2} \hat{\mM}_t \hat{\mC}_t^{-1/2}.
\end{equation}
As a side note, using a Kronecker factorization leads to scale ambiguity since for all $a>0$,
\[
(a \mC) \otimes (a^{-1}\mR) = \mC \otimes \mR.
\]
Though this scale ambiguity does not affect the resulting product, it can result in numerical rounding issues if one of the factors races off to infinity, and the other to zero.  To avoid such numerical issues we normalize our estimates $(\hat{\mC}_t,\hat{\mR}_t)$ of these factors to have equal trace $\tr(\hat{\mR}_t) = \tr(\hat{\mC}_t)$.

Next we will compute this Kronecker factorization by solving~\eqref{eq:f-structured} restricted to Kronecker-factored matrices.

% \rob{Repetitive and not that interesting}
% One way to enforce this Kronecker factorization is by solving~\eqref{eq:full-optimization_problem} where the covariance $\mSigma$ is restricted to having a Kronecker factorization, that is
% \begin{align}
% (\hat{\vm}_t, \hat{\mR}_t, \hat{\mC}_t) \in  \argmin{\vm, \, \mSigma=\mC\otimes \mR} \Big\{ &-\log \Normal(\vg_t \mid \vm, \mSigma) \label{eq:mn-sigma_optimization}
%  % \\ &\hspace{5mm}  \notag
% + \frac{\beta}{1-\beta} \KL\left(\Normal(\hat{\vm}_{t-1}, \hat{\mSigma}_{t-1}) ~\|~ \Normal(\vm, \mSigma)\right)\Big\}.
%  \end{align}  %\vspace{-0.3cm}
% % as we prove in~\Cref{prop:regularized-likelihood}.

\subsection{Computing mean and covariance factors}

Choosing $\mathcal{S}$ to be the set of Kronecker-factored matrices in~\eqref{eq:f-structured} gives the following optimization problem
\begin{equation}\label{eq:mn-sigma_optimization}
\begin{array}{ll}
\underset{\substack{\mSigma = \mC \otimes \mR\\
                    \mR \succ 0,\;\mC \succ 0}}{\mbox{minimize}}
&
f(\mSigma) = \log|\mSigma|
+ \tr(\mSigma^{-1}\hat{\mSigma}'_t).
\end{array}
\end{equation}
Recall that $\hat{\mSigma}'_t$ is given by \eqref{eq:covariance_update}.
Initially, finding a global solution to~\eqref{eq:mn-sigma_optimization} may seem
hopeless, because the Kronecker constraint makes the problem nonconvex. Fortunately, despite $f(\mC\otimes\mR)$  being nonconvex, its minimizing Kronecker product is unique, and the factors can be found using alternating minimization over $\mR$ and $\mC$ followed by trace normalization.

\begin{restatable}[Unique solution and alternating minimization]{proposition}{mrcalternating} 
\label{lem:alternating}
Assume $0<\beta<1$ and that $\hat{\mR}_{t-1}$ and $\hat{\mC}_{t-1}$ are positive definite.
The function $f(\mC\otimes \mR)$ in \eqref{eq:mn-sigma_optimization} admits a unique global minimizer $(\hat{\mR}_t^{\star}, \hat{\mC}^{\star}_t)$ satisfying
$\tr(\hat{\mR}_t^{\star}) = \tr(\hat{\mC}_t^{\star})$. For steps $k=1, 2, \ldots$, if we alternately minimize
\[
  \hat{\mR}_{t}^{(k+1)} = \argmin{\mR \succ 0} f(\hat{\mC}_{t}^{(k)} \otimes \mR),
  \qquad 
  \hat{\mC}_t^{(k+1)} = \argmin{\mC \succ 0} f(\mC \otimes \hat{\mR}_t^{(k+1)}),
\]
and rescale the factors reciprocally so that $\tr(\hat{\mR}_{t}^{(k+1)}) = \tr(\hat{\mC}_{t}^{(k+1)})$, we get convergence from any positive definite initialization: $(\hat{\mR}_t^{(k)}, \hat{\mC}^{(k)}_t) \stackrel{k \to \infty}{\longrightarrow} (\hat{\mR}_t^{\star}, \hat{\mC}_t^{\star})$. Defining $\mX_t  = \mG_t - \hat{\mM}_{t-1}$, each block minimization, before rescaling, admits the closed-form solution:
\begin{align}
  \hat{\mR}_{t}^{(k+1)} &= \frac{\beta}{n}\tr\bigl((\hat{\mC}_{t}^{(k)})^{-1}\hat{\mC}_{t-1}\bigr)\hat{\mR}_{t-1}
  + \frac{\beta(1-\beta)}{n}\mX_t(\hat{\mC}_{t}^{(k)})^{-1}\mX_t^{\top},
  \label{eq:mn-R_update}\\
  \hat{\mC}_t^{(k+1)} &= \frac{\beta}{m}\tr\bigl((\hat{\mR}_t^{(k+1)})^{-1}\hat{\mR}_{t-1}\bigr)\hat{\mC}_{t-1}
  + \frac{\beta(1-\beta)}{m}\mX_t^{\top}(\hat{\mR}_t^{(k+1)})^{-1}\mX_t.
  \label{eq:mn-C_update}
\end{align}
\end{restatable}
%% \rob{This doesn't help}
% The mean and unrestricted covariance share the discount $\beta$; the additional
% factor $\beta$ multiplies the covariance innovation in~\eqref{eq:covariance_update}.
% The projected factor updates~\eqref{eq:mn-R_update}--\eqref{eq:mn-C_update}
% also contain trace-weighted prior terms.
The minimizer of $f(\mSigma)$ is unique because $f$ is strictly geodesically convex in the covariance
product $\mSigma=\mC\otimes\mR$ under the affine-invariant geometry; see \Cref{sec:results-from-2-proof-1} and ~\Cref{cor:mn-optimization_solution} for details.

%% \rob{This corollary repeats too much from the proposition. I think it's just enough in text to say what is going and cite the Corollary}
% \begin{restatable}[Solution to \eqref{eq:mn-sigma_optimization}]{corollary}{mnoptsolution}
% \label{cor:mn-optimization_solution}
% Assume $0 < \beta < 1$ and that $\hat{\mR}_{t-1}$ and $\hat{\mC}_{t-1}$ are positive
% definite. The covariance projection problem \eqref{eq:mn-sigma_optimization} admits a
% global solution, unique in $(\mR,\mC)$ under the normalization
% $\tr(\mR)=\tr(\mC)$. Together with $\vm=\hat{\vm}_t$, it also solves
% the regularized likelihood problem \eqref{eq:full-optimization_problem}
% restricted to $\mSigma=\mC\otimes\mR$,
% and it is unique in $(\vm, \mR, \mC)$ under the normalization
% $\tr(\mR) = \tr(\mC)$. This solution is given by the mean update
% \eqref{eq:mean_update}, together with the limit of the factor updates
% \eqref{eq:mn-R_update} and \eqref{eq:mn-C_update} iterated to convergence.
% \end{restatable}

Though it may take many iterations in $k$ for 
\eqref{eq:mn-R_update} and \eqref{eq:mn-C_update} to converge to the solution, we found that using $k=1$ was sufficient in practice.

\subsection{Algorithm}

% \paragraph{The MND method.} 

Computing the updates \eqref{eq:mn-R_update} and \eqref{eq:mn-C_update} requires storing and inverting the factor matrices, which can be computationally expensive and numerically sensitive, as discussed in~\Cref{sec:Newtonschulz}. Indeed, storing $\mR$ and $\mC$ costs $\mathcal{O}(n^2 +m^2)$ in memory, which can be significantly more than the $\mathcal{O}(nm)$ required for storing Adam's second-moment estimator.  Because of this cost in memory, our experiments are focused on using diagonal factors. For these diagonal factors we write $\hat{\mR}_t=\diag(\hat\vr_t)$ and $\hat{\mC}_t=\diag(\hat\vc_t)$, where $\hat\vr_t\in\R^m$ and $\hat\vc_t\in\R^n$. By restricting the factors to diagonal matrices in \eqref{eq:mn-sigma_optimization}, we obtain  
our MND (Matrix Normal Diagonal) method, see \Cref{alg:mnd-method}. The diagonal-factor updates are derived in~\Cref{sec:mnd-factor-proof}. We also implement the (approximate) full MN (Matrix Normal) method, see \Cref{alg:mn-method}.

\begin{algorithm} 
\small
\caption{MND method}
\label{alg:mnd-method}
\begin{algorithmic}[1]
\Require Gradient matrix $\mG_t \in \R^{m \times n}$, step size $\eta > 0$,
momentum coefficient $\beta \in (0, 1)$,
and $\varepsilon>0$
\State Initialize $\hat{\mM}_0\gets0$,
$\hat\vr_0\gets\vone_m/m$, and $\hat\vc_0\gets\vone_n/n$
\For{$t=1,2,\ldots$}
    \State $\mX_t\gets\mG_t-\hat{\mM}_{t-1}$, \quad
    $\hat{\mM}_t\gets\beta\hat{\mM}_{t-1}+(1-\beta)\mG_t$ \label{ln:X-samples}
    \State $\vr^+\gets
    \beta\hat\vr_{t-1}
    +\dfrac{\beta(1-\beta)}{n}
    (\mX_t\circ\mX_t)\mathrm{inv}_\varepsilon(\hat\vc_{t-1})$ \Comment{see \eqref{eq:clipinv}} \label{ln:rd-update}
    \State $\vc^+\gets
    \dfrac{\beta}{m}
    \bigl\langle\mathrm{inv}_\varepsilon(\vr^+),\hat\vr_{t-1}\bigr\rangle\hat\vc_{t-1}
    +\dfrac{\beta(1-\beta)}{m}
    (\mX_t\circ\mX_t)^\top
    \mathrm{inv}_\varepsilon(\vr^+)$ \label{ln:cd-update}
    \State $(\hat\vr_t,\hat\vc_t)
    \gets  \left(
       \alpha \vr^+,\;
        \alpha^{-1} \vc^+
    \right), \quad \mbox{where }  \alpha :=  
    \sqrt{\tfrac{
    \max\{\vone_n^\top\vc^+,\varepsilon\}}{\max\{\vone_m^\top\vr^+,\varepsilon\}}}$
    % \sqrt{
    % \max\{\vone_n^\top\vc^+,\varepsilon\} \big/\max\{\vone_m^\top\vr^+,\varepsilon\}}$ 
    \Comment{trace balancing} \label{ln:trace-balance}
    % \State $\Delta\mW_t\gets-\eta\,
    % \diag\!\bigl(\mathrm{isqrt}_\varepsilon(\hat\vr_t)\bigr)
    % \hat{\mM}_t
    % \diag\!\bigl(\mathrm{isqrt}_\varepsilon(\hat\vc_t)\bigr)$
    \State $\Delta\mW_t \gets -\eta \hat{\mM}_t \circ
    \Bigl(
    \mathrm{isqrt}_\varepsilon(\hat\vr_t)\,
    \mathrm{isqrt}_\varepsilon(\hat\vc_t)^\top
    \Bigr)$
    \State $\mW_{t+1}\gets\mW_t+\Delta\mW_t$
\EndFor
\end{algorithmic}
\end{algorithm}

\paragraph{Relation to KL-Shampoo.}
The MN method is closely related to KL-Shampoo~\citep{lin2026klshampoo},
whose full-matrix factor updates are
\begin{align}
\hat{\mR}_t
&=\beta_2\hat{\mR}_{t-1}
+\frac{1-\beta_2}{n}\mG_t\hat{\mC}_{t-1}^{-1}\mG_t^{\top},
\label{eq:kl-shampoo-R_update}\\
\hat{\mC}_t
&=\beta_2\hat{\mC}_{t-1}
+\frac{1-\beta_2}{m}\mG_t^{\top}\hat{\mR}_{t-1}^{-1}\mG_t.
\label{eq:kl-shampoo-C_update}
\end{align}
Here $\beta_2\in(0,1)$ weights the previous factor estimate and can differ
from the momentum discount $\beta$.
KL-Shampoo computes both new factors
from $(\hat{\mR}_{t-1},\hat{\mC}_{t-1})$.
In each MN iteration, the column-factor update instead uses the
newly computed row factor $\hat{\mR}_t^{(k+1)}$.
For $\beta_2=\beta$, the remaining differences between
\eqref{eq:mn-R_update}--\eqref{eq:mn-C_update} and KL-Shampoo~(\ref{eq:kl-shampoo-R_update})--\eqref{eq:kl-shampoo-C_update} are the absence of the extra
$\beta$ on the innovation, the replacement of the trace weights by $n$ and
$m$, and the use of $\mG_t$ instead of $\mX_t$.
These differences arise because KL-Shampoo  
does not jointly model the mean and the factors.
We will test each of these differences in our experiments further down in~\Cref{sec:mn-experiments,sec:further-mn-experiments}.

% \begin{equation} \label{eq:trace-term}
% \tr\bigl((\hat{\mC}_{t}^{(k)})^{-1}\hat{\mC}_{t-1}\bigr) \rightarrow n,
%     \qquad \tr\bigl((\hat{\mR}_t^{(k+1)})^{-1}\hat{\mR}_{t-1}\bigr) \rightarrow m,
% \end{equation}

\paragraph{Relation to Shampoo.} While we take a likelihood approach to finding $\hat{\mR}_t$ and $\hat{\mC}_t$, we could have taken a moment matching approach instead. 
 By \Cref{lem:mn-gram-moments} we have that 
\[ 
\E{(\mG_t - \mM_t)(\mG_t - \mM_t)^{\top}} = \tr(\mC_t) \mR_t,  
\qquad 
\E{(\mG_t - \mM_t)^{\top}(\mG_t - \mM_t)} = \tr(\mR_t) \mC_t.
\]
 We show how Shampoo can be seen as a method that uses an exponential moving average to fit these identities; see~\Cref{sec:shampoo-moment-matching}.
 % Using $(\mG_t - \mM_t)(\mG_t - \mM_t)^{\top}$ and $(\mG_t - \mM_t)^{\top}(\mG_t - \mM_t)$ as unbiased observations of $\mR_t$ and $\mC_t$ within an exponential moving average yields the original Shampoo method (\Cref{sec:shampoo-moment-matching}).

\section{Experiments with  MND}
\label{sec:mn-experiments}

We now compare MND, which uses diagonal covariance factors, against two popular variants of Shampoo: KL-Shampoo~\citep{lin2026klshampoo} and the official implementation of Distributed Shampoo~\citep{shi2023distributedshampoo}.

\paragraph{Numerical details.} Implementing our diagonal version  MND in a numerically stable way was straightforward. Since diagonal factors allow elementwise inverses and inverse square roots, we could explicitly clip these inverses to be numerically stable as follows
\begin{equation}
\mathrm{inv}_{\varepsilon}(v)
=\max\{v,\varepsilon\}^{-1},
\qquad
\mathrm{isqrt}_{\varepsilon}(v)
={\max\{v,\varepsilon\}}^{-1/2}.
\label{eq:clipinv}
\end{equation}
Finally, to take care of the scale ambiguity,
we enforce $\tr(\mR) = \vone_m^\top  \vr = \vone_n^\top \vc = \tr(\mC).$ 
% The reciprocal rescaling in line~\ref{ln:trace-balance} of~\Cref{alg:mnd-method}
% preserves $\mC\otimes\mR$. It equalizes the actual traces when neither
% trace guard is active.
% This is equivalent to applying the following trace balancing operation
% \begin{equation}
%     \mathcal B_{\rm d}(\vr,\vc)
%     =
%     \left(
%        \alpha \vr,\;
%         \alpha^{-1} \vc
%     \right), \quad \mbox{where }  \alpha := \textstyle \sqrt{\frac{
%     \max\{\vone_n^\top\vc,\varepsilon\}}{\max\{\vone_m^\top\vr,\varepsilon\}}}.
%     \label{eq:mnd-transient-balance}
% \end{equation}

To test what component of the MN method most contributes to its effectiveness, we consider three variants of MN. MND refers to the implementation in~\Cref{alg:mnd-method}, whereas
MND (uncentered) refers to~\Cref{alg:mnd-method} but where the samples are  uncentered $\mX_t=\mG_t$, and finally we use MND $(\beta_2)$ to refer to the variant that uses both uncentered 
 samples and a separate factor discount $\beta_2=0.99$ to update the factors. This 
 MND $(\beta_2)$ uses essentially the same update to the factors as KL-Shampoo (\ref{eq:kl-shampoo-R_update})--\eqref{eq:kl-shampoo-C_update}, 
except that the factors are diagonal, and the prior terms retain the trace weights.
 
 %% \rob{Generated text?}
% For the updated experiments, the factor recursion includes trace-weighted
% prior terms and dimension-normalized evidence, with one row/column sweep
% and trace balancing, as in \Cref{alg:mnd-method} (see \Cref{app:mnd-comparison}).

\paragraph{Results.} \Cref{fig:mn-d} compares the methods when training a
124M-parameter GPT-small model on FineWeb-10B~\citep{penedo2024fineweb}
for 2.5B tokens, at Chinchilla scale~\citep{chinchilla}.
We evaluate each learning rate with seeds 1, 2, and 42 and select the rate
with the lowest mean final validation loss. The method-specific grids and training settings are given in \Cref{app:mnd-comparison}.
%  the  learning grid
% and detailed results are in~\Cref{tab:mnd-updated-results}.
At their selected learning rates, centered and uncentered MND attain mean validation losses of
3.373 and 3.376, respectively, compared with 3.409 for
MND ($\beta_2$), 3.444 for KL-Shampoo, and 3.494 for
Distributed Shampoo. Diagonal covariance factors store
$\mathcal O(m+n)$ entries per matrix, compared with
$\mathcal O(m^2+n^2)$ for dense factors.
\Cref{fig:mn-d-medium} extends the comparison to GPT-medium, trained
for 7.09B tokens with batch size 768 and the same three seeds. 
At both model sizes, 
centered MND achieves the lowest mean final validation
loss over the tested learning rate grids, followed closely by uncentered MND.
The small gap suggests that centering has a limited effect on the best
validation loss in these experiments, consistent with the similar performance
of Centered Adam and Adam in~\Cref{sec:Adam}.
The tested MND ($\beta_2$) variant performs worse than both variants that
use a shared decay factor, so introducing a separate $\beta_2$
does not improve the best validation loss.
Among the two established implementations, KL-Shampoo achieves a lower
best mean validation loss on GPT-small, whereas Distributed Shampoo achieves
a lower best mean validation loss on GPT-medium. Both attain higher best
mean validation losses than all three MND variants at both model sizes. 

\paragraph{The MN method.} We also test the full-factor MN method in \Cref{sec:further-mn-experiments}, but we found it not to be as competitive as MND; see \Cref{fig:mn-lr-sens} for example. We find that MN underperformed because of numerical stability issues in approximating the matrix inverse. Indeed, 
MN needs matrix inverses in~\eqref{eq:mn-R_update} and~\eqref{eq:mn-C_update}, and inverse square roots in~\eqref{eq:mn-sgd_update}. We compute both with Newton--Schulz iterations~\citep{highamfunctions}, which use only matrix multiplications and, therefore, run fast on GPUs~\citep{amsel2025polar}. 
Newton--Schulz is numerically stable for the inverse square root~\citep{higham_stable_1997,highamfunctions}, but for the inverse it is unstable in \text{BF16}, the precision that makes it affordable, see \Cref{sec:Newtonschulz} for details. Another alternative  would be to use an exact SVD,  as is done periodically in the official Distributed Shampoo implementation~\cite{shi2023distributedshampoo}. But computing a full SVD is very slow on GPU, so we ruled out this approach. We conjecture that a stable, GPU-friendly matrix inversion is needed to make the full MN method truly competitive. 
% Full-factor MN is less sensitive to the learning rate than Distributed Shampoo and achieves a lower best validation loss (3.453 versus 3.494). The official implementation \footnote{\url{https://github.com/yorkerlin/KL-Methods}} of KL-Shampoo achieves a best validation loss of 3.444, improving on MN by 0.009 (\Cref{sec:further-mn-experiments}).

\begin{figure}
    \centering
    \includegraphics[width=0.98\linewidth]{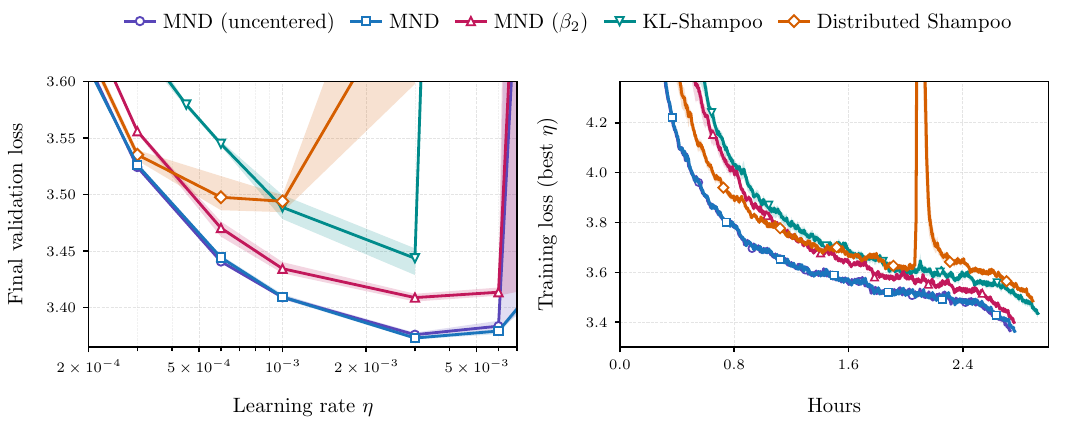}
    \caption{GPT-small on FineWeb, batch 512, 2.5B tokens.
    Left: validation loss versus learning rate. Right: training loss
    versus recorded training hours at tuned rates. Lines: three-seed means;
    shading: seed ranges. Times are cumulative step times averaged across
    seeds. Details: \Cref{app:mnd-comparison}.}
    \label{fig:mn-d}
\end{figure}

\begin{figure}[t]
    \centering
    \includegraphics[width=0.98\linewidth]{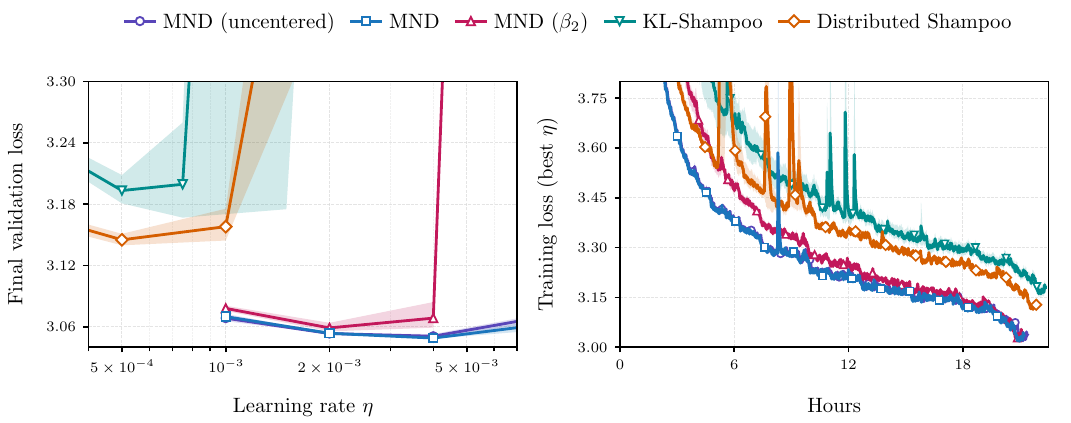}
    \caption{GPT-medium on FineWeb, batch 768, 7.09B tokens.
    Left: validation loss versus learning rate. Right: training loss
    versus recorded training hours at tuned rates. Lines: three-seed means;
    shading: seed ranges. Times are cumulative step times averaged across
    seeds.}
    \label{fig:mn-d-medium}
\end{figure}
% \FloatBarrier

\section{A statistical view of Muon}
\label{sec:mn-muon}

% \rob{Tanya I know you said you would re-write this today, but I took a stab at making it more like my slides already, and incorporating Paulo's comments.}
  
We now switch to the Muon method. We first show that the update direction of the Muon method is the mode of a Matrix Langevin distribution. We use this to motivate a Bayesian viewpoint of Muon. This Bayesian view of Muon requires assuming an isotropic model of the gradient. We use this to motivate a preconditioned view of Muon by instead assuming an anisotropic model of the gradient.
%%%%%%%%%%%%%%%%
\subsection{Muon as a Matrix Langevin mode}
%%%%%%%%%%%%%%%%%
The Muon update is given by
\begin{equation}\label{eq:muon}
\mW_{t+1}=\mW_t-\eta \hat{\mD}_t,
\qquad \hat{\mD}_t=\Polar(\hat{\mM}_t),
\end{equation}
where $\hat{\mM}_t$ is the momentum estimate
\begin{equation}\label{eq:mom-matrix}
\hat{\mM}_t=\beta\hat{\mM}_{t-1}+(1-\beta)\mG_t,
\qquad 0<\beta<1.
\end{equation}
The polar factor in~\eqref{eq:muon} is given by $\Polar(\hat{\mM}_t) = \mU\mV^\top$, where $\hat{\mM}_t=\mU\mS\mV^\top$ is a compact singular value decomposition. 
For simplicity, we assume that $\Polar$ is applied only to full rank matrices.
Thus, Muon models the \emph{direction} (singular vectors) of the momentum estimate and discards its singular values. This direction also has a geometric interpretation, since
the space of matrices with unit singular values is known as the Stiefel manifold~\citep{anthology},
\begin{equation}\label{eq:stiefel}
\St(m,n)=\{\mD\in\R^{m\times n} \mid
\mD^\top\mD=\mI_n\ \mbox{or}\ \mD\mD^\top=\mI_m\}.
\end{equation}
% In particular the polar factor is a nearest matrix on $\St(m,n)$,
% and can be viewed as the matrix analogue of a unit direction.
This gradient direction also has a statistical interpretation, as the mode of a
\emph{Matrix Langevin} distribution. The Matrix
Langevin distribution $\mathrm{ML}_{m,n}(\mM)$ has density
\begin{equation} \label{eq:mat-langevin}
p(\mD\mid\mM)\propto\exp(\langle\mM,\mD\rangle),
\qquad \mD\in\St(m,n),
\end{equation}
with respect to the uniform measure, where
$\langle\mM,\mD\rangle=\tr(\mM^\top\mD)$.
It favors directions aligned with $\mM$.
By~\Cref{lem:ml-mode}, the unique mode of~\eqref{eq:mat-langevin} is $\Polar(\mM)$. Thus, 
\begin{equation} \label{eq:muon-as-mode}
    \hat{\mD}_t = \Polar(\hat{\mM}_t)= \argmax{ \mD\in\St(m,n)} p(\mD\mid\hat{\mM}_t).
\end{equation}
That is,
Muon's direction
$\hat{\mD}_t$ is the unique mode of $\mathrm{ML}_{m,n}(\hat{\mM}_t)$.
Next we derive the entire Muon update, including the momentum recursion~\eqref{eq:mom-matrix}, from a Bayesian observation model.

\subsection{A Bayesian view of Muon}

Suppose the gradient is a noisy observation of a latent direction
$\mD_t\in\St(m,n)$, with isotropic Matrix Normal noise:
\begin{equation}\label{eq:muon-isotropic-observation}
p(\mG_t\mid\mD_t) = \MN_{m,n}(\mD_t,\mI_m,\mI_n).
\end{equation}
To both encode that we want only the direction and 
to make use of previous gradients, we define the prior as 
\begin{equation}\label{eq:muon-isotropic-prior}
p(\mD_t\mid\hat{\mM}_{t-1})
=\mathrm{ML}_{m,n}\!\left(\frac{\beta}{1-\beta}\hat{\mM}_{t-1}\right).
\end{equation}
As in~\Cref{sec:bayesian-and-optimizer-agree}, $\beta\in(0,1)$
controls forgetting: larger values give more weight to past information.
With $\hat{\mM}_0=0$, the initial direction prior is uniform. 
With the likelihood~\eqref{eq:muon-isotropic-observation} and the prior~\eqref{eq:muon-isotropic-prior}, we can now write down the posterior distribution of our desired direction $\mD_t$, which by Bayes' rule is
\begin{equation} \label{eq:muon-post} p(\mD_t\mid\mG_t,\hat{\mM}_{t-1})  \propto p(\mG_t\mid\mD_t) p(\mD_t\mid\hat{\mM}_{t-1}).\end{equation}
This posterior also has a Matrix Langevin distribution (see~\Cref{lem:ml-posterior}). 
Indeed, since
$\|\mD_t\|_F^2=\min\{m,n\}$ is constant on $\St(m,n)$, the likelihood~\eqref{eq:muon-isotropic-observation}
 is proportional to
$\exp(\langle\mG_t,\mD_t\rangle)$. Thus the likelihood times the prior~\eqref{eq:muon-post} gives
\begin{align}
p(\mD_t\mid\mG_t,\hat{\mM}_{t-1})
&\propto\exp\!\left(
\left\langle\frac{\beta}{1-\beta}\hat{\mM}_{t-1}+\mG_t,\mD_t\right\rangle
\right)
\;=\exp\!\left(\frac{1}{1-\beta}
\langle\hat{\mM}_t,\mD_t\rangle\right),
\label{eq:muon-isotropic-posterior}
\end{align}
where the updated momentum $\hat{\mM}_t$ is given by~\eqref{eq:mom-matrix}.
Since the posterior again has distribution $\mathrm{ML}_{m,n}(\frac{1}{1-\beta}\hat{\mM}_t)$, its unique mode, or maximum a posteriori estimate, is given by
\begin{equation} \label{eq:muon-as-map}
   \Polar(\hat{\mM}_t) = \Polar(\tfrac{1}{1-\beta}\hat{\mM}_t)= \argmax{ \mD\in\St(m,n)}p(\mD\mid\mG_t,\hat{\mM}_{t-1}),
\end{equation}
where we used~\Cref{lem:ml-mode}.
Equation~\eqref{eq:muon-as-map} gives precisely the Muon update~\eqref{eq:muon}.
Raising this posterior density to the power $\beta$ and renormalizing
gives the prior~\eqref{eq:muon-isotropic-prior} for the next iteration.
This discounting changes the concentration of the distribution but
preserves its mode, giving a recursive Bayesian interpretation of
momentum followed by polar decomposition.

\subsection{From Muon to MN-Muon}

The isotropic model~\eqref{eq:muon-isotropic-observation} treats all gradient coordinates as equally noisy. 
To account for different noise levels and correlations, we now introduce
positive definite row and column covariance factors
$\mR_t\in\R^{m\times m}$ and $\mC_t\in\R^{n\times n}$, and use the following anisotropic likelihood model for the gradient
\begin{equation}\label{eq:observe-mn-muon}
p(\mG_t\mid\mD_t,\mR_t,\mC_t)
=\MN_{m,n}(\mR_t^{1/2}\mD_t\mC_t^{1/2},\mR_t,\mC_t),
\qquad \mD_t\in\St(m,n).
\end{equation}
Whitening gives the isotropic model of the previous subsection
\[
p(\mR_t^{-1/2}\mG_t\mC_t^{-1/2}\mid\mD_t, \mR_t, \mC_t)
=\MN_{m,n}(\mD_t,\mI_m,\mI_n).
\]
Thus $\mC_t\otimes\mR_t$ is the covariance of $\vec(\mG_t)$, and
$\mR_t^{1/2}\mD_t\mC_t^{1/2}$ is its mean. Unlike the
MN model of~\Cref{sec:mn}, the mean depends on the covariance
factors.
% The covariance of $\vec(\mG_t)$ is $\mC_t\otimes\mR_t$, as
% in~\Cref{sec:mn}. The mean in the original coordinates,
% $\mR_t^{1/2}\mD_t\mC_t^{1/2}$, now depends on the covariance factors.
% This coupling introduces a crucial complication in estimating
% $\mD_t$, $\mR_t$ and $\mC_t$.
% Indeed, while the MN model of~\Cref{sec:mn} decouples the mean
% gradient update from the covariance updates, this is no longer possible
% when modeling the direction $\mD_t$.

It is convenient to start with the Bayesian approach.
Suppose we are given positive definite factors $\hat{\mR}_{t-1}$ and
$\hat{\mC}_{t-1}$, together with the momentum of previously whitened
gradients, $\widetilde{\mM}_{t-1}$, which is recursively defined as 
\begin{equation} \label{eq:white-mom}
  \widetilde{\mM}_{s} = \beta\widetilde{\mM}_{s-1}
+(1-\beta)\hat{\mR}_s^{-1/2}\mG_s\hat{\mC}_s^{-1/2},
\quad s=1,\ldots,t-1,\qquad \widetilde{\mM}_{0}=0.
\end{equation}
% The momentum need not lie on $\St(m,n)$, and is held fixed when
% estimating the current direction and factors.
% As in the isotropic case, a reasonable direction prior is
These quantities remain fixed while we estimate the current direction and
factors. The momentum need not lie on $\St(m,n)$.
As in the isotropic model, we use the direction prior
\begin{equation}
p(\mD_t\mid\widetilde{\mM}_{t-1})
=\mathrm{ML}_{m,n}\!\left(
\frac{\beta}{1-\beta}\widetilde{\mM}_{t-1}\right).
\label{eq:mn-muon-direction-prior}
\end{equation}
Again, $\beta\in(0,1)$ discounts past information, with the same
interpretation as in~\Cref{sec:bayesian-and-optimizer-agree}.
% {\color{magenta} T: Rob, note that there is a difference between covariance distribution in MN section and covariance factor distribution in Muon section.
% Specifially, IW density over \(\mSigma\) is proper when its degrees of freedom exceed $mn-1$, while evaluating IW kernel at \(\mSigma=\mC\otimes\mR\) does not give an integrable density, because the unrestricted rescaling \((\mR,\mC)\mapsto(s\mR,\mC/s)\) leaves the kernel unchanged and makes its integral under \(d\mR\,d\mC\) infinite.}
For the covariance factors, set $\mSigma_t=\mC_t\otimes\mR_t$ and
$\hat{\mSigma}_{t-1}=\hat{\mC}_{t-1}\otimes\hat{\mR}_{t-1}$, and choose the
Inverse-Wishart density kernel
% \rob{Maybe TMI giving the formula below. Maybe instead we just call it the Inverse-Wishart given certain parameters? Or, we call upon~\Cref{prop:regularized-likelihood}  again and use that 
% $p(\mR_t,\mC_t\mid\hat{\mR}_{t-1},\hat{\mC}_{t-1}) \propto \exp\!\left(-\frac{\beta}{1-\beta}\KL\!\left(
% \Normal(0,\hat{\mSigma}_{t-1})
% \,\middle\|\,
% \Normal(0,\mSigma_t)
% \right)\right)$?
% }
% {\color{magenta} T: see the discussion above, the normalization constant does not exist for the covariance factorization. How about we just mention that its Inverse-Wishart density kernel
% without specifying it?}
\begin{equation}\label{eq:mn-muon-covariance-prior}
p(\mR_t,\mC_t\mid\hat{\mR}_{t-1},\hat{\mC}_{t-1})
\propto
|\mSigma_t|^{-\frac{\beta}{2(1-\beta)}}
\exp\!\left(-\frac{\beta}{2(1-\beta)}
\tr(\mSigma_t^{-1}\hat{\mSigma}_{t-1})\right).
\end{equation}
This kernel is maximized when
$\mSigma_t=\hat{\mSigma}_{t-1}$.
% We use the product of the direction density and the factor kernel,
% conditional on $(\widetilde{\mM}_{t-1},\hat{\mR}_{t-1},\hat{\mC}_{t-1})$,
% Multiplying by the likelihood gives the posterior kernel
Multiplying the likelihood, direction
prior, and factor kernel gives the joint posterior kernel
\begin{equation}\label{eq:mn-muon-joint-posterior}
p(\mD,\mR,\mC\mid\mG_t,\widetilde{\mM}_{t-1},\hat{\mR}_{t-1},\hat{\mC}_{t-1})\propto 
p(\mG_t\mid\mD,\mR,\mC)\,
p(\mD\mid\widetilde{\mM}_{t-1})\,
p(\mR,\mC\mid\hat{\mR}_{t-1},\hat{\mC}_{t-1}).
\end{equation}
This product cannot be normalized over unrestricted factors: scaling one factor up and the other down leaves it unchanged, making its integral infinite.
% We define mode
% estimation directly by maximizing this posterior kernel.
We estimate the direction and factors jointly by maximizing this kernel
% \begin{equation}\label{eq:mn-muon-joint-map}
% \begin{array}{ll}
% \mbox{maximize} &
% p(\mG_t\mid\mD,\mR,\mC)\,
% p(\mD\mid\widetilde{\mM}_{t-1})\,
% p(\mR,\mC\mid\hat{\mR}_{t-1},\hat{\mC}_{t-1}) \\[1ex]
% \mbox{subject to} & \mD\in\St(m,n),\quad \mR,\mC\succ0.
% \end{array}
% \end{equation}
\begin{equation}\label{eq:mn-muon-joint-map}
\begin{array}{ll}
\underset{\substack{\mD\in\St(m,n)\\
                    \mR,\mC\succ0}}{\mbox{maximize}}
&
p(\mG_t\mid\mD,\mR,\mC)\,
p(\mD\mid\widetilde{\mM}_{t-1})\,
p(\mR,\mC\mid\hat{\mR}_{t-1},\hat{\mC}_{t-1}).
\end{array}
\end{equation}
% We take the direction and covariance priors to be independent.
% Because reciprocal rescaling of the factors leaves their Kronecker
% product unchanged, this kernel need not be normalizable over the
% factors. We use it to define the posterior mode, as detailed
% in~\Cref{sec:results-from-4-optimization}.

After observing $\mG_t$, we multiply the likelihood
\eqref{eq:observe-mn-muon} by the priors
\eqref{eq:mn-muon-direction-prior}--\eqref{eq:mn-muon-covariance-prior}.
For any fixed positive definite factors $(\mR_t,\mC_t)$,
the conditional posterior of $\mD_t$ is again Matrix Langevin
(\Cref{lem:ml-posterior}):
\begin{equation}
p(\mD_t\mid\mG_t,\mR_t,\mC_t,\widetilde{\mM}_{t-1})
=\mathrm{ML}_{m,n}\!\left(\frac{1}{1-\beta}\widetilde{\mM}_t\right),
\label{eq:mn-muon-momentum}
\end{equation}
where
\begin{equation}
\widetilde{\mM}_t
:=\beta\widetilde{\mM}_{t-1}
+(1-\beta)\mR_t^{-1/2}\mG_t\mC_t^{-1/2}.
\label{eq:white-momentum}
\end{equation}
Thus the conditional MAP direction is
\begin{equation}
\hat{\mD}_t=\Polar(\widetilde{\mM}_t)
=\argmax{\mD\in\St(m,n)}
p(\mD\mid\mG_t,\mR_t,\mC_t,\widetilde{\mM}_{t-1}).
\label{eq:mn-muon-direction}
\end{equation}
This is momentum followed by polar decomposition in whitened coordinates.
Setting $\mR_t=\mI_m$ and $\mC_t=\mI_n$ recovers the isotropic
Muon model.
% This has the same momentum-then-polar structure as Muon,
% now applied to whitened gradients. Fixing both factors to the identity
% recovers the isotropic Muon model.

The direction prior and conditional posterior are therefore conjugate
for fixed covariance factors, but this conjugacy does not extend to
the factors. As in~\Cref{sec:mn}, imposing Kronecker structure prevents
an exact conjugate covariance update; here the covariance-dependent
mean adds a further coupling.
Consequently, we cannot estimate $\mD_t$ independently of the factors
and then estimate the covariance.
We can, however, substitute the conditional MAP direction
\eqref{eq:mn-muon-direction} into the joint posterior and numerically
optimize the resulting \emph{profiled} objective over the factors.
Profiling maximizes over the direction rather than integrating it out.

% This leads to a recursive estimation procedure: the previous estimates
% determine the priors, the new gradient determines the likelihood, and
% their product defines the current joint MAP problem.
% Below we formulate this optimization problem and give tractable
% alternating updates for its direction and covariance blocks.
% After the numerical updates, we retain
% $(\widetilde{\mM}_t,\hat{\mR}_t,\hat{\mC}_t)$ to construct the next priors.
% Discounting the conditional direction posterior preserves its Matrix
% Langevin form and mode. For the factors, rebuilding the prior around
% the new estimates is an approximation to posterior propagation;
% see~\Cref{sec:results-from-4-discounting}.

After these updates, we retain
$(\widetilde{\mM}_t,\hat{\mR}_t,\hat{\mC}_t)$ for the next iteration.
Raising the conditional direction posterior to the power $\beta$ and
renormalizing gives the next direction prior. Rebuilding the factor
kernel around the new estimates approximates posterior propagation
(\Cref{sec:results-from-4-discounting}).

\subsection{Optimization problem}

Given the priors~\eqref{eq:mn-muon-direction-prior}--\eqref{eq:mn-muon-covariance-prior} and the observation model~\eqref{eq:observe-mn-muon}, computing the
joint posterior mode of $(\mD_t, \mR_t, \mC_t)$ is equivalent to solving the regularized likelihood
problem
\begin{align}
\underset{\mD \in \St(m,n),\ \mR,\mC\succ0}{\mbox{minimize}} \Big\{
&-\log \MN_{m,n}\bigl(\mG_t \mid \mR^{1/2}\mD\mC^{1/2}, \mR, \mC\bigr)
\label{eq:mn-muon-optimization_problem}\\
&\qquad + \tfrac{\beta}{1-\beta}\KL\!\left(\MN_{m,n}(\widetilde{\mM}_{t-1}, \mI_m, \mI_n) \,\middle\|\, \MN_{m,n}(\mD, \mI_m, \mI_n)\right) \notag\\
&\qquad + \tfrac{\beta}{1-\beta}\KL\!\left(\MN_{m,n}(0, \hat{\mR}_{t-1}, \hat{\mC}_{t-1}) \,\middle\|\, \MN_{m,n}(0, \mR, \mC)\right)\Big\}. \notag
\end{align}

\Cref{prop:mn-muon-map} establishes this equivalence. We can interpret the objective above analogously to \eqref{eq:full-optimization_problem}: (i) we want estimates that make the newly observed gradient $\mG_t$ likely; but (ii) since gradients are expected to move slowly across iterations, we want the whitened mean direction to stay close
to the previous whitened momentum, and likewise with the covariance.
We could penalize a single joint distribution instead, as in \eqref{eq:full-optimization_problem}; however, that does not propagate the whitened momentum that recovers Muon.

Unlike the optimization problem underlying MN~\eqref{eq:mn-sigma_optimization}, the optimization problem above~\eqref{eq:mn-muon-optimization_problem} couples the
direction $\mD$ and covariance factors $\mR,\mC$ through the whitened gradient.
The problem is nonconvex, but we can alternate between tractable updates of
these three variables, holding the others fixed (\Cref{prop:mn-muon-blocks}). 
\begin{restatable}[Blockwise updates]{proposition}{mnmuonblocks}
\label{prop:mn-muon-blocks}
Each block of \eqref{eq:mn-muon-optimization_problem} has a simple solution given the
others:
\begin{enumerate}[(i)]
\item \emph{Direction $\mD$.} For fixed factors, minimizing over $\mD\in\St(m,n)$ gives
$\Polar\bigl(\beta\widetilde{\mM}_{t-1} + (1-\beta)\mR^{-1/2}\mG_t\mC^{-1/2}\bigr)$, that is,
the momentum update~\eqref{eq:white-momentum} followed by~\eqref{eq:mn-muon-direction}.
% It is unique when this argument has full rank.
\item \emph{Factors.} For fixed $\mD$ and $\mC$, the objective is strictly convex in
$\mR^{-1/2}$, so it has a unique minimizer over $\mR$; likewise for $\mC$.
\item \emph{Diagonal factors.} If $\mR$ and $\mC$ are diagonal, both factor minimizers have
closed forms.
\end{enumerate}
\end{restatable}
The explicit factor updates and the first-order conditions
 characterizing the dense-factor minimizers are derived
  in~\Cref{sec:results-from-4-solution-1}.

We minimize~\eqref{eq:mn-muon-optimization_problem} by alternating over the blocks, starting from
the previous estimates: update $\mR$, then $\mC$, rescale them so that $\tr(\mR) = \tr(\mC)$, and
then update $\mD$. The rescaling changes nothing, since the objective depends on the factors only
through $\mC \otimes \mR$.

\begin{restatable}[Descent and limit points]{proposition}{mnmuonconvergence}
\label{prop:mn-muon-convergence}
Suppose the exact block updates of \Cref{prop:mn-muon-blocks} are repeated in the
order $\mR$, $\mC$, trace balancing, then $\mD$, starting from any
$\mD \in \St(m,n)$ and $\mR,\mC \succ 0$.
Then the objective \eqref{eq:mn-muon-optimization_problem} never increases, and the
trace-balanced iterates remain bounded. Every limit point is a critical point;
at that point, each of $\mD$, $\mR$, and $\mC$ minimizes the objective with the
other two held fixed.
\end{restatable}

Unlike \Cref{lem:alternating}, this result does not guarantee a globally optimal
solution. The coupling between direction and covariance breaks the geodesic convexity
exploited in \Cref{sec:mn}, so which critical point is reached can depend on the warm start.
For a fixed gradient and fixed previous estimates, one exact sweep from a
feasible direction also decreases this objective. 
% The training implementation
% uses clipping and a finite-step polar approximation, so this guarantee does
% not automatically apply to its numerical updates or to the changing
% objectives across training iterations.

\subsection{Algorithm}

% \paragraph{The MN-Muon method.}
% Let $\mW \in \R^{m \times n}$ be a matrix parameter and $\mG_t$ the stochastic gradient at the iterate
% $\mW_t$. The MN-Muon step moves along the estimated whitened mean direction,
% \begin{equation}
% \mW_{t+1} = \mW_t - \eta\hat{\mD}_t,
% \quad \mbox{where} \quad
% \hat{\mD}_t = \Polar(\widetilde{\mM}_t),
% \label{eq:mn-muon-sgd_update}
% \end{equation}
% and where $\widetilde{\mM}_t$ is given by \eqref{eq:mn-muon-momentum} with the factors
% $(\hat{\mR}_t, \hat{\mC}_t)$ obtained from one row-then-column sweep of
% \Cref{prop:mn-muon-blocks}, evaluated at $\mD = \hat{\mD}_{t-1}$. This mirrors
% \eqref{eq:mn-sgd_update}: MN steps along the whitened mean, and MN-Muon along
% the whitened mean direction.
% % Dense factors cost $\mathcal{O}(m^{2}+n^{2})$ memory and require a numerical solve of the first-order conditions. 
% As in Muon, the polar factor can be approximated by
% Newton--Schulz.

% \paragraph{The MND-Muon method.}
Restricting the factors to diagonal matrices $\mR = \diag(\vr)$ and
$\mC = \diag(\vc)$, a row--column sweep with
$\mD=\hat{\mD}_{t-1}$ held fixed uses
\[
\vr^{+} = \Psi_n\bigl(\mQ_t\vc^{-1}, \mB_t\vc^{-1/2}\bigr),
\qquad
\vc^{+} = \Psi_m\bigl(\mQ_t^\top(\vr^{+})^{-1}, \mB_t^\top(\vr^{+})^{-1/2}\bigr),
\]
where vector powers are entrywise, $\circ$ denotes entrywise multiplication,
\begin{equation}
\mQ_t = \beta\hat{\vr}_{t-1}\hat{\vc}_{t-1}^\top + (1-\beta)(\mG_t \circ \mG_t),
\qquad
\mB_t = (1-\beta)(\mG_t \circ \hat{\mD}_{t-1}),
\label{eq:mn-muon-statistics}
\end{equation}
and $\Psi_q$ is applied entrywise for $q\in\{m,n\}$, with
\begin{equation}
\Psi_q(u, v) := \left(\frac{\sqrt{v^2 + 4qu} - v}{2q}\right)^2 .
\label{eq:mn-muon-variance-root}
\end{equation}
We refer to the resulting method as the MND-Muon (Matrix Normal Diagonal Muon) method (\Cref{alg:mn-muon-d-method}). The numerical factor solver uses
$\Psi_{q,\varepsilon}(u,v)=\max\{\Psi_q(u,v),\varepsilon\}$,
clipped inverses and inverse square roots~\eqref{eq:clipinv}, and reciprocal
trace balancing. The latter preserves the covariance product and equalizes
traces when its trace guards are inactive.
We initialize $\hat{\mD}_0=0$ so that the first factor sweep has
$\mB_1=0$; this is an initialization convention rather than a point
on $\St(m,n)$. 
The first polar update produces a feasible direction,
so the block descent result in exact arithmetic applies to subsequent sweeps.

% \rob{ Alternative: With an exact polar computation, the first polar update produces a feasible direction.
% The exact-block descent result then applies to subsequent sweeps when clipping is inactive.}
% \eqref{eq:mnd-transient-balance}.
% where we use the floors of \eqref{eq:clipinv} together with
% $\Psi_{q,\varepsilon}(u, v) = \max\{\Psi_q(u, v), \varepsilon\}$, and balance the factors with
% $\mathcal{B}_{\rm d}$ from \eqref{eq:mnd-transient-balance}.

% We initialize $\hat{\mD}_0 = 0$. This choice is not on $\St(m,n)$, but it gives
% $\mB_1 = 0$, so the first factor sweep uses $\Psi_q(u, 0) = u/q$, and the first polar
% factor $\hat{\mD}_1$ is feasible. For $t \geq 2$, exact block updates do not increase
% the objective. 
% Active floors and approximate polar factors need not preserve this property.
{\small
\begin{algorithm}
\small
\caption{MND-Muon method}
\label{alg:mn-muon-d-method}
\begin{algorithmic}[1]
\Require Gradient matrix $\mG_t \in \R^{m \times n}$, step size $\eta > 0$,
momentum coefficient $\beta \in (0, 1)$, and $\varepsilon > 0$
\State Initialize $\widetilde{\mM}_0 \gets 0$, $\hat{\mD}_0 \gets 0$,
$\hat\vr_0 \gets \vone_m/m$, and $\hat\vc_0 \gets \vone_n/n$
\For{$t = 1, 2, \ldots$}
    \State $\mQ_t \gets \beta\hat\vr_{t-1}\hat\vc_{t-1}^\top + (1-\beta)(\mG_t \circ \mG_t)$
    \State $\mB_t \gets (1-\beta)(\mG_t \circ \hat{\mD}_{t-1})$
    \Comment{previous direction}
    \State $\vr^+ \gets \Psi_{n,\varepsilon}\!\left(
      \mQ_t\,\mathrm{inv}_\varepsilon(\hat\vc_{t-1}),\,
      \mB_t\,\mathrm{isqrt}_\varepsilon(\hat\vc_{t-1})\right)$
    \Comment{see \eqref{eq:mn-muon-variance-root}}
    \State $\vc^+ \gets \Psi_{m,\varepsilon}\!\left(
      \mQ_t^\top\mathrm{inv}_\varepsilon(\vr^+),\,
      \mB_t^\top\mathrm{isqrt}_\varepsilon(\vr^+)\right)$
    \Comment{uses the new row factor}
    % \State $(\hat\vr_t, \hat\vc_t) \gets \mathcal{B}_{\rm d}(\vr^+, \vc^+)$    \Comment{see \eqref{eq:mnd-transient-balance}}
    \State $(\hat\vr_t,\hat\vc_t)
    \gets  \left(
       \alpha \vr^+,\;
        \alpha^{-1} \vc^+
    \right), \quad \mbox{where }  \alpha :=  
    \sqrt{\tfrac{
    \max\{\vone_n^\top\vc^+,\varepsilon\}}{\max\{\vone_m^\top\vr^+,\varepsilon\}}}$
    % \sqrt{
    % \max\{\vone_n^\top\vc^+,\varepsilon\} \max\{\vone_m^\top\vr^+,\varepsilon\}^{-1}}$ 
    \Comment{trace balancing}
    % \State $\widetilde{\mM}_t \gets \beta\widetilde{\mM}_{t-1}
    %   + (1-\beta)\diag\!\bigl(\mathrm{isqrt}_\varepsilon(\hat\vr_t)\bigr)
    %   \mG_t\diag\!\bigl(\mathrm{isqrt}_\varepsilon(\hat\vc_t)\bigr)$
    \State $\widetilde{\mM}_t \gets \beta\widetilde{\mM}_{t-1}
    + (1-\beta)\,\mG_t \circ
    \Bigl(
    \mathrm{isqrt}_\varepsilon(\hat\vr_t)\,
    \mathrm{isqrt}_\varepsilon(\hat\vc_t)^\top
    \Bigr)$
    \Comment{whitened momentum}
    \State $\hat{\mD}_t \gets \Polar(\widetilde{\mM}_t)$
    \State $\mW_{t+1} \gets \mW_t - \eta\hat{\mD}_t$
\EndFor
\end{algorithmic}
\end{algorithm}
}
\paragraph{Relation to Muon.}
If the factors are fixed at $\mR = \mI_m$ and $\mC = \mI_n$, 
then~\Cref{alg:mn-muon-d-method} reduces to the Muon method.
% the objective
% \eqref{eq:mn-muon-optimization_problem} reduces to
% $-2\langle \beta\widetilde{\mM}_{t-1} + (1-\beta)\mG_t, \mD \rangle$ plus a constant.
% Its minimizer is $\hat{\mD}_t = \Polar(\beta\widetilde{\mM}_{t-1} + (1-\beta)\mG_t)$, and
% \eqref{eq:mn-muon-sgd_update} becomes Muon with heavy-ball momentum. The exponential
% moving average differs from Muon's momentum buffer only by a positive scale, which does
% not change the polar factor. MN-Muon thus extends Muon by learning the metric in
% which the momentum is orthogonalized.

\paragraph{Relation to MN.}

When the direction term is absent (\eg $\mB_t = 0$, as in the first iteration), the
diagonal updates reduce to $\Psi_q(u, 0) = u/q$, and more generally the factor updates become
alternating updates on uncentered second moments.
They differ from \eqref{eq:mn-R_update} and
\eqref{eq:mn-C_update} in using the raw gradient instead of the centered $\mX_t$ and in lacking
the extra factor $\beta$ on the new observation. Both differences arise because the mean is no
longer a free parameter that can be profiled out. %Instead, the direction-dependent linear term
% accounts for the mean inside the factor updates.

\section{Experiments with MND-Muon}\label{sec:mn-muon-experiments}

All experiments in this section use the diagonal-factor implementation,
MND-Muon (\Cref{alg:mn-muon-d-method}), with $\beta=0.9$.

\subsection{GPT-small}
We train GPT-small (124M parameters) for 2.5B tokens, using 1024-token
sequences.  For each optimizer, corpus, and effective batch size, we select
the learning rate by validation loss in separate 1B-token runs with seed 42
(\Cref{app:further-mn-muon-experiments}). 
% gives the training and optimizer settings,
% complete three-method validation results, and a six-rate GPT-medium comparison
% after 7.09B training tokens.

On FineWeb at batch 512, \Cref{fig:mn-muon-competitive} compares
MND-Muon with Muon, NorMuon, AdamW,
Distributed Shampoo, and KL-Shampoo, using nine learning rates per optimizer
for the 1B-token sweep with seed 42. We hold the selected learning rates
fixed for the 2.5B-token runs with seeds 42, 1, and 2.
At the selected learning rate of 0.03, MND-Muon reaches mean
final validation loss 3.329, compared with 3.336 for Muon.
NorMuon, also at 0.03, has the lowest mean loss of 3.327.
The remaining baselines reach 3.381 for AdamW, 3.449 for
KL-Shampoo, and 3.501 for Distributed Shampoo.

\begin{figure}
\centering
\includegraphics[width=0.98\linewidth]{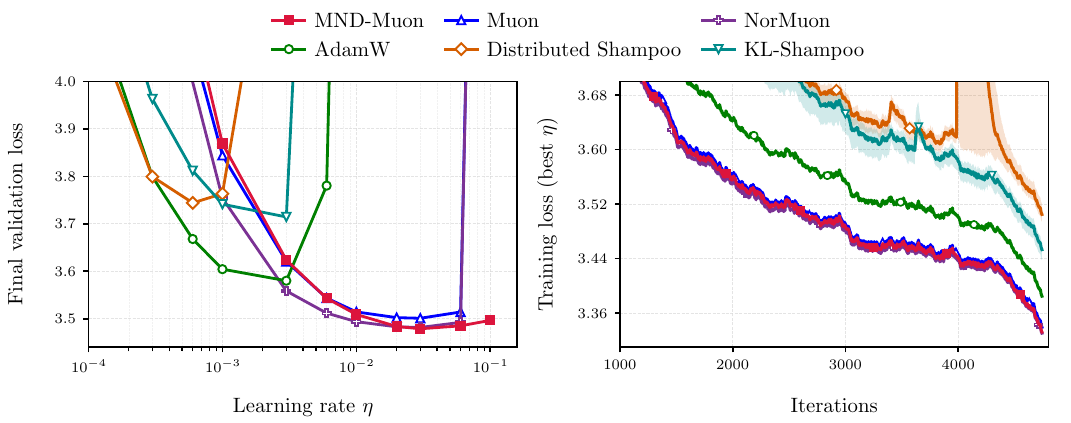}
\caption{GPT-small on FineWeb, batch 512. Left: validation loss versus
learning rate after 1B tokens (seed 42). Right: training loss through
2.5B tokens at selected rates, with three-seed means and shaded seed ranges.
Mean final validation losses after 2.5B tokens are 3.327 for NorMuon,
3.329 for MND-Muon, and 3.336 for Muon (\Cref{app:experimental-details}).}
\label{fig:mn-muon-competitive}
\end{figure}

At a fixed 2.5B-token budget on FineWeb and SlimPajama, increasing the
batch size generally improves the validation performance of Muon,
NorMuon, and MND-Muon
(\Cref{fig:mn-muon-small-batch-curves,fig:mn-muon-batch-size} in
\Cref{app:mn-muon-final-comparison}).
The validation loss gaps between MND-Muon and Muon, and
between NorMuon and Muon, shrink as batch size increases,
indicating that both methods provide larger gains over Muon at
smaller batches.

\subsection{GPT-medium}
\label{sec:gpt-medium}

We compare MND-Muon, Muon, and NorMuon on GPT-medium
(353M active parameters), trained on FineWeb for 7.09B tokens at batch
768 with seed 42. We test six learning rates per method, each for the
full training budget. The lowest final validation losses are 3.0139 for
MND-Muon at learning rate 0.045, 3.017 for Muon at
0.04, and 3.007 for NorMuon at 0.045 (\Cref{fig:mn-muon-medium}).

On this grid, MND-Muon maintains validation losses between 3.014
and 3.016 across the tested rates from 0.02 to 0.06, a range of 0.0024
nats per token. Muon is unstable at
0.045 (loss 6.295), as is NorMuon at 0.09 (loss 6.071).
The right panel compares training at each method's best tested rate.
Full results and the largest-rate stability comparison are in
\Cref{app:mn-muon-medium}.

\begin{figure}
\centering
\includegraphics[width=0.98\linewidth]{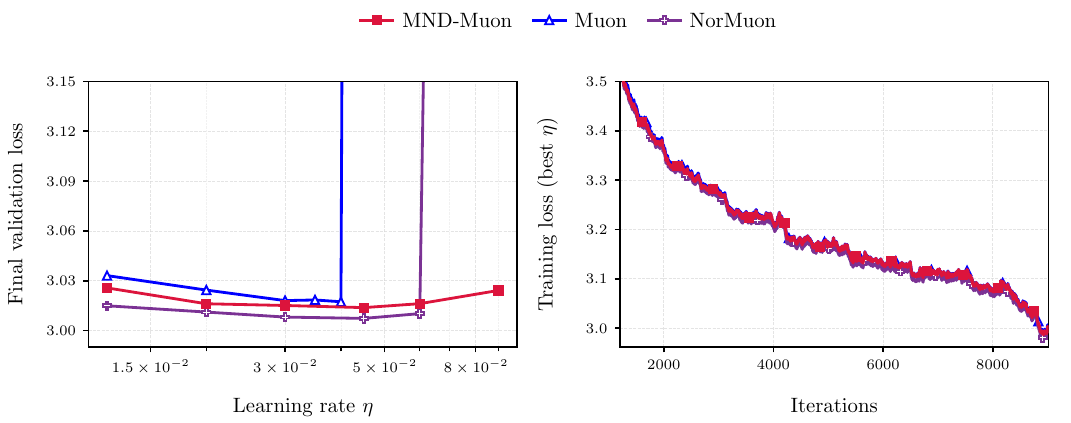}
\caption{GPT-medium on FineWeb: 7.09B tokens, batch 768, seed 42.
Left: validation loss versus learning rate. Right: training loss at the
best tested rates (MND-Muon: 0.045; Muon: 0.04;
NorMuon: 0.045).
Final validation losses are 3.007 for NorMuon,
3.014 for MND-Muon, and 3.017 for Muon.}
\label{fig:mn-muon-medium}
\end{figure}

\section{Conclusions, limitations and future work}
Our work postulates that several stochastic optimization methods can be seen as instances of a Whitened Momentum method~\eqref{eq:whiten-steps}. Using different structured approximations for the covariance, we arrive at variants of Adam and Shampoo.

There are a few limitations of our work, which also point to future work. First, our justification for using a Whitened Momentum method~\eqref{eq:whiten-steps} is mostly empirical: The methods that follow from this framework work well. But having more theoretical insight into the Whitened Momentum method on toy models or certain classes of functions would lend further evidence. Furthermore, we have so far only verified experimentally the new Centered Adam, MND and MND-Muon on two different sizes of GPT-2 at Chinchilla scale. More models, datasets and scales are needed to strengthen our conclusions.

Our approach required assuming the gradients have a Gaussian distribution, which could be reasonable for image classification tasks, but has been shown to be unreasonable for language models~\citep{zhang2020adaptive,kunstner2023noise}. This raises the question of whether we could use other distributions to model the gradient, and whether we could match those distributions to the distributions of the gradients.

Finally, our statistical model of Muon again required assuming a Gaussian model for the gradient, and thus it remains open what the result would be for a different, heavier-tailed model of the gradient distribution.

\subsection*{Acknowledgements}
We thank Umut Simsekli and 
Lénaïc Chizat for initial conversations regarding an early version of the draft. In particular, we would like to acknowledge that Lénaïc Chizat had also independently developed a Bayesian filtration view of the two momentum buffers of Adam.

% \newpage 
\begingroup
\let\url\nolinkurl
\def\doi#1{doi: \nolinkurl{#1}}
\bibliography{references}
\endgroup
\bibliographystyle{plainnat-no-url}

\clearpage
\appendix

\tableofcontents

\clearpage
\section{Extended background}
\label{sec:extended-background}
 
Matrix Normal models and Kronecker-structured approximations appear in
several approaches to neural-network optimization. Their relationship
depends on what the factors represent and how they enter the update.
We first review gradient models, second-moment estimators, and curvature
approximations, then discuss their connections to our methods.

\paragraph{Tensor normal gradient models.}
Prior work has made use of a tensor normal model in
Tensor Normal Training (TNT)~\citep{ren2021tensornormal}.  TNT assumes that the
sampling-based gradient tensor follows a tensor-normal distribution, so that its
covariance has a Kronecker-separable form.  For a matrix-valued parameter, this
reduces to a Matrix Normal model in which the gradient covariance is represented
by separate row and column covariance factors,  
whose inverses define an approximate natural-gradient update. Our centered
formulation instead models gradients computed with observed training labels
and jointly estimates their generally nonzero mean and residual covariance.
The resulting factors scale the update by inverse-square-root whitening of
the estimated gradient noise, rather than inverse Fisher preconditioning.

\paragraph{Structured second-moment estimation.}
A second line of related work consists of Kronecker-structured adaptive
optimizers.
Recent work interprets structured second-moment estimation in
Shampoo and SOAP through covariance estimation and Kullback--Leibler
minimization~\citep{lin2026klshampoo}. In particular, KL-Shampoo
admits a maximum-likelihood interpretation under a zero-mean Matrix Normal
model and uses the fitted factors for whitening. 
Fitting a zero-mean model to uncentered gradients targets a second moment
that contains both residual variability and the mean-gradient contribution.
Our centered formulation jointly estimates the mean and residual covariance
through a KL-regularized likelihood. This joint estimation uses gradient
residuals $\mG_t-\hat{\mM}_{t-1}$, obtained by centering each gradient using
the previous mean estimate, and a shared discount parameter $\beta$ for
momentum and covariance. It also introduces an additional factor $\beta$
in the covariance update, giving the residual contribution a weight of
$\beta(1-\beta)$.

\paragraph{Kronecker-factored curvature.}
There is also a large literature on Kronecker-factored curvature and natural
gradient methods.  K-FAC~\citep{martens2015kfac} approximates Fisher blocks by
Kronecker products, and EKFAC~\citep{george2018ekfac} refines this approximation
by estimating diagonal second moments in a Kronecker-factored eigenbasis;
the corrected matrix need not itself be a single Kronecker product. These
representations are related to Matrix Normal preconditioning, but their
statistical targets differ. For a negative log-likelihood gradient
$h_\theta(x,y)=-\nabla_\theta\log p_\theta(y\mid x)$, the model Fisher is
\[
    \mF(\theta)
    =\mathbb E_{x\sim q_x,\,\widetilde y\sim p_\theta(\cdot\mid x)}
    \!\left[h_\theta(x,\widetilde y)
    h_\theta(x,\widetilde y)^\top\right],
\]
where $q_x$ is the training input distribution. The gradient has zero mean under the model label expectation,
so the Fisher is also a covariance under that distribution. The empirical
Fisher instead averages outer products of per-example gradients with observed
labels; it estimates an uncentered second moment and need not capture Fisher
curvature~\citep{kunstner2019empiricalfisher}. Such second moments can still
inform adaptation to gradient noise. Our centered formulation explicitly
models the covariance $\mSigma$ of the actual stochastic-gradient residuals
and uses $\mSigma^{-1/2}\vm$ to whiten the mean gradient, whereas a
natural-gradient update uses $\mF^{-1}\vm$. Thus our statistical target and
the matrix power used in the update both differ from Fisher preconditioning.

\paragraph{Preconditioned Muon.}
Concurrent work, Mousse~\citep{zhang2026mousserectifyinggeometrymuon},
combines Muon with Kronecker-factored statistics derived from Shampoo.
It performs orthogonalization in whitened coordinates and motivates
the update through an anisotropic constraint on the step. Its factors are
estimated from uncentered gradient Gram matrices; it whitens accumulated
momentum and applies the preconditioning factors again after the polar step.
Whitening before polar decomposition is therefore shared with our approach.
Our formulation derives the whitening transformation from a gradient
observation model, and MN-Muon estimates its factors
jointly with the direction and accumulates momentum from already-whitened
gradients. It then uses the polar factor directly as the update direction
in~\Cref{alg:mn-muon-d-method}, without the additional factor multiplication
after orthogonalization used by Mousse.

NorMuon~\citep{li2025normuon} and
AdaMuon~\citep{si2025adamuon} apply adaptive scaling after
orthogonalization. NorMuon tracks an EMA of the mean
squared entry in each row of the orthogonalized update and uses it to
scale that row. AdaMuon takes the elementwise sign of the momentum before 
orthogonalization and tracks a separate second moment
for each entry of the resulting polar factor. 
NorMuon and AdaMuon estimate their adaptive scales from orthogonalized
updates, whereas MN-Muon estimates its covariance factors from
the gradient observation and the previous direction.

Muon$^2$~\citep{liu2026muon2} applies elementwise second-moment
scaling before orthogonalization. It divides accumulated momentum by the
square root of a moving average of squared gradients, aiming to improve
the conditioning of the subsequent Newton--Schulz iteration. This shares
the order of preconditioning and orthogonalization with MN-Muon,
but uses uncentered gradient statistics and preconditions momentum after
accumulation. In MN-Muon, each gradient is whitened before it
enters the momentum average.

Aurora~\citep{dewulf2026aurora} addresses nonuniform neuron
updates through row balancing. For tall matrices, it alternates damped
row normalization with polar computations, seeking updates with
approximately orthonormal columns and equal row norms. The row scaling
is determined within this inner iteration. Aurora chooses its 
row scaling to equalize the row norms of the
update, whereas MN-Muon derives its whitening factors from the
anisotropic gradient likelihood and covariance prior.

\paragraph{Matrix Normal weight models.}
Finally, Matrix Normal distributions have also appeared in Bayesian deep
learning, though typically as distributions over weights rather than over
stochastic gradients.  Matrix-Gaussian variational posteriors were used by
Louizos and Welling~\citep{louizos2016matrixgaussian} to model correlations among
the input and output dimensions of a neural network layer.  Noisy natural
gradient methods have also been interpreted as variational inference with
Gaussian or matrix-variate Gaussian posteriors~\citep{zhang2018noisyngvi}.  These
works provide useful probabilistic background for matrix-variate Gaussian
structure in neural networks, but they are complementary to the present
gradient-observation viewpoint.

\paragraph{Online Bayesian updating and projection.}
Our recursion alternates an exact conjugate Normal-Inverse-Wishart update
with a Kullback--Leibler projection onto a Kronecker-structured family,
placing it in the long line of ``update, then project'' methods for
sequential Bayesian inference.  The classical instance is assumed density
filtering~\citep{maybeck1979stochastic, opper1999bayesian}, which projects
each exact posterior onto a tractable family in the forward-KL direction;
for exponential families this reduces to moment matching, and ADF forms the
basis of expectation propagation~\citep{minka2001ep}.  Reverse-KL projection
instead yields a one-observation variational inference problem with the
previous approximation as prior, since the Bayesian posterior itself solves
a KL-regularized variational problem~\citep{zellner1988optimal,
knoblauch2022optimization}.  This view underlies streaming variational
Bayes~\citep{broderick2013streaming} and variational continual
learning~\citep{nguyen2018vcl}, and connects to natural-gradient methods
through the Bayesian learning rule~\citep{khan2023blr}.  Our projection lies
on this variational side: over the mean and covariance, it is a reverse-KL
projection onto Normal-Inverse-Wishart distributions with a
Kronecker-structured scale matrix.  
Raising the previous approximate
posterior to the power $\beta$ implements exponential
forgetting~\citep{kulhavy1993forgetting} and is closely related to the
discount factors used in Bayesian dynamic linear
models~\citep{west1997bayesian}.  Kronecker-structured approximations have
also been propagated online for network weights, as in structured Laplace
methods for continual learning~\citep{ritter2018online}.  Unlike these
parameter-belief methods, we maintain beliefs over the parameters of the
stochastic gradient distribution, making the Bayesian update closed form
and leaving the Kronecker projection iterative.

\paragraph{Details of \Cref{fig:sgd-precon,fig:white-mom}.}
Both figures minimize $\ell(\vw)=\frac12\|\vw\|^2$ from $\vw_0=(2.4,0)$, with
Gaussian gradient noise of covariance $\diag(0.1,100)$.
Each run takes 300 steps with $\eta_t=\eta/\sqrt{t+1}$,
for $t=0,\ldots,299$; the horizontal axis shows $\eta$.
Within each seed, both methods and all tested rates use the same noise draws.
For each method, we test each learning rate with seeds $7,\ldots,11$ and choose the
learning rate with the smallest average final loss. Shading shows the range over
these seeds. All displayed trajectories use seed 7.
Whitened SGD uses the exact noise covariance.
Whitened Momentum uses $\beta=0.9$, initializes momentum at the first
gradient and covariance at $\mI$, and whitens with
$(\hat{\mSigma}_t+0.1\mI)^{-1/2}$.

\section{Derivations and proofs of~\Cref{sec:bayesian-and-optimizer-agree} }
\label{sec:derivations-proofs}

\subsection{Joint prior}
\label{sec:results-from-2-joint-prior}

When the prior is proper, we model the latent parameters $\vm_t \in \R^{d}, \mSigma_t \in \R^{d \times d}$ hierarchically through
\begin{align*}
  \mSigma_t \mid \hat{\mSigma}_{t-1} &\sim \IW_{d}\left(\frac{\beta}{1-\beta}-(d+2), \frac{\beta}{1-\beta}\hat{\mSigma}_{t-1}\right),\\
  \vm_t \mid \mSigma_{t}, \hat{\vm}_{t-1} &\sim \Normal\!\left(\hat{\vm}_{t-1}, \frac{1-\beta}{\beta}\mSigma_t\right).
\end{align*}
Here, $0<\beta<1$ and $\hat{\mSigma}_{t-1}\succ0$.

\paragraph{The Inverse-Wishart distribution.} The distribution $\mSigma \sim \IW_p(\nu, \bm{\Psi})$
has density proportional to \\
$|\mSigma|^{-(\nu+p+1)/2}\exp\bigl(-\tfrac{1}{2}\tr(\bm{\Psi}\mSigma^{-1})\bigr)$.
In our case, the dimension is $p=d$, and we choose the degrees of freedom as $$\nu = \frac{\beta}{1-\beta}-(d+2),$$
and the scale matrix is $\bm{\Psi} =\frac{\beta}{1-\beta}\hat{\mSigma}_{t-1}$. The density is therefore proportional to
\begin{align*}
p(\mSigma_t \mid \hat{\mSigma}_{t-1})
\propto |\mSigma_t|^{-\frac{1}{2}\left(\frac{\beta}{1-\beta}-1\right)}
\exp\!\left(-\frac{1}{2}\frac{\beta}{1-\beta}
\tr\!\left(\hat{\mSigma}_{t-1}\mSigma_t^{-1}\right)\right).
\end{align*}
The density is proper as long as the degrees of freedom exceed $d-1$,
that is, when $\frac{\beta}{1-\beta} > 2d+1$.
Whenever this condition fails, we use the joint expression below as an
unnormalized prior kernel. Its mode and the regularized likelihood
interpretation still hold for every $0<\beta<1$, but it need not define
a probability distribution.

\paragraph{Joint prior density.} The conditional normal density is
\[
p(\vm_t \mid \mSigma_t, \hat{\vm}_{t-1})
= (2\pi)^{-\frac{d}{2}} \left(\frac{\beta}{1-\beta}\right)^{\frac{d}{2}} |\mSigma_t|^{-\frac{1}{2}}
\exp\!\left(-\frac{1}{2}\frac{\beta}{1-\beta}
(\vm_t-\hat{\vm}_{t-1})^\top \mSigma_t^{-1} (\vm_t-\hat{\vm}_{t-1})\right).
\]
Multiplying the two densities, the powers of $|\mSigma_t|$ combine to
$-\frac{1}{2}\left(\frac{\beta}{1-\beta}-1\right) - \frac{1}{2} = -\frac{1}{2}\frac{\beta}{1-\beta}$.
Using 
\[
(\vm_t-\hat{\vm}_{t-1})^\top \mSigma_t^{-1} (\vm_t-\hat{\vm}_{t-1})
= \tr\bigl(\mSigma_t^{-1}(\vm_t-\hat{\vm}_{t-1})(\vm_t-\hat{\vm}_{t-1})^\top\bigr),
\]
the two exponentials combine into a single trace, giving
\begin{align} \label{eq:jointprior}
p(\vm_t, \mSigma_t \mid \hat{\vm}_{t-1}, \hat{\mSigma}_{t-1})
&\propto |\mSigma_t|^{-\frac{1}{2}\frac{\beta}{1-\beta}}
\exp\!\left(-\frac{1}{2}\frac{\beta}{1-\beta}
\tr\!\left(\mSigma_t^{-1}\left[\hat{\mSigma}_{t-1}
+ (\vm_t-\hat{\vm}_{t-1})(\vm_t-\hat{\vm}_{t-1})^\top\right]\right)\right).
\end{align}
When proper, the joint prior above follows a Normal-Inverse-Wishart distribution;
otherwise it has the same unnormalized kernel. In the usual notation
$\mathcal{NIW}(\vm_0, \lambda, \bm{\Psi}, \nu)$, defined by
$\vm \mid \mSigma \sim \Normal(\vm_0, \mSigma/\lambda)$ and
$\mSigma \sim \IW(\nu, \bm{\Psi})$, we have 
\begin{equation}\label{app:eq:niw}
\vm_t, \mSigma_t \mid \hat{\vm}_{t-1}, \hat{\mSigma}_{t-1} \sim \mathcal{NIW}\!\left(\hat{\vm}_{t-1},\
\frac{\beta}{1-\beta},\ \frac{\beta}{1-\beta}\hat{\mSigma}_{t-1},\
\frac{\beta}{1-\beta}-(d+2)\right).
\end{equation}
It is the conjugate prior for a multivariate normal with unknown mean and covariance.

\paragraph{Mode of the joint prior density.} The mode of the joint prior above is precisely $(\hat{\vm}_{t-1}, \hat{\mSigma}_{t-1})$, which is one of the motivations behind this parametrization. This is proved as a proposition below.

\begin{proposition}\label{prop:prior-mode}
Assume $0 < \beta < 1$ and that $\hat{\mSigma}_{t-1}$ is positive definite. Then the joint
prior kernel \eqref{eq:jointprior}, over $\vm_t \in \R^{d}$ and positive definite
$\mSigma_t \in \R^{d \times d}$, is such that 
\begin{equation} \label{eq:prior-equal-kl}
    \log p(\vm_t, \mSigma_t \mid \hat{\vm}_{t-1}, \hat{\mSigma}_{t-1}) \; = \; -\frac{\beta}{1-\beta}\,
\KL\!\left(
\Normal(\hat{\vm}_{t-1},\hat{\mSigma}_{t-1})
\,\middle\|\,
\Normal(\vm_t,\mSigma_t)
\right)+\mathrm{const}.
\end{equation}
Consequently the prior kernel
has a unique maximizer at
$(\vm_t, \mSigma_t) = (\hat{\vm}_{t-1}, \hat{\mSigma}_{t-1})$.
\end{proposition}
\begin{proof}
First note that the KL divergence between two normal distributions is given by 
\begin{align} \label{eq:KL-normal}
    \KL\!\left(
\Normal(\hat{\vm}_{t-1},\hat{\mSigma}_{t-1})
\,\middle\|\,
\Normal(\vm_t,\mSigma_t)\right)
&=
\frac{1}{2}\left[
\tr\!\left(\mSigma_t^{-1}\hat{\mSigma}_{t-1}\right)
+
(\vm_t-\hat{\vm}_{t-1})^\top
\mSigma_t^{-1}
(\vm_t-\hat{\vm}_{t-1})  \right] \nonumber \\
& \qquad +\frac{1}{2} \left[\log|\mSigma_t| -\log|\hat{\mSigma}_{t-1}|
 -d
\right].
\end{align}
Taking logarithms in the joint prior~\eqref{eq:jointprior} gives
\begin{align*}
\log p&(\vm_t, \mSigma_t
\mid \hat{\vm}_{t-1}, \hat{\mSigma}_{t-1}) \\
&= -\frac{\beta}{2(1-\beta)}\left[
\log|\mSigma_t|
+ \tr\!\left(\hat{\mSigma}_{t-1}\mSigma_t^{-1}\right)
+ (\vm_t-\hat{\vm}_{t-1})^\top
\mSigma_t^{-1}
(\vm_t-\hat{\vm}_{t-1})
\right] + \mathrm{const} \\
&= -\frac{\beta}{1-\beta}\,\left(
\KL\!\left(
\Normal(\hat{\vm}_{t-1},\hat{\mSigma}_{t-1})
\,\middle\|\,
\Normal(\vm_t,\mSigma_t)\right)
+\frac{\log|\hat{\mSigma}_{t-1}|+d}{2}
 \right)+\mathrm{const},
\end{align*}
which proves~\eqref{eq:prior-equal-kl}. Since the KL is a proper divergence, we have that it is only minimized (or the negative is maximized) when the inputs are equal, that is, when $(\vm_t, \mSigma_t) = (\hat{\vm}_{t-1}, \hat{\mSigma}_{t-1})$.
\end{proof}

\subsection{Joint posterior}
\label{sec:results-from-2-joint-posterior}

By conjugacy, observing the gradient $\vg_t$ turns the Normal-Inverse-Wishart prior above into a
Normal-Inverse-Wishart posterior whose mode $(\hat{\vm}_t, \hat{\mSigma}'_t)$ follows
the recursions~\eqref{eq:mean_update}--\eqref{eq:covariance_update}. 
Hence we need to store
only these two statistics, not the full posterior (\Cref{prop:posterior}).
% $$ - \log p(\vm_t, \mSigma_t \mid \vg_t, \hat{\vm}_{t-1}, \hat{\mSigma}_{t-1})$$
% $$\hat{\vm}_t, \hat{\mSigma}_t =\argmin{\vm, \mSigma}f(\vm, \mSigma )$$
\begin{proposition}\label{prop:posterior}
Assume $0 < \beta < 1$, $\frac{\beta}{1-\beta} > 2d+1$, and that $\hat{\mSigma}_{t-1}$ is
positive definite. Let the prior on $(\vm_t, \mSigma_t)$ be given by~\eqref{eq:jointprior}, and let
$\vg_t \mid \vm_t, \mSigma_t \sim \Normal(\vm_t, \mSigma_t)$ be the likelihood. The resulting posterior is then given by
\begin{equation}\label{eq:niw-posterior}
(\vm_t, \mSigma_t) \mid \vg_t, \hat{\vm}_{t-1}, \hat{\mSigma}_{t-1}
\sim \mathcal{NIW}\!\left(\hat{\vm}_t,\ \frac{1}{1-\beta},\
\frac{1}{1-\beta}\hat{\mSigma}_t',\ \frac{\beta}{1-\beta}-(d+1)\right),
\end{equation}
where
$\hat{\vm}_t$ and $\hat{\mSigma}_t'$ are defined
in \eqref{eq:mean_update} and \eqref{eq:covariance_update}. 
\end{proposition}

\begin{proof}
By Bayes' rule, the posterior is proportional to the Gaussian likelihood
times the prior~\eqref{eq:jointprior},
\[
p(\vm_t,\mSigma_t \mid\vg_t,\hat{\vm}_{t-1},\hat{\mSigma}_{t-1})
\propto
p(\vg_t \mid \vm_t, \mSigma_t)
p(\vm_t, \mSigma_t \mid \hat{\vm}_{t-1}, \hat{\mSigma}_{t-1}).
\]
Completing the square gives
\begin{align}
&\frac{\beta}{1-\beta}
(\vm_t-\hat{\vm}_{t-1})(\vm_t-\hat{\vm}_{t-1})^\top
+(\vg_t-\vm_t)(\vg_t-\vm_t)^\top \notag\\
&\qquad\qquad=
\frac{1}{1-\beta}(\vm_t-\hat{\vm}_t)(\vm_t-\hat{\vm}_t)^\top
+\beta(\vg_t-\hat{\vm}_{t-1})(\vg_t-\hat{\vm}_{t-1})^\top.
\label{eq:complete-square}
\end{align}
Using \eqref{eq:covariance_update}, the posterior is proportional to
\begin{align}
&p(\vm_t,\mSigma_t
\mid\vg_t,\hat{\vm}_{t-1},\hat{\mSigma}_{t-1}) \notag\\
&\quad\propto
|\mSigma_t|^{-\frac{1}{2(1-\beta)}}
\exp\!\left(-\frac{1}{2(1-\beta)}
\left[
\tr\!\left(\mSigma_t^{-1}\hat{\mSigma}'_t\right)
+(\vm_t-\hat{\vm}_t)^\top
\mSigma_t^{-1}(\vm_t-\hat{\vm}_t)
\right]\right).
\label{eq:posterior-kernel}
\end{align}
This is the desired NIW density with parameters
$\lambda=1/(1-\beta)$,
$\bm{\Psi}=\frac{1}{1-\beta}\hat{\mSigma}'_t$, and
$\nu+d+2=\frac{1}{1-\beta}$.
The scale matrix is positive definite and $\nu>d-1$ by the assumptions.

\end{proof}

The new observation increases the NIW parameter $\lambda$ from
$\frac{\beta}{1-\beta}$ in \eqref{app:eq:niw} to
$\frac{1}{1-\beta}$ in \eqref{eq:niw-posterior}.
Raising the posterior density to the power $\beta$
restores the original weight for the next iteration. Thus $\beta$ acts
as a forgetting factor, reducing the influence of past observations.
More precisely, raising~\eqref{eq:posterior-kernel} to this power gives
the joint prior kernel~\eqref{eq:jointprior} with
$(\hat{\vm}_{t-1},\hat{\mSigma}_{t-1})$ replaced by
$(\hat{\vm}_t,\hat{\mSigma}'_t)$.
For a structured covariance, the next prior is instead constructed
with the projected $\hat{\mSigma}_t$, as described in~\Cref{sec:structured}.

\subsection{Closed-form solution and equivalence between optimization and Bayesian approaches}

%%%%%%%%%%%%%%%%%%%%%%%%%%%
\closdformfull*
%%%%%%%%%%%%%%%%%%%%%%%%%%
\begin{proof}
Assume $0<\beta<1$ and $\hat{\mSigma}_{t-1}\succ0$.
By Bayes' rule, maximizing~\eqref{eq:bayes-up} is equivalent to
minimizing the negative log likelihood plus the negative log prior.
By~\Cref{prop:prior-mode}, the latter is the KL regularizer in
\eqref{eq:full-optimization_problem}, up to an additive constant.
Thus the Bayesian and optimization objectives have the same minimizers.

Next, completing the square as in~\eqref{eq:complete-square}
gives the posterior kernel~\eqref{eq:posterior-kernel}, with
$\hat{\vm}_t$ and $\hat{\mSigma}'_t$ given by
\eqref{eq:mean_update}--\eqref{eq:covariance_update}.
Since $\beta\hat{\mSigma}_{t-1}\succ0$ and the remaining covariance
term is positive semidefinite, $\hat{\mSigma}'_t\succ0$.
Using the Gaussian KL formula~\eqref{eq:KL-normal} with this new
mean and covariance, we obtain
\[
-\log p(\vm,\mSigma
\mid\vg_t,\hat{\vm}_{t-1},\hat{\mSigma}_{t-1})
=
\frac{1}{1-\beta}
\KL\!\left(
\Normal(\hat{\vm}_t,\hat{\mSigma}'_t)
\,\middle\|\,
\Normal(\vm,\mSigma)
\right)+\mathrm{const}.
\]
The multiplicative factor is positive and the constant is independent
of $(\vm,\mSigma)$, which proves the equivalence with
\eqref{eq:KL-viewpoint}. The KL divergence is nonnegative and vanishes
exactly at $(\vm,\mSigma)=(\hat{\vm}_t,\hat{\mSigma}'_t)$.
This proves both the claimed updates and uniqueness.

Under the proper-prior condition in~\Cref{prop:posterior}, the kernel
above is the density of the NIW posterior~\eqref{eq:niw-posterior}.
For other values of $\beta$, the same algebra proves the statements
for the unnormalized Bayesian objective.
\end{proof}

\section{Derivations and proofs of~\Cref{sec:structured}}

We now show that the mean update and covariance projection solve the
regularized likelihood problem for any prescribed covariance structure.

\begin{proposition}\label{prop:regularized-likelihood}
Assume $0<\beta<1$ and $\hat{\mSigma}_{t-1}\succ0$, and let
$\mathcal{S}$ be a nonempty class of positive definite matrices in
$\R^{d\times d}$.
Let $\hat{\vm}_t$ and $\hat{\mSigma}'_t$ be given by
\eqref{eq:mean_update} and \eqref{eq:covariance_update}.
A pair $(\vm,\mSigma)$ with $\mSigma\in\mathcal{S}$ solves
\eqref{eq:full-optimization_problem} restricted to
$\mSigma\in\mathcal{S}$ if and only if $\vm=\hat{\vm}_t$
and $\mSigma$ solves~\eqref{eq:f-structured}.
Equivalently, it solves the structured KL projection
\eqref{eq:KL-viewpoint-structured} if and only if
$\vm=\hat{\vm}_t$ and $\mSigma$ solves
\eqref{eq:KL-viewpoint-cov-gen}.
These equivalences do not require existence or uniqueness of a
minimizer over an arbitrary $\mathcal{S}$.
\end{proposition}

\begin{proof}
Let $\mathcal L(\vm,\mSigma)$ denote the objective in
\eqref{eq:full-optimization_problem}.
By~\Cref{prop:alternating} and the Gaussian KL formula
\eqref{eq:KL-normal},
\begin{equation}
\mathcal L(\vm,\mSigma)
=\frac{1}{2(1-\beta)}
\left[
\log|\mSigma|
+\tr\!\left(\mSigma^{-1}\hat{\mSigma}'_t\right)
+(\vm-\hat{\vm}_t)^\top\mSigma^{-1}(\vm-\hat{\vm}_t)
\right]+\mathrm{const}.
\label{eq:objective-rewritten}
\end{equation}
For each $\mSigma\succ0$, the quadratic term is nonnegative and
vanishes only at $\vm=\hat{\vm}_t$.
Since the restriction concerns only the covariance, this mean is
optimal for every $\mSigma\in\mathcal{S}$. Substituting it leaves
\begin{equation}
f(\mSigma)=\log|\mSigma|
+\tr\!\left(\mSigma^{-1}\hat{\mSigma}'_t\right),
\qquad \mSigma\in\mathcal{S}.
\label{eq:reduced-objective}
\end{equation}
Finally, by~\eqref{eq:KL-normal},
\[
\KL\!\left(
\Normal(0,\hat{\mSigma}'_t)
\,\middle\|\,
\Normal(0,\mSigma)
\right)
=\frac12\left[f(\mSigma)-d-\log|\hat{\mSigma}'_t|\right].
\]
Dropping the positive multiplicative and additive constants gives
\eqref{eq:f-structured} and~\eqref{eq:KL-viewpoint-cov-gen}.
The same quadratic mean term appears in the structured joint KL
projection, proving all the equivalences.
\end{proof}

 \begin{remark} \label{rem:update-project}
  
From the strict Bayesian viewpoint, the restriction to $\mathcal{S}$ in~\eqref{eq:KL-viewpoint-structured} breaks the exact Bayesian recursion,
 since the unconstrained MAP covariance $\hat{\mSigma}_t'$ given by~\eqref{eq:covariance_update} may not 
 admit the structure imposed by $\mathcal{S}$, in which case it cannot be the solution to~\eqref{eq:f-structured}.
This construction follows the \emph{update, then project} pattern used in approximate sequential Bayesian inference~\citep{opper1999bayesian}. 
After each observation, we retain the updated mean $\hat{\vm}_t$ from~\eqref{eq:mean_update} and project the Gaussian 
covariance $\hat{\mSigma}_t'$ onto $\mathcal{S}$ via~\eqref{eq:KL-viewpoint-cov-gen}. We then use $(\hat{\vm}_t,\hat{\mSigma}_t)$ 
to construct the next prior through~\eqref{eq:niw} and repeat after observing $\vg_{t+1}$. This gives an approximate Bayesian 
recursion in which the projection acts on the fitted gradient distribution, summarized by its MAP mean and covariance.
% This gives a natural evolution for the gradients $\{\vg_{t}\}_{t=1}^{\infty}$ through latent parameters $(\vm_t, \mR_t, \mC_t)$. 
 \end{remark}

\subsection{Diagonal projection and the second moment in Adam}
\label{sec:diagonal-projection-proof}

Let $\mathcal{S}$ be the positive definite diagonal matrices and write
$\mSigma=\diag(\bm{\sigma})$, with $\sigma_i>0$.
By~\Cref{prop:regularized-likelihood}, it suffices to minimize
\[
f(\diag(\bm{\sigma}))
=\sum_{i=1}^{d}
\left[\log\sigma_i+\frac{[\hat{\mSigma}'_t]_{ii}}{\sigma_i}\right].
\]
Each $s_i=[\hat{\mSigma}'_t]_{ii}$ is positive.
Putting $x_i=s_i/\sigma_i$ gives
$\log\sigma_i+s_i/\sigma_i=\log s_i+x_i-\log x_i$,
which is uniquely minimized at $x_i=1$ because
$x-\log x\geq1$, with equality only at $x=1$.
Thus $\hat{\bm{\sigma}}_t$ is the diagonal of $\hat{\mSigma}'_t$.
Taking the diagonal in~\eqref{eq:covariance_update}, with
$\hat{\mSigma}_{t-1}=\diag(\hat{\bm{\sigma}}_{t-1})$, proves
\eqref{eq:covariance_update-diag}.

To relate this covariance to the uncentered second moment, set
$\vv_t=\hat{\bm{\sigma}}_t+\hat{\vm}_t\circ\hat{\vm}_t$.
The identity
\[
\bigl(\beta\va+(1-\beta)\vb\bigr)\circ
\bigl(\beta\va+(1-\beta)\vb\bigr)
+\beta(1-\beta)(\vb-\va)\circ(\vb-\va)
=\beta\,\va\circ\va+(1-\beta)\,\vb\circ\vb
\]
and~\eqref{eq:mean_update}--\eqref{eq:covariance_update-diag} give
\[
\vv_t=\beta\vv_{t-1}+(1-\beta)\vg_t\circ\vg_t.
\]
Hence $\hat{\vm}_t$ and $\vv_t$ follow the two uncorrected
Adam buffer recursions when $\beta_1=\beta_2=\beta$ and
$\vv_0=\hat{\bm{\sigma}}_0+\hat{\vm}_0\circ\hat{\vm}_0$.
This proves~\eqref{eq:adam-update-uncentered} when $\varepsilon=0$
and the buffers are used without bias correction.
%% \rob{Not needed}
% With zero initialization and the usual bias correction, the step instead uses
% $\hat{\vm}_t/(1-\beta^t)$ and $\vv_t/(1-\beta^t)$.
% Zero covariance initialization is a limiting case of the buffer identities;
% the positive definite projection argument requires strictly positive
% diagonal entries.
% Using the centered covariance itself gives
% \eqref{eq:adam-update-centered}.

\section{Adam and Centered Adam experiments}
\label{app:adam-experiments}

The Adam and Centered Adam methods in~\Cref{fig:adam}
use AdamW and its centered variant.
Both initialize $\hat{\vm}_0=\hat \vv_0=0$ and use
$\beta=0.95$  and $\varepsilon=10^{-8}$.
Their moment updates and bias corrections are
\[
\begin{aligned}
  \hat{\vm}_t &= \beta\hat{\vm}_{t-1}+(1-\beta)\vg_t,
  &\qquad
  \hat \vv_t &= \beta \hat \vv_{t-1}+(1-\beta)\vg_t\circ\vg_t,\\
  \bar{\vm}_t &= \frac{\hat{\vm}_t}{1-\beta^t},
  &
  \bar{\vv}_t &= \frac{\hat \vv_t}{1-\beta^t}.
\end{aligned}
\]
The two methods differ in the diagonal estimate used to scale the update:
\[
\begin{aligned}
  \text{Adam}:\quad
  \vw_{t+1} &= (1-\eta_t\lambda_{\mathrm{wd}})\vw_t
  -\eta_t \frac{\bar{\vm}_t}{\sqrt{\bar{\vv}_t}+\varepsilon},\\[6pt]
  \text{Centered Adam}:\quad
  \vw_{t+1} &= (1-\eta_t\lambda_{\mathrm{wd}})\vw_t
  -\eta_t \frac{\bar{\vm}_t}{
  \sqrt{\max\{\bar{\vv}_t-\bar{\vm}_t\circ\bar{\vm}_t,0\}}+\varepsilon},
\end{aligned}
\]
where all vector operations are entrywise, and $\lambda_{\mathrm{wd}}$ is the
decoupled weight decay coefficient.
Centered Adam uses the AdamW denominator for updates
1 to 50, then switches to the centered denominator at update 51.
This avoids the zero centered variance of the first observation.
The switch preserves the buffers, bias-correction counter, and learning
rate schedule; these 50 updates are included in the training budget.

The model architectures and FineWeb split follow
\Cref{app:experimental-details,app:mn-muon-medium}.
GPT-small uses 4,743 updates at batch 512 (2.49B tokens), with seeds
1, 2, and 42; GPT-medium uses 9,016 updates at batch 768 (7.09B tokens),
with seed 42. Sequences have length 1024. Runs use BF16 autocast,
FP32 parameters and optimizer state, compiled updates, and highest
FP32 matrix-multiplication precision. 
Weight decay is zero ($\lambda_{\mathrm{wd}}=0$) and the global gradient-norm clip is 1.
The learning rate rises linearly for 10\% of the updates, stays constant
for 80\%, and decays linearly for the final 10\%.
Final validation evaluates all complete batches of the held-out
100M-token split and reports cross-entropy in nats per token.

Each learning rate is tested for the full training budget.
The GPT-small grid is
\[
\{0.00001,\,0.00003,\,0.0001,\,0.0003,\,
0.0006,\,0.001,\,0.003,\,0.006,\,0.01\},
\]
and the GPT-medium grid is
\[
\{0.0003,\,0.0006,\,0.001,\,0.002,\,0.003,\,0.006,\,0.009,\,0.012,\,0.02\}.
\]
For each method and model size, the best rate minimizes final
validation loss, averaged over the three seeds for GPT-small and using
one seed for GPT-medium.
\Cref{fig:adam} displays the
range $0.0003$--$0.01$: solid lines are three-seed means, shading spans
the seed minimum and maximum, and dashed lines use only one seed.
\Cref{fig:adam-training} shows the training losses at these selected rates.
\begin{figure}
    \centering
    \includegraphics[width=0.98\linewidth]{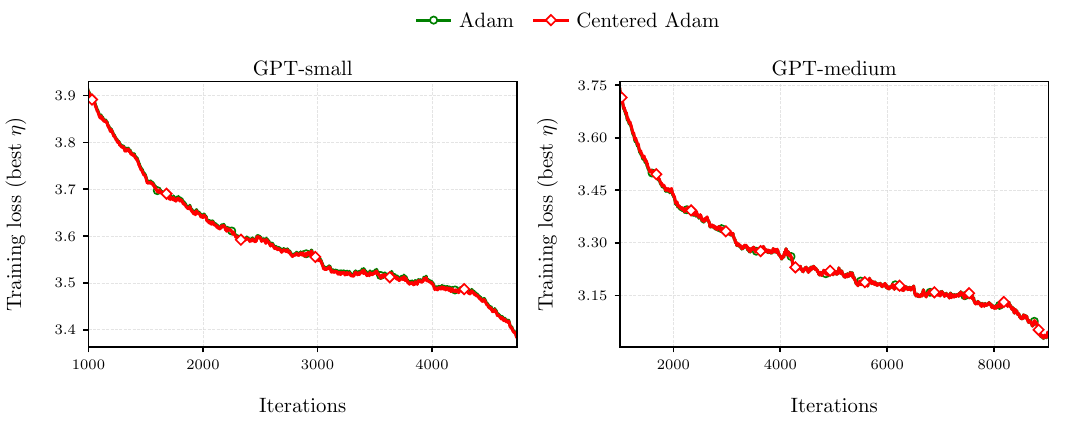}
    \caption{Adam and Centered Adam on FineWeb at the learning rates
    selected from the sweeps in~\Cref{fig:adam}. Left: GPT-small, with
    learning rate $0.003$ for both methods; lines show means over three seeds; shading: seed ranges.
    Right: GPT-medium, one seed, with learning rates $0.006$ for Adam
    and $0.003$ for Centered Adam. 
    % Training losses are smoothed separately
    % per seed with a 25M-token exponential moving average.
    }
    \label{fig:adam-training}
\end{figure}

\section{Derivations and proofs of~\Cref{sec:mn}}

\subsection{Proof of \Cref{lem:alternating}}
\label{sec:results-from-2-proof-1}

By~\Cref{prop:regularized-likelihood}, the mean has already been
eliminated and it remains to solve the covariance problem
\eqref{eq:mn-sigma_optimization}. We prove the closed-form block
updates, existence and uniqueness after trace normalization, and
convergence of alternating minimization. The argument uses the
geodesic convexity of Kronecker-structured covariance estimation
\citep{wiesel2012geodesic}.

\mrcalternating*

\begin{proof}
For every positive definite $\mS \in \R^{p \times p}$ and every
positive definite $\mSigma \in \R^{p \times p}$, the same scalar
inequality used in~\Cref{sec:diagonal-projection-proof} gives
\begin{equation}
\log|\mSigma| + \tr(\mSigma^{-1}\mS)
= \log|\mS| + \sum_{j=1}^{p}(\lambda_j - \log\lambda_j)
\geq \log|\mS| + p,
\label{eq:logdet-trace-bound}
\end{equation}
where $\lambda_1, \dots, \lambda_p > 0$ are the eigenvalues of
$\mS^{1/2}\mSigma^{-1}\mS^{1/2}$. Indeed, their sum is
$\tr(\mSigma^{-1}\mS)$ and their product is $|\mS|/|\mSigma|$,
so the equality follows by taking logarithms; the inequality follows
from $\lambda-\log\lambda\geq1$.
Equality holds if and only if
$\mSigma = \mS$. In the proof below we use
\eqref{eq:logdet-trace-bound} with $\mS = \hat{\mSigma}_t'$. This matrix is positive
definite because $\beta\hat{\mC}_{t-1} \otimes \hat{\mR}_{t-1}$ is positive definite and
the rank-one term in \eqref{eq:covariance_update} is positive semidefinite.

\emph{Step 1: $f$ on Kronecker products.}
Let $\mR$ and $\mC$ be positive definite, and set
$\mX_t=\mG_t-\hat{\mM}_{t-1}$, where
$\vec(\hat{\mM}_{t-1})=\hat{\vm}_{t-1}$.
We use four standard identities:
\begin{enumerate}[(i)]
  \item $|\mC \otimes \mR| = |\mC|^m|\mR|^n$;
  \item $(\mC \otimes \mR)^{-1} = \mC^{-1} \otimes \mR^{-1}$;
  \item $\tr\bigl((\mC^{-1} \otimes \mR^{-1})(\hat{\mC}_{t-1} \otimes \hat{\mR}_{t-1})\bigr)
  = \tr(\mC^{-1}\hat{\mC}_{t-1})\tr(\mR^{-1}\hat{\mR}_{t-1})$;
  \item $(\mC^{-1} \otimes \mR^{-1})\vec(\mX)
  = \vec(\mR^{-1}\mX\mC^{-1})$ for every $m \times n$ matrix $\mX$.
\end{enumerate}
Together with \eqref{eq:covariance_update}, these give
\begin{align}
f(\mC \otimes \mR)
&= m\log|\mC| + n\log|\mR|
+ \beta\tr(\mC^{-1}\hat{\mC}_{t-1})\tr(\mR^{-1}\hat{\mR}_{t-1}) \nonumber \\
&\quad + \beta(1-\beta)\tr\!\left(\mR^{-1}(\mG_t - \hat{\mM}_{t-1})\mC^{-1}(\mG_t - \hat{\mM}_{t-1})^\top\right).
\label{eq:f-kronecker}
\end{align}
The last trace can equivalently be written as
$\tr\bigl(\mC^{-1}\mX_t^\top\mR^{-1}\mX_t\bigr)$.

\emph{Step 2: The block minimizations.}
For positive definite $\mC$, let $\mathcal{R}(\mC)$ denote the right-hand side of
\eqref{eq:mn-R_update} with $\hat{\mC}_t^{(k)}$ replaced by $\mC$. For positive definite
$\mR$, let $\mathcal{C}(\mR)$ denote the right-hand side of \eqref{eq:mn-C_update} with
$\hat{\mR}_t^{(k+1)}$ replaced by $\mR$.

Note that $\tr(\mC^{-1}\hat{\mC}_{t-1}) = \tr(\mC^{-1/2}\hat{\mC}_{t-1}\mC^{-1/2}) > 0$.
Hence $\mathcal{R}(\mC)$ is a positive multiple of $\hat{\mR}_{t-1}$ plus a positive
semidefinite matrix, and is therefore positive definite. The same argument shows that
$\mathcal{C}(\mR)$ is positive definite.

Fix $\mC$. By \eqref{eq:f-kronecker}, the terms of $f(\mC \otimes \mR)$ that depend on
$\mR$ are
\[
n\log|\mR| + \tr\bigl(\mR^{-1}\,n\mathcal{R}(\mC)\bigr)
= n\left[\log|\mR| + \tr\bigl(\mR^{-1}\mathcal{R}(\mC)\bigr)\right].
\]
By \eqref{eq:logdet-trace-bound}, this is uniquely minimized at $\mR = \mathcal{R}(\mC)$,
which gives \eqref{eq:mn-R_update}. Symmetrically, for fixed $\mR$, the function
$\mC \mapsto f(\mC \otimes \mR)$ is uniquely minimized at
$\mC = \mathcal{C}(\mR)$, which gives \eqref{eq:mn-C_update}.

Both maps are continuous and scale-equivariant:
$\mathcal{R}(a\mC) = a^{-1}\mathcal{R}(\mC)$ and
$\mathcal{C}(a^{-1}\mR) = a\,\mathcal{C}(\mR)$ for $a > 0$. The rescaling step therefore
changes neither $f$ nor the relation
$\hat{\mC}_t^{(k+1)} = \mathcal{C}(\hat{\mR}_t^{(k+1)})$.

\emph{Step 3: Compact sublevel sets and existence.}
Fix $c \in \R$. Let $\mathcal{K}_c$ be the set of positive definite pairs $(\mR, \mC)$
with $\tr(\mR) = \tr(\mC)$ and $f(\mC \otimes \mR) \leq c$. We show that
$\mathcal{K}_c$ is compact.

Let $(\mR, \mC) \in \mathcal{K}_c$. Each summand $\lambda_j - \log\lambda_j$ in
\eqref{eq:logdet-trace-bound} is at least $1$. Hence each satisfies
$\lambda_j - \log\lambda_j \leq c - \log|\hat{\mSigma}_t'| - (mn - 1)$. Since
$x - \log x \to \infty$ as $x \to 0^+$ and as $x \to \infty$, all eigenvalues of
$(\hat{\mSigma}_t')^{1/2}(\mC \otimes \mR)^{-1}(\hat{\mSigma}_t')^{1/2}$ lie in an
interval $[a_c, b_c] \subset (0, \infty)$ that depends only on $c$. Consequently
$\ell_c\mI_{mn} \preceq \mC \otimes \mR \preceq u_c\mI_{mn}$, with
$\ell_c = \lambda_{\min}(\hat{\mSigma}_t')/b_c > 0$ and
$u_c = \lambda_{\max}(\hat{\mSigma}_t')/a_c$.

Since $\tr(\mC)\tr(\mR) = \tr(\mC \otimes \mR) \in [mn\,\ell_c, mn\,u_c]$ and
$\tr(\mR) = \tr(\mC)$, both traces lie in $[\sqrt{mn\,\ell_c}, \sqrt{mn\,u_c}]$. In
particular, $\lambda_{\max}(\mR)$ and $\lambda_{\max}(\mC)$ are at most
$\sqrt{mn\,u_c}$. Since
$\lambda_{\min}(\mC)\lambda_{\min}(\mR) = \lambda_{\min}(\mC \otimes \mR) \geq \ell_c$,
we also get $\lambda_{\min}(\mR) \geq \ell_c/\sqrt{mn\,u_c}$, and likewise for $\mC$.

Thus $\mathcal{K}_c$ is contained in the compact set of pairs whose eigenvalues all lie in
$[\ell_c/\sqrt{mn\,u_c}, \sqrt{mn\,u_c}]$. On that set, $f(\mC \otimes \mR)$ and
$\tr(\mR) - \tr(\mC)$ are continuous, so $\mathcal{K}_c$ is closed in it and hence
compact.

The reciprocal rescaling $(\mR,\mC)\mapsto(s\mR,s^{-1}\mC)$
preserves $\mC\otimes\mR$ and hence $f$. Choosing
$s=\sqrt{\tr(\mC)/\tr(\mR)}$ makes the traces equal.
Thus every positive definite pair can be trace normalized without
changing $f$. Choose $c$ such that $\mathcal{K}_c$ is nonempty. The continuous function
$f(\mC \otimes \mR)$ attains its minimum over the compact set $\mathcal{K}_c$, and this
minimum is the global minimum.

\emph{Step 4: Geodesic convexity and uniqueness.}
For positive definite matrices $\mA$ and $\mB$ of equal size and $s \in \R$, let
\[
\gamma_{\mA, \mB}(s) := \mA^{1/2}(\mA^{-1/2}\mB\mA^{-1/2})^{s}\mA^{1/2}.
\]
This is positive definite, with $\gamma_{\mA, \mB}(0) = \mA$ and
$\gamma_{\mA, \mB}(1) = \mB$. Let $\sum_i \lambda_i\vu_i\vu_i^\top$ be an eigendecomposition of
$\mA^{-1/2}\mB\mA^{-1/2}$, and take $\mA, \mB \in \R^{mn \times mn}$. Then
\begin{align*}
\log|\gamma_{\mA, \mB}(s)| &= \log|\mA| + s\sum_i \log\lambda_i, \\
\tr\bigl(\gamma_{\mA, \mB}(s)^{-1}\hat{\mSigma}_t'\bigr)
&= \sum_i \bigl(\vu_i^\top\mA^{-1/2}\hat{\mSigma}_t'\mA^{-1/2}\vu_i\bigr)\,e^{-s\log\lambda_i}.
\end{align*}
The first expression is affine in $s$. In the second, the weights are positive and each
$e^{-s\log\lambda_i}$ is convex, strictly so when $\lambda_i \neq 1$. Hence
$s \mapsto f(\gamma_{\mA, \mB}(s))$ is convex, and strictly convex unless all
$\lambda_i = 1$, that is, unless $\mA = \mB$.

Now suppose $\mA = \mC_1 \otimes \mR_1$ and $\mB = \mC_2 \otimes \mR_2$.
Then $\mA^{\pm 1/2} = \mC_1^{\pm 1/2} \otimes \mR_1^{\pm 1/2}$, and by the
mixed-product property
$\mA^{-1/2}\mB\mA^{-1/2}
= (\mC_1^{-1/2}\mC_2\mC_1^{-1/2}) \otimes (\mR_1^{-1/2}\mR_2\mR_1^{-1/2})$.
Diagonalizing both factors shows that the $s$-th power of a Kronecker product of positive
definite matrices is the Kronecker product of their $s$-th powers. Hence
\begin{equation}
\gamma_{\mC_1 \otimes \mR_1,\, \mC_2 \otimes \mR_2}(s)
= \gamma_{\mC_1, \mC_2}(s) \otimes \gamma_{\mR_1, \mR_2}(s).
\label{eq:kron-geodesic}
\end{equation}
In particular, this curve stays within Kronecker products of positive definite matrices.

Suppose $\mC_1 \otimes \mR_1 \neq \mC_2 \otimes \mR_2$ were both global minimizers.
By strict convexity, $f$ at the feasible point
$\gamma_{\mC_1, \mC_2}(\tfrac{1}{2}) \otimes \gamma_{\mR_1, \mR_2}(\tfrac{1}{2})$
would be strictly smaller than the minimum, which is a contradiction. The minimizing product
is therefore unique.

It remains to show uniqueness of the normalized factors. Suppose
$\mC \otimes \mR = \mC' \otimes \mR'$ with all four matrices positive definite.
Comparing the $(k,l)$-th $m \times m$ blocks gives
$[\mC]_{kl}\mR = [\mC']_{kl}\mR'$. Taking $k = l = 1$ and using
$[\mC]_{11}, [\mC']_{11} > 0$ yields $\mR' = a^{-1}\mR$ with
$a = [\mC']_{11}/[\mC]_{11} > 0$. Substituting back gives $\mC' = a\mC$. The
normalization $\tr(a^{-1}\mR) = \tr(a\mC)$ determines $a > 0$ uniquely. This proves the
existence and uniqueness of $(\hat{\mR}_t^{\star}, \hat{\mC}_t^{\star})$.

\emph{Step 5: Blockwise minimizers are global minimizers.}
Let $(\bar{\mR}, \bar{\mC})$ be positive definite with
$\bar{\mR} = \mathcal{R}(\bar{\mC})$ and $\bar{\mC} = \mathcal{C}(\bar{\mR})$.
The function $(\mR, \mC) \mapsto f(\mC \otimes \mR)$ is smooth on the open set of
positive definite pairs. By Step 2, $(\bar{\mR}, \bar{\mC})$ is an unconstrained
minimizer in each block separately. Both partial derivatives therefore vanish there, and hence
so does the full derivative.

Let
$\varphi(s) = f\bigl(\gamma_{\bar{\mC}, \hat{\mC}_t^{\star}}(s) \otimes \gamma_{\bar{\mR}, \hat{\mR}_t^{\star}}(s)\bigr)$
for $s \in \R$. By \eqref{eq:kron-geodesic} and Step 4, $\varphi$ is convex. By the chain rule
and the vanishing derivative, $\varphi'(0) = 0$. Hence
$\varphi(1) \geq \varphi(0) + \varphi'(0) = \varphi(0)$, that is,
$f(\hat{\mC}_t^{\star} \otimes \hat{\mR}_t^{\star}) \geq f(\bar{\mC} \otimes \bar{\mR})$.
So $(\bar{\mR}, \bar{\mC})$ is a global minimizer. By Step 4,
$\bar{\mC} \otimes \bar{\mR} = \hat{\mC}_t^{\star} \otimes \hat{\mR}_t^{\star}$.

\emph{Step 6: Convergence of the alternating scheme.}
Consider the normalized iterates for $k \geq 2$, and let
$h_k=f(\hat{\mC}_t^{(k)}\otimes\hat{\mR}_t^{(k)})$.
Write the factors before normalization as
$\mR^+=\mathcal{R}(\hat{\mC}_t^{(k)})$ and
$\mC^+=\mathcal{C}(\mR^+)$.
The normalized pair is $(a\mR^+,a^{-1}\mC^+)$, where
$a=\sqrt{\tr(\mC^+)/\tr(\mR^+)}$.
By Step 2 and invariance under this reciprocal rescaling,
\[
h_{k+1}
=f(\mC^+\otimes\mR^+)
\leq f(\hat{\mC}_t^{(k)}\otimes\mR^+)
\leq f(\hat{\mC}_t^{(k)}\otimes\hat{\mR}_t^{(k)})
=h_k.
\]
Here the intermediate row factor is the unnormalized block minimizer;
normalization is applied only after both block minimizations.
The sequence $(h_k)$ is nonincreasing and, by \eqref{eq:logdet-trace-bound}, bounded below by
$\log|\hat{\mSigma}_t'| + mn$. It therefore converges to some limit $\bar{h}$. All iterates
lie in the compact set $\mathcal{K}_{h_2}$ of Step 3.

Let $(\bar{\mR}, \bar{\mC})$ be the limit of a convergent subsequence
$(\hat{\mR}_t^{(k_j)}, \hat{\mC}_t^{(k_j)})$. This limit is a positive definite pair with
$\tr(\bar{\mR}) = \tr(\bar{\mC})$. Up to rescaling, which does not affect $f$,
\[
h_{k_j+1} = f\bigl(\mathcal{C}(\mathcal{R}(\hat{\mC}_t^{(k_j)})) \otimes \mathcal{R}(\hat{\mC}_t^{(k_j)})\bigr).
\]
By continuity of $\mathcal{R}$, $\mathcal{C}$ and $f$, the right-hand side converges to
$f\bigl(\mathcal{C}(\mathcal{R}(\bar{\mC})) \otimes \mathcal{R}(\bar{\mC})\bigr)$, and the
left-hand side converges to $\bar{h}$. Also $f(\bar{\mC} \otimes \bar{\mR}) = \bar{h}$.
Using the block-minimizing property of $\mathcal{C}$ and then that of $\mathcal{R}$,
\[
\bar{h}
= f\bigl(\mathcal{C}(\mathcal{R}(\bar{\mC})) \otimes \mathcal{R}(\bar{\mC})\bigr)
\leq f\bigl(\bar{\mC} \otimes \mathcal{R}(\bar{\mC})\bigr)
\leq f(\bar{\mC} \otimes \bar{\mR})
= \bar{h}.
\]
All inequalities are therefore equalities.

The second equality says that $\bar{\mR}$ minimizes $f(\bar{\mC} \otimes \cdot\,)$. By
uniqueness of the block minimizer, $\bar{\mR} = \mathcal{R}(\bar{\mC})$. The first
equality then reads
$f(\mathcal{C}(\bar{\mR}) \otimes \bar{\mR}) = f(\bar{\mC} \otimes \bar{\mR})$, so
$\bar{\mC} = \mathcal{C}(\bar{\mR})$ by the same uniqueness.

By Step 5, $\bar{\mC} \otimes \bar{\mR} = \hat{\mC}_t^{\star} \otimes \hat{\mR}_t^{\star}$.
Since $(\bar{\mR}, \bar{\mC})$ is normalized, Step 4 gives
$(\bar{\mR}, \bar{\mC}) = (\hat{\mR}_t^{\star}, \hat{\mC}_t^{\star})$. The iterates
lie in a compact set, and every convergent subsequence has this same limit. Hence the whole
sequence converges to $(\hat{\mR}_t^{\star}, \hat{\mC}_t^{\star})$.
\end{proof}

\subsection{Proof of \Cref{cor:mn-optimization_solution}}
\label{sec:results-from-2-proof-2}

Finally, we combine all the results above to obtain a proof of \Cref{cor:mn-optimization_solution}.

% \mnoptsolution*
\begin{restatable}[Solution to \eqref{eq:mn-sigma_optimization}]{corollary}{mnoptsolution}
\label{cor:mn-optimization_solution}
Assume $0 < \beta < 1$ and that $\hat{\mR}_{t-1}$ and $\hat{\mC}_{t-1}$ are positive
definite. The covariance projection problem \eqref{eq:mn-sigma_optimization} admits a
global solution, unique in $(\mR,\mC)$ under the normalization
$\tr(\mR)=\tr(\mC)$. Together with $\vm=\hat{\vm}_t$, it also solves
the regularized likelihood problem \eqref{eq:full-optimization_problem}
restricted to $\mSigma=\mC\otimes\mR$,
and it is unique in $(\vm, \mR, \mC)$ under the normalization
$\tr(\mR) = \tr(\mC)$. This solution is given by the mean update
\eqref{eq:mean_update}, together with the limit of the factor updates
\eqref{eq:mn-R_update} and \eqref{eq:mn-C_update} iterated to convergence.
\end{restatable}

\begin{proof}
Since $\hat{\mR}_{t-1}$ and $\hat{\mC}_{t-1}$ are positive definite,
$\hat{\mSigma}_{t-1}=\hat{\mC}_{t-1}\otimes\hat{\mR}_{t-1}\succ0$.
By~\Cref{lem:alternating}, the covariance problem
\eqref{eq:mn-sigma_optimization} has a unique minimizing Kronecker
product, and its factors are unique when their traces are equal.
The same proposition proves that the normalized factor iterates
\eqref{eq:mn-R_update}--\eqref{eq:mn-C_update} converge to these
factors from any positive definite initialization.

Apply~\Cref{prop:regularized-likelihood} with
$\mathcal{S}=\{\mC\otimes\mR:\mR\succ0,\mC\succ0\}$.
It shows that adjoining the unique mean $\vm=\hat{\vm}_t$ from
\eqref{eq:mean_update} gives a global solution of
\eqref{eq:full-optimization_problem} restricted to this
$\mathcal{S}$. Uniqueness of the mean and of the trace-normalized
factors gives uniqueness of the resulting triple.
\end{proof}

\subsection{Diagonal covariance factors and MND}
\label{sec:mnd-factor-proof}

For completeness, we derive the diagonal-factor specialization used
in~\Cref{alg:mnd-method}. Write
$\mR=\diag(\vr)$ and $\mC=\diag(\vc)$, with positive entries,
and assume the previous factors are likewise diagonal and positive.
Let
$\mX_t=\mG_t-\hat{\mM}_{t-1}$ as in~\Cref{lem:alternating}.
For fixed $\vc$, the row-dependent terms in~\eqref{eq:f-kronecker}
are
\[
n\sum_{i=1}^m\log r_i
+\sum_{i=1}^m\frac{
\beta\langle\vc^{-1},\hat\vc_{t-1}\rangle\hat r_{t-1,i}
+\beta(1-\beta)\sum_{j=1}^n X_{t,ij}^2/c_j
}{r_i},
\]
where $\vc^{-1}$ denotes the elementwise reciprocal.
By the scalar minimization in~\Cref{sec:diagonal-projection-proof},
the unique row minimizer is
\[
\vr^+
=\frac{\beta}{n}\langle\vc^{-1},\hat\vc_{t-1}\rangle\hat\vr_{t-1}
+\frac{\beta(1-\beta)}{n}(\mX_t\circ\mX_t)\vc^{-1}.
\]
Similarly, the unique column minimizer, using this new row factor, is
\[
\vc^+
=\frac{\beta}{m}\langle(\vr^+)^{-1},\hat\vr_{t-1}\rangle\hat\vc_{t-1}
+\frac{\beta(1-\beta)}{m}(\mX_t\circ\mX_t)^\top(\vr^+)^{-1}.
\]
Starting a sweep at $\vc=\hat\vc_{t-1}$ makes the first inner product
equal to $n$, giving the row update in~\Cref{alg:mnd-method}.
The column update and reciprocal trace balancing then give the
remaining factor steps. The clipped inverses in that algorithm agree
with these exact minimizers whenever the entries being inverted are
at least $\varepsilon$; clipping otherwise gives a stabilized update.
The algorithm uses one sweep per observation.

If the exact block updates are instead iterated to convergence,
the existence, uniqueness after trace normalization, and convergence
arguments in~\Cref{sec:results-from-2-proof-1} apply within the class
of diagonal factors. Indeed, diagonal matrices remain diagonal along
the geodesics used there, and the block minimizers above are positive
and unique. Thus this specialization solves the diagonal-factor
restriction of~\eqref{eq:mn-sigma_optimization}.

\subsection{Shampoo via moment matching}
\label{sec:shampoo-moment-matching}

We can estimate the covariance factors $\mR_t$ and $\mC_t$ using products of the gradients. Indeed, the relationship between the moments of the centered residuals $\mX_t = \mG_t - \mM_t$ and $\mR_t$ and $\mC_t$ is given by the following lemma.

\begin{lemma}[Matrix Normal Gram moments]
\label{lem:mn-gram-moments}
Let $\mX\sim\MN_{m,n}(0,\mR,\mC)$, where $\mR\succ0$ and
$\mC\succ0$.  Then
\begin{equation}
    \mathbb{E}\left[\mX\mX^\top\right]
    =
    \tr(\mC)\mR,
    \qquad
    \mathbb{E}\left[\mX^\top\mX\right]
    =
     \tr(\mR)\mC.
    \label{eq:appendix-gram-moments}
\end{equation}
Thus the expected Gram matrices are proportional to the covariance factors.
% Consequently, the row and column Gram matrices estimate the two covariance
% factors up to global scale constants.
\end{lemma}

\begin{proof}
By the definition of the Matrix Normal distribution,  $\vec(\mX) \sim \Normal(0, \mC\otimes \mR)$. Thus, $\E{\vec(\mX)\vec(\mX)^{\top}} = \mC \otimes \mR$, and
$\mathbb{E}[\mX_{ij}\mX_{k\ell}]
=\mR_{ik}\mC_{j\ell}$.  Therefore
\[
    \left[\mathbb{E}(\mX\mX^\top)\right]_{ik}
    =
    \sum_{j=1}^n
    \mathbb{E}[\mX_{ij}\mX_{kj}]
    =
    \sum_{j=1}^n
    \mR_{ik}\mC_{jj}
    =
    \tr(\mC)\mR_{ik}.
\]
Similarly,
\[
    \left[\mathbb{E}(\mX^\top\mX)\right]_{j\ell}
    =
    \sum_{i=1}^m
    \mathbb{E}[\mX_{ij}\mX_{i\ell}]
    =
    \sum_{i=1}^m
    \mR_{ii}\mC_{j\ell}
    =
   \tr(\mR)\mC_{j\ell}.
\]
\end{proof}

Note this relationship requires access to the population mean of the gradient $\mM_t$, which is infeasible in practice. We can use the identities~\eqref{eq:appendix-gram-moments} to build estimates of our factor matrices $\hat{\mR}_t$ and $\hat{\mC}_t$ as follows.

Let $\hat{\mM}_t$ be our current estimate of the mean gradient (often using momentum) and let $\hat{\mX}_t := \mG_t - \hat{\mM}_t$ be our estimate of the centered gradient. To estimate the expectations in~\eqref{eq:appendix-gram-moments} we can use an exponential moving average given by
\begin{align} \label{eq:shampoofactor}
    \tilde{\mR}_t 
   & =
    \gamma\tilde{\mR}_{t-1} 
    +
    (1-\gamma)\hat{\mX}_t\hat{\mX}_t^\top,
    \qquad
    \tilde{\mC}_t
    =
    \gamma\tilde{\mC}_{t-1}
    +
    (1-\gamma)\hat{\mX}_t^\top \hat{\mX}_t.
\end{align} 

  If one uses $\mG_t$ instead of the centered
residual $\hat{\mX}_t$, then the same formulas estimate uncentered second moments
rather than the covariance of the Matrix Normal noise.  This is often the
choice made by adaptive optimizers such as Shampoo~\citep{gupta2018shampoo}.

There are a few disadvantages to the above approach. First, replacing the population mean by $\hat{\mM}_t$ changes the Gram moments; the moment identities do not by themselves account for the error in this mean estimate. Furthermore, \eqref{eq:shampoofactor} does not correct for the scale of the factors. Indeed, the identity~\eqref{eq:appendix-gram-moments} that gives the $\mR_t$ factor depends on the trace of $\mC_t$, which is not reflected in the updates~\eqref{eq:shampoofactor}. Our approach of modeling $(\hat{\mM}_t, \hat{\mR}_t, \hat{\mC}_t)$ jointly, on the other hand, can more properly address these issues.

\section{Derivations and proofs of~\Cref{sec:mn-muon}}
\label{app:mn-muon}

Throughout this section, $0 < \beta < 1$, the previous factors $\hat{\mR}_{t-1}$ and
$\hat{\mC}_{t-1}$ are positive definite, and
$\hat{\mSigma}_{t-1} = \hat{\mC}_{t-1} \otimes \hat{\mR}_{t-1}$. We write
$p = \min\{m,n\}$, so that $\|\mD\|_F^2 = p$ for every $\mD \in \St(m,n)$.

\subsection{Prior}
\label{sec:results-from-4-prior}

The Stiefel manifold $\St(m,n)$ is compact. We equip it with its uniform measure, that is, the
unique probability measure invariant under $\mD \mapsto \mU\mD\mV$ for orthogonal
$\mU \in \R^{m \times m}$ and $\mV \in \R^{n \times n}$. For $\mF \in \R^{m \times n}$,
the Matrix Langevin distribution $\mathrm{ML}_{m,n}(\mF)$ has density
\begin{equation}
p(\mD \mid \mF) = \frac{\exp(\langle \mF, \mD \rangle)}{c(\mF)},
\qquad
c(\mF) = \int_{\St(m,n)} \exp(\langle \mF, \mD' \rangle)\, d\mD',
\label{eq:ml-density}
\end{equation}
with respect to the uniform measure. The normalizing constant is finite and positive, since the
integrand is continuous and positive on a compact set. 

Suppose $\mF$ has full row or column rank, i.e., $\operatorname{rank}(\mF)=p$.
Let $\mF = \mU\mS\mV^\top$ be a compact singular value decomposition, with
$\mU \in \R^{m \times p}$ and $\mV \in \R^{n \times p}$ having orthonormal columns and
$\mS = \diag(s_1, \dots, s_p)$ with $s_1 \geq \dots \geq s_p > 0$. We write
$\Polar(\mF) = \mU\mV^\top$, which lies in $\St(m,n)$.
This polar factor does not depend on the chosen decomposition. For $m \geq n$,
$\Polar(\mF) = \mF(\mF^\top\mF)^{-1/2}$, and for $m < n$, $\Polar(\mF)=(\mF\mF^\top)^{-1/2}\mF$.

\begin{lemma}[Mode of the matrix Langevin distribution]
\label{lem:ml-mode}
Let $\mF \in \R^{m \times n}$ have rank $p=\min\{m,n\}$.
\begin{enumerate}[(i)]
  \item $\|\mF\|_* = \underset{\mD \in \St(m,n)}{\max}\langle \mF, \mD \rangle$, where
  $\|\mF\|_* = \sum_{i=1}^{p} s_i$ is the nuclear norm. A maximum is attained at
  $\Polar(\mF)$.
  \item $\Polar(\mF)$ is the unique maximizer.
  \item $\Polar(a\mF) = \Polar(\mF)$ for every $a > 0$.
\end{enumerate}
Consequently, since \eqref{eq:ml-density} is increasing in $\langle \mF, \mD \rangle$,
$\Polar(\mF)$ is the unique mode of $\mathrm{ML}_{m,n}(\mF)$.
\end{lemma}

\begin{proof}
Using $\mU^\top\mU = \mV^\top\mV = \mI_p$, we have
$\langle \mF, \Polar(\mF) \rangle
= \langle \mF, \mU\mV^\top \rangle = \tr(\mS) = \|\mF\|_*$.
For the upper bound, let $\mD \in \St(m,n)$ and $\mZ = \mU^\top\mD\mV \in \R^{p \times p}$.
All singular values of $\mD$ equal one, so $\|\mZ\|_{\rm op} \leq 1$, and in particular each
column of $\mZ$ has Euclidean norm at most one. Hence
\[
\langle \mF, \mD \rangle
= \tr(\mV\mS\mU^\top\mD)
= \tr(\mS\mZ)
= \sum_{i=1}^{p} s_i Z_{ii}
\leq \sum_{i=1}^{p} s_i .
\]

Since all $s_i > 0$, equality in the bound forces $Z_{ii} = 1$ for all $i$. Since each column
of $\mZ$ has norm at most one, this forces $\mZ = \mI_p$.

Assume first $m \geq n$, so $p = n$ and $\mV$ is orthogonal. From
$\mU^\top\mD\mV = \mI_n$ we get $\mU^\top\mD = \mV^\top$. Let
$\mP = \mI_m - \mU\mU^\top$. Then $\mD = \mU\mV^\top + \mP\mD$, and,
since $\mU^\top\mP = 0$ and $\mP^2 = \mP$,
\[
\mI_n = \mD^\top\mD = \mV\mV^\top + (\mP\mD)^\top(\mP\mD)
= \mI_n + (\mP\mD)^\top(\mP\mD).
\]
Hence $\mP\mD = 0$ and $\mD = \mU\mV^\top$. If $m < n$, apply the same
argument to $\mF^\top$ and $\mD^\top \in \St(n,m)$, using
$\langle \mF, \mD \rangle = \langle \mF^\top, \mD^\top \rangle$.

Finally,
$a\mF = \mU(a\mS)\mV^\top$ is a singular value decomposition of $a\mF$ with the
same $\mU$ and $\mV$. Since the polar factor is unique, this gives
$\Polar(a\mF)=\Polar(\mF)$.
\end{proof}

\subsection{Posterior}
\label{sec:results-from-4-posterior}

Recall that $\mG \sim \MN_{m,n}(\mM, \mR, \mC)$ means
$\vec(\mG) \sim \Normal(\vec(\mM), \mC \otimes \mR)$. Since
$|\mC \otimes \mR| = |\mC|^m|\mR|^n$, its density is
\begin{equation}
\MN_{m,n}(\mG \mid \mM, \mR, \mC)
= (2\pi)^{-\frac{mn}{2}}|\mR|^{-\frac{n}{2}}|\mC|^{-\frac{m}{2}}
\exp\!\left(-\frac{1}{2}\bigl\|\mR^{-1/2}(\mG - \mM)\mC^{-1/2}\bigr\|_F^2\right).
\label{eq:mn-density}
\end{equation}

\begin{lemma}[Conditional posterior of the direction]
\label{lem:ml-posterior}
Fix positive definite $\mR_t$ and $\mC_t$. Under the model
$\mG_t \mid \mD_t, \mR_t, \mC_t \sim \MN_{m,n}(\mR_t^{1/2}\mD_t\mC_t^{1/2}, \mR_t, \mC_t)$
and the priors \eqref{eq:mn-muon-direction-prior}--\eqref{eq:mn-muon-covariance-prior}, we have
\[
\mD_t \mid \mG_t, \mR_t, \mC_t, \widetilde{\mM}_{t-1}
\sim \mathrm{ML}_{m,n}\left(\frac{1}{1-\beta}\widetilde{\mM}_t\right),
\]
with $\widetilde{\mM}_t$ as in \eqref{eq:white-momentum}. If $\widetilde{\mM}_t$ has full rank,
its unique mode is $\hat{\mD}_t = \Polar(\widetilde{\mM}_t)$.
\end{lemma}

\begin{proof}
The prior kernel factorizes as the product of
\eqref{eq:mn-muon-direction-prior} and \eqref{eq:mn-muon-covariance-prior}.
For fixed factors, the factor kernel is constant in $\mD_t$, so the
conditional direction density is proportional to the likelihood times
\eqref{eq:mn-muon-direction-prior}. This density is proper because
$\St(m,n)$ is compact.

With $\mM = \mR_t^{1/2}\mD_t\mC_t^{1/2}$, the exponent of
the likelihood is
\[
-\frac{1}{2}\bigl\|\mR_t^{-1/2}\mG_t\mC_t^{-1/2} - \mD_t\bigr\|_F^2
= -\frac{1}{2}\bigl\|\mR_t^{-1/2}\mG_t\mC_t^{-1/2}\bigr\|_F^2
+ \bigl\langle \mR_t^{-1/2}\mG_t\mC_t^{-1/2}, \mD_t \bigr\rangle
- \frac{p}{2}.
\]
As a function of $\mD_t$, the likelihood is thus proportional to
$\exp\bigl(\langle \mR_t^{-1/2}\mG_t\mC_t^{-1/2}, \mD_t \rangle\bigr)$. Multiplying
by the kernel of \eqref{eq:mn-muon-direction-prior} gives
\begin{align*}
p(\mD_t \mid \mG_t, \mR_t, \mC_t, \widetilde{\mM}_{t-1})
&\propto \exp\!\left(\left\langle \mR_t^{-1/2}\mG_t\mC_t^{-1/2}
+ \frac{\beta}{1-\beta}\widetilde{\mM}_{t-1}, \mD_t \right\rangle\right)\\
&= \exp\!\left(\left\langle \frac{1}{1-\beta}\widetilde{\mM}_t, \mD_t \right\rangle\right),
\end{align*}
since
$\mR_t^{-1/2}\mG_t\mC_t^{-1/2} + \frac{\beta}{1-\beta}\widetilde{\mM}_{t-1}
= \frac{1}{1-\beta}\bigl[(1-\beta)\mR_t^{-1/2}\mG_t\mC_t^{-1/2} + \beta\widetilde{\mM}_{t-1}\bigr]$.
This is the kernel of $\mathrm{ML}_{m,n}\bigl(\frac{1}{1-\beta}\widetilde{\mM}_t\bigr)$.
If $\widetilde{\mM}_t$ has full rank, \Cref{lem:ml-mode} gives its unique mode as
$\Polar\bigl(\frac{1}{1-\beta}\widetilde{\mM}_t\bigr)=\Polar(\widetilde{\mM}_t)$.
\end{proof}

\subsection{Discounting}
\label{sec:results-from-4-discounting}

By discounting a density $p$ with factor $\beta$ we mean replacing it with the density
proportional to $p^{\beta}$, whenever the latter is integrable. This is exponential forgetting:
it downweights all accumulated information by the same factor $\beta$.

\paragraph{Direction.}
The density \eqref{eq:ml-density} satisfies
$p(\mD \mid \mF)^{\beta} \propto \exp(\langle \beta\mF, \mD \rangle)$, so discounting
$\mathrm{ML}_{m,n}(\mF)$ gives $\mathrm{ML}_{m,n}(\beta\mF)$. This is always a proper
distribution, since $\St(m,n)$ is compact. Applying this to the conditional posterior of
\Cref{lem:ml-posterior}, with the factors set to their estimates $(\hat{\mR}_t, \hat{\mC}_t)$,
gives
\[
\mathrm{ML}_{m,n}\left(\frac{\beta}{1-\beta}\widetilde{\mM}_t\right),
\]
which is \eqref{eq:mn-muon-direction-prior} with $t-1$ replaced by $t$, that is, the direction
prior at time $t+1$.

The weights admit a pseudo-count interpretation. The prior carries weight $\frac{\beta}{1-\beta}$,
the new observation adds weight one, and the posterior carries weight
$1/(1-\beta) = \beta/(1-\beta) + 1$. Discounting multiplies the posterior weight by
$\beta$, which restores $\beta/(1-\beta)$. The effective number of pseudo-observations is
therefore constant over time.

\paragraph{Covariance.}
The joint posterior of the factors is not of the form \eqref{eq:mn-muon-covariance-prior},
because the raw mean depends on the factors. The covariance prior at time $t+1$ is instead
obtained by discounting the kernel
$\exp\bigl(-\frac{1}{1-\beta}\KL\bigl(\Normal(0, \hat{\mSigma}_t) \,\bigm\|\, \Normal(0, \mSigma)\bigr)\bigr)$.
This kernel has mode $\hat{\mSigma}_t$. It also carries the same log-determinant weight
$\frac{1}{2}\frac{1}{1-\beta}$ as the posterior kernel, which combines weight
$\frac{1}{2}\frac{\beta}{1-\beta}$ from the prior with weight $\frac{1}{2}$ from the likelihood.
Its $\beta$-power is exactly \eqref{eq:mn-muon-covariance-prior} with $t-1$ replaced by $t$.
Replacing the posterior over the factors with a kernel of the prior's form, centered at the
current estimate, plays the role that the Kronecker projection plays in \Cref{sec:mn}.
The analogous Normal-Inverse-Wishart discounting for \Cref{sec:mn}
is derived in~\Cref{sec:results-from-2-joint-posterior}.
% \rob{Note: removed the repeated Normal-Inverse-Wishart discounting derivation; see the reference above.}

\subsection{Optimization problem}
\label{sec:results-from-4-optimization}

Write the posterior kernel as
\[
K_t(\mD,\mR,\mC)
:= p(\mG_t \mid \mD, \mR, \mC)\,p(\mD \mid \widetilde{\mM}_{t-1})\,
p(\mR, \mC \mid \hat{\mR}_{t-1}, \hat{\mC}_{t-1}).
\]
We call any maximizer of $K_t$ over $\mD \in \St(m,n)$ and positive definite
$\mR,\mC$ a joint posterior mode. This definition does not require the kernel to be
normalizable.

Let $\Phi(\mD,\mR,\mC)$ denote the objective in
\eqref{eq:mn-muon-optimization_problem}. With $\mSigma=\mC\otimes\mR$, $\hat{\mSigma}_{t-1} =\hat{\mC}_{t-1}\otimes\hat{\mR}_{t-1}$  and
$\vg_t=\vec(\mG_t)$, define
\begin{align}
\mathcal{J}(\mD,\mR,\mC)
&:= \log|\mSigma|
+ \tr\Bigl(\mSigma^{-1}\bigl[\beta\hat{\mSigma}_{t-1}
+ (1-\beta)\vg_t\vg_t^\top\bigr]\Bigr)\nonumber\\
&\quad -2\bigl\langle \beta\widetilde{\mM}_{t-1}
+ (1-\beta)\mR^{-1/2}\mG_t\mC^{-1/2},\mD\bigr\rangle.\label{eq:def-J-DRC}
\end{align}

\Needspace{9\baselineskip}
\begin{proposition}[MAP formulation]
\label{prop:mn-muon-map}
For every $\mD \in \St(m,n)$ and positive definite $\mR$ and $\mC$,
\[
-\log K_t(\mD,\mR,\mC)
= \Phi(\mD,\mR,\mC)+\mathrm{const}
= \frac{\mathcal{J}(\mD,\mR,\mC)}{2(1-\beta)}+\mathrm{const}.
\] 
Consequently, the joint posterior modes are exactly the minimizers of $\Phi$
in \eqref{eq:mn-muon-optimization_problem},
equivalently of $\mathcal{J}$.
\end{proposition}

\begin{proof}
All additive constants below are independent of
$(\mD,\mR,\mC)$.
Each weighted KL penalty in $\Phi$ is a negative log-prior, up to a constant.
For the direction, the Gaussian KL formula gives
\begin{align*}
&\frac{\beta}{1-\beta}\KL\!\left(\MN_{m,n}(\widetilde{\mM}_{t-1},\mI_m,\mI_n)
\,\middle\|\,\MN_{m,n}(\mD,\mI_m,\mI_n)\right)\\
&\qquad = \frac{\beta}{2(1-\beta)}\|\widetilde{\mM}_{t-1}-\mD\|_F^2\\
&\qquad = -\frac{\beta}{1-\beta}\langle\widetilde{\mM}_{t-1},\mD\rangle+\mathrm{const}\\
&\qquad = -\log p(\mD\mid\widetilde{\mM}_{t-1})+\mathrm{const},
\end{align*}
where we used $\|\mD\|_F^2=p$ and the Matrix Langevin
prior~\eqref{eq:mn-muon-direction-prior}.
For the covariance, the same KL formula gives
\begin{align*}
&\frac{\beta}{1-\beta}\KL\!\left(\MN_{m,n}(0,\hat{\mR}_{t-1},\hat{\mC}_{t-1})
\,\middle\|\,\MN_{m,n}(0,\mR,\mC)\right)\\
&\qquad = \frac{\beta}{2(1-\beta)}\left[
\tr(\mSigma^{-1}\hat{\mSigma}_{t-1})+\log|\mSigma|
\right]+\mathrm{const}\\
&\qquad = -\log p(\mR,\mC\mid\hat{\mR}_{t-1},\hat{\mC}_{t-1})+\mathrm{const},
\end{align*}
where we used the prior kernel~\eqref{eq:mn-muon-covariance-prior}.
The terms involving only $\hat{\mSigma}_{t-1}$ and the dimensions are constant.

The first term in $\Phi$ is the negative log-likelihood
$-\log p(\mG_t\mid\mD,\mR,\mC)$.
Since $K_t$ is the product of the likelihood and the two priors, 
these identities give
\[
-\log K_t(\mD,\mR,\mC)  
= \Phi(\mD,\mR,\mC)+\mathrm{const}.
\]

To obtain the identity with $\mathcal{J}$, set $\mZ=\mR^{-1/2}\mG_t\mC^{-1/2}$.
Using the mean parameterization
$\mM=\mR^{1/2}\mD\mC^{1/2}$ in
\eqref{eq:mn-muon-optimization_problem}, we have
$\mR^{-1/2}(\mG_t-\mM)\mC^{-1/2}=\mZ-\mD$.
Thus, by \eqref{eq:mn-density},
\[
-2\log p(\mG_t\mid\mD,\mR,\mC)
= \log|\mSigma|+\|\mZ\|_F^2-2\langle\mZ,\mD\rangle+p
+mn\log(2\pi).
\]
Moreover, the trace and vectorization identities give
\[
\|\mZ\|_F^2
= \tr(\mR^{-1}\mG_t\mC^{-1}\mG_t^\top)
= \vg_t^\top\mSigma^{-1}\vg_t
= \tr(\mSigma^{-1}\vg_t\vg_t^\top).
\]
Combining the likelihood and the two KL penalties yields
\begin{align*}
2(1-\beta)\Phi(\mD,\mR,\mC)
&= \log|\mSigma|
+ \tr\Bigl(\mSigma^{-1}\bigl[\beta\hat{\mSigma}_{t-1}
+(1-\beta)\vg_t\vg_t^\top\bigr]\Bigr)\\
&\qquad -2\langle\beta\widetilde{\mM}_{t-1}+(1-\beta)\mZ,\mD\rangle
+\mathrm{const}\\
&= \mathcal{J}(\mD,\mR,\mC)+\mathrm{const}.
\end{align*}
Here the log-determinant coefficient is $(1-\beta)+\beta=1$.
This proves the second equality. Finally, $K_t>0$, $-\log$ is strictly decreasing,
and $2(1-\beta)>0$, so the three optimization problems have the same solutions.
\end{proof}

\subsection{Proof of \Cref{prop:mn-muon-blocks}}
\label{sec:results-from-4-solution-1}

\mnmuonblocks*

\begin{proof}
By \Cref{prop:mn-muon-map}, it suffices to minimize $\mathcal{J}$ over each block.
We use $\mathrm{const}$ for terms independent of the block being optimized.

First, recall $\mathcal{J}$ in \eqref{eq:def-J-DRC} depends on the factors only 
through $\mSigma=\mC\otimes\mR$,
so the rescaling $(\mR,\mC)\mapsto(a^{-1}\mR,a\mC)$, $a>0$, leaves
$\mathcal{J}$ unchanged.

\emph{(i) Direction.}
Fix $\mR,\mC\succ0$ and set
\[
\mF=\beta\widetilde{\mM}_{t-1}
+(1-\beta)\mR^{-1/2}\mG_t\mC^{-1/2}.
\]
Minimizing \eqref{eq:def-J-DRC} over the direction $\mD$ reduces to
\[
\underset{\mD\in\St(m,n)}{\mbox{minimize}}\;
-2\langle\mF,\mD\rangle.
\]
For full rank $\mF$, \Cref{lem:ml-mode} gives the unique minimizer $\Polar(\mF)$.

\emph{(ii) Factors.}
Fix $\mD$ and $\mC\succ0$. Define
\begin{align*}
\mQ
&=\beta\tr(\mC^{-1}\hat{\mC}_{t-1})\hat{\mR}_{t-1}
+(1-\beta)\mG_t\mC^{-1}\mG_t^\top,\\
\mL
&=(1-\beta)\operatorname{sym}(\mG_t\mC^{-1/2}\mD^\top),
\end{align*}
where $\operatorname{sym}(\mY)=(\mY+\mY^\top)/2$.
Using
\[
\log|\mSigma| = m\log|\mC| + n\log|\mR|, \qquad 
\tr\bigl(\mSigma^{-1}(\hat{\mC}_{t-1} \otimes \hat{\mR}_{t-1})\bigr)
= \tr(\mC^{-1}\hat{\mC}_{t-1})\tr(\mR^{-1}\hat{\mR}_{t-1}),
\] 
and
$\tr(\mSigma^{-1}\vg_t\vg_t^\top) = \tr(\mR^{-1}\mG_t\mC^{-1}\mG_t^\top)$,
the terms of $\mathcal{J}$~\eqref{eq:def-J-DRC}  that depend on $\mR$ are
\[
\mathcal{J}(\mD,\mR,\mC)
=n\log|\mR|+\tr(\mR^{-1}\mQ)
-2\tr(\mR^{-1/2}\mL)+\mathrm{const},
\]
since $\mR^{-1/2}$ is symmetric, only the symmetric part of $\mG_t\mC^{-1/2}\mD^\top$ 
contributes to the trace.

The change of variable $\mA=\mR^{-1/2}$ is a bijection on the positive
definite cone, with inverse $\mR=\mA^{-2}$. Thus the row block is equivalent to
\[
\underset{\mA\succ0}{\mbox{minimize}}\quad
\phi(\mA)
:=-n\log|\mA|+\frac12\tr(\mA\mQ\mA)-\tr(\mA\mL),
\]
since $\mathcal{J}(\mD,\mA^{-2},\mC)=2\phi(\mA)+\mathrm{const}$.
The matrix $\mQ$ is positive definite: its first term is positive definite
because $\beta>0$, $\hat{\mR}_{t-1}\succ0$, and
$\tr(\mC^{-1}\hat{\mC}_{t-1})>0$; its second term is positive semidefinite.

We show that $\phi$ has a unique minimizer. For any nonzero symmetric $\bm{H}$,
\[
\left.\frac{d^2}{ds^2}\phi(\mA+s\bm{H})\right|_{s=0}
=n\bigl\|\mA^{-1/2}\bm{H}\mA^{-1/2}\bigr\|_F^2
+\tr(\bm{H}\mQ\bm{H})>0,
\]
so $\phi$ is strictly convex. To establish existence, let $a_1,\ldots,a_m>0$
be the eigenvalues of $\mA$. Then
\[
\phi(\mA)\geq\sum_{i=1}^m h(a_i),
\qquad
h(a)=-n\log a+\frac{\lambda_{\min}(\mQ)}{2}a^2-\|\mL\|_F a.
\]
The function $h$ is bounded below and tends to $+\infty$ as $a\to0^+$ or
$a\to\infty$. Hence every sublevel set of $\phi$ restricts the eigenvalues of
$\mA$ to a compact interval in $(0,\infty)$ and is compact.
Continuity gives existence; strict convexity gives uniqueness.

Differentiating along symmetric directions, the necessary and sufficient
optimality condition is
\begin{equation}
\frac12(\mQ\mA+\mA\mQ)-n\mA^{-1}=\mL.
\label{eq:mn-muon-foc}
\end{equation}
Its unique positive definite solution gives $\mR^+=\mA^{-2}$.
If $\mL=0$, the solution is $\mA=\sqrt{n}\,\mQ^{-1/2}$, so
$\mR^+=\mQ/n$.

The column block also has a unique minimizer by the same argument, transposing $\mG_t$ and $\mD$ and exchanging the row and column factors.

\emph{(iii) Diagonal factors.}
Suppose the previous and current factors are diagonal, with
$\mR=\diag(\vr)$ and $\mC=\diag(\vc)$, where $\vr,\vc>0$ entrywise.
For fixed $\mD$ and $\vc$, let $\mQ_t$ be as in \eqref{eq:mn-muon-statistics},
set $\mB_t=(1-\beta)(\mG_t\circ\mD)$, and define
$u_i=(\mQ_t\vc^{-1})_i>0$ and $v_i=(\mB_t\vc^{-1/2})_i$,
with vector powers taken entrywise.
The row objective separates over the entries of $\vr$:
\[
\mathcal{J}(\mD,\diag(\vr),\diag(\vc))
=\sum_i\left(n\log r_i+\frac{u_i}{r_i}
-\frac{2v_i}{\sqrt{r_i}}\right)+\mathrm{const}.
\]
Setting $y=r_i^{-1/2}$ reduces each scalar problem to minimizing
$-n\log y+\tfrac12u_i y^2-v_i y$ over $y>0$.
This function is strictly convex, and its derivative vanishes at the unique
positive root of $u_i y^2-v_i y-n=0$, giving
\[
r_i^+=y^{-2}=\left(\frac{\sqrt{v_i^2+4nu_i}-v_i}{2n}\right)^2
=\Psi_n(u_i,v_i),
\]
where $\Psi_q$ is defined in \eqref{eq:mn-muon-variance-root}.
The column update follows by exchanging rows and columns and replacing $n$ by $m$;
in a row-then-column sweep, it uses the updated $\vr^+$.
\end{proof}

\subsection{Proof of \Cref{prop:mn-muon-convergence}}
\label{sec:results-from-4-solution-2}

\mnmuonconvergence*

\begin{proof}
We work with $\mathcal{J}$ in \eqref{eq:def-J-DRC} which is equivalent to
\eqref{eq:mn-muon-optimization_problem} as per \Cref{prop:mn-muon-map}.
Start from any $\mD^{(1)}\in\St(m,n)$ and $\mR^{(1)},\mC^{(1)}\succ0$.
One sweep consists of:
\begin{enumerate}[(i)]
  \item \emph{Minimize over the factors:}
  \[
  \mR^{(k+1)}_\circ=\argmin{\mR\succ0}\mathcal{J}(\mD^{(k)},\mR,\mC^{(k)}),
  \qquad
  \mC^{(k+1)}_\circ=\argmin{\mC\succ0}\mathcal{J}(\mD^{(k)},\mR^{(k+1)}_\circ,\mC).
  \]
  \item \emph{Balance the traces:}
  \[
  (\mR^{(k+1)},\mC^{(k+1)})
  =(a_k^{-1}\mR^{(k+1)}_\circ,a_k\mC^{(k+1)}_\circ),
  \qquad
  a_k=\sqrt{\frac{\tr(\mR^{(k+1)}_\circ)}{\tr(\mC^{(k+1)}_\circ)}}.
  \]
  \item \emph{Minimize over the direction:}
  choose $\mD^{(k+1)}\in\argmin{\mD\in\St(m,n)}
  \mathcal{J}(\mD,\mR^{(k+1)},\mC^{(k+1)})$.
\end{enumerate}
The factor minimizers are unique and positive definite by \Cref{prop:mn-muon-blocks}.
The direction minimizer exists because $\St(m,n)$ is compact.

\emph{Step 1: Descent and convergence of the objective.}
Write $h_k=\mathcal{J}(\mD^{(k)},\mR^{(k)},\mC^{(k)})$.
Trace balancing preserves $\mC\otimes\mR$ and hence $\mathcal{J}$.
Since each minimization can only decrease the objective,
\[
h_{k+1}
\leq\mathcal{J}(\mD^{(k)},\mR^{(k+1)}_\circ,\mC^{(k+1)}_\circ)
\leq\mathcal{J}(\mD^{(k)},\mR^{(k+1)}_\circ,\mC^{(k)})
\leq h_k.
\]
To bound $h_k$ below, set $\mS=\beta\hat{\mSigma}_{t-1}\succ0$
and $\mZ=\mR^{-1/2}\mG_t\mC^{-1/2}$.
Using $\|\mD\|_F^2=p$, we can write \eqref{eq:def-J-DRC} as
\[
\mathcal{J}(\mD,\mR,\mC)
=\log|\mSigma|+\tr(\mSigma^{-1}\mS)
+(1-\beta)\|\mZ-\mD\|_F^2
-(1-\beta)p-2\beta\langle\widetilde{\mM}_{t-1},\mD\rangle,
\]
where $\mSigma=\mC\otimes\mR$.
Cauchy--Schwarz therefore gives
\begin{equation}
\mathcal{J}(\mD,\mR,\mC)
\geq\log|\mSigma|+\tr(\mSigma^{-1}\mS)-c_0,
\label{eq:mn-muon-lower-bound}
\end{equation}
with $c_0=(1-\beta)p+2\beta\sqrt{p}\,\|\widetilde{\mM}_{t-1}\|_F$.
By \eqref{eq:logdet-trace-bound}, the right-hand side is at least
$\log|\mS|+mn-c_0$. Thus $\{h_k\}_{k=1}^\infty$ is nonincreasing and bounded below, so it
converges to a finite limit $\bar h$.

\emph{Step 2: Compactness of the iterates.}
Let $\mSigma^{(k)}=\mC^{(k)}\otimes\mR^{(k)}$.
Since $h_k\leq h_1$, \eqref{eq:mn-muon-lower-bound} implies
\[
\log|\mSigma^{(k)}|+\tr\bigl((\mSigma^{(k)})^{-1}\mS\bigr)\leq h_1+c_0.
\]
In \eqref{eq:logdet-trace-bound}, each term $\lambda-\log\lambda$ is at least $1$
and tends to infinity as $\lambda\to0^+$ or $\lambda\to\infty$.
Thus the eigenvalues of $\mS^{1/2}(\mSigma^{(k)})^{-1}\mS^{1/2}$ stay in a fixed
compact interval in $(0,\infty)$, which gives constants $0<\ell\leq u$ such that
$\ell\mI_{mn}\preceq\mSigma^{(k)}\preceq u\mI_{mn}$ for all $k=1, 2, \ldots$.

For $k\geq2$, trace balancing gives
\[
\tr(\mR^{(k)})=\tr(\mC^{(k)})
=\sqrt{\tr(\mSigma^{(k)})}\leq\sqrt{mn\,u}.
\]
Also, $\lambda_{\min}(\mR^{(k)})\lambda_{\min}(\mC^{(k)})
=\lambda_{\min}(\mSigma^{(k)})\geq\ell$.
Hence every eigenvalue of both factors lies in
$[\ell/\sqrt{mn\,u},\sqrt{mn\,u}]$.
Since $\St(m,n)$ is compact, all iterates with $k\ge2$ lie in a compact set.
% This set remains compact when the initial iterate is included.
In particular, limit points exist and their factors are positive definite.

\emph{Step 3: Optimality of each block at every limit point.}
Denote the unique factor minimizers by $\mathcal{R}(\mD,\mC)$ and
$\mathcal{C}(\mD,\mR)$. Both maps are continuous: the Hessian of $\phi$ in the
proof of \Cref{prop:mn-muon-blocks} is positive definite, and the coefficients
$\mQ$ and $\mL$ in \eqref{eq:mn-muon-foc} vary smoothly with the fixed blocks.
The implicit function theorem therefore makes the solution $\mA$, and hence
$\mathcal{R}(\mD,\mC)=\mA^{-2}$, locally smooth; uniqueness identifies these local
solutions with the row minimizer. The column argument is identical.

Take a convergent subsequence
$(\mD^{(k_j)},\mR^{(k_j)},\mC^{(k_j)})\to(\bar{\mD},\bar{\mR},\bar{\mC})$.
Continuity of the factor maps gives
\[
\mR^{(k_j+1)}_\circ\to\mR^+:=\mathcal{R}(\bar{\mD},\bar{\mC}),
\qquad
\mC^{(k_j+1)}_\circ\to\mC^+:=\mathcal{C}(\bar{\mD},\mR^+).
\]
Both intermediate objective values in Step 1 lie between $h_{k_j+1}$ and $h_{k_j}$,
so they converge to $\bar h$. Sending to the limit gives
\[
\mathcal{J}(\bar{\mD},\bar{\mR},\bar{\mC})
=\mathcal{J}(\bar{\mD},\mR^+,\bar{\mC})
=\mathcal{J}(\bar{\mD},\mR^+,\mC^+)=\bar h.
\]
The first equality and uniqueness of the row minimizer imply $\bar{\mR}=\mR^+$.
Substituting this into the second equality and using uniqueness of the column
minimizer gives $\bar{\mC}=\mC^+$. Thus both factor blocks are optimal at the limit.

For the direction block, each $k\geq2$ satisfies
\[
\mathcal{J}(\mD^{(k)},\mR^{(k)},\mC^{(k)})
\leq\mathcal{J}(\mD,\mR^{(k)},\mC^{(k)}),
\quad\text{for all }\mD\in\St(m,n).
\]
Taking the limit along the same subsequence proves that $\bar{\mD}$ minimizes
$\mathcal{J}(\cdot,\bar{\mR},\bar{\mC})$, 
although the minimizing direction need not be unique.

Finally, $\mathcal{J}$ is smooth on the feasible product manifold.
Optimality of each block makes its derivative vanish along every feasible tangent
direction in that block. Since these tangent spaces form the tangent space of the
product manifold, the full derivative vanishes at
$(\bar{\mD},\bar{\mR},\bar{\mC})$. Thus every limit point is a critical point.
\end{proof}

\section{Numerical instability of Newton--Schulz}
\label{sec:Newtonschulz}

To invert the covariance factors in MN, we use the Newton--Schulz iteration
\begin{equation}\label{eq:def-newton-shulz-inv}
    \mX_{k+1}=\mX_k(2\mI-\mA\mX_k),
    \qquad
    \mX_0=\frac{\mA}{\|\mA\|_F^2},
\end{equation}
where $\mA$ is symmetric positive definite.
Although this iteration converges quadratically to $\mA^{-1}$ in exact arithmetic~\cite{highamfunctions}, finite precision can limit its accuracy or cause divergence. Two numerical examples illustrate these difficulties.

For each working precision in $\{\text{BF16},\text{FP32},\text{FP64}\}$,
we apply Newton--Schulz iteration \eqref{eq:def-newton-shulz-inv} to $\mA+\varepsilon\mI$ formed in that precision, with $\varepsilon=10^{-8}$ in both examples.
We evaluate the unnormalized residual
\[
    r_k=\|\mI-(\mA+\varepsilon\mI)\mX_k\|_F
\]
in \text{FP64} against the common target.

\paragraph{GPT-small gradient-Gram EMA.}
We train GPT-small on FineWeb for $400$ Muon updates with global batch $512$ and context length $1024$, using BF16 autocast with FP32 parameters and optimizer state.
We use seed $42$, Muon momentum $0.95$, zero weight decay, and peak learning rate $0.03$, with $40$ warmup, $320$ constant-rate, and $40$ decay updates. 

Let $\mG_t\in\R^{2048\times768}$ be the gradient of the first transformer block's MLP input projection, after gradient accumulation and before clipping.
Along this training trajectory, we accumulate in \text{FP64}
\[
    \mN_t=\beta\mN_{t-1}+(1-\beta)\mG_t^\top\mG_t,
    \qquad \mN_0=0,
\]
and set $\mA=\mN_{400}$.
This $768\times768$ matrix has $\kappa_2(\mA)\approx2\times10^6$; adding $\varepsilon=10^{-8}$ gives $\kappa_2(\mA+\varepsilon\mI)\approx 4.7\times10^5$.
As shown in~\Cref{fig:newtonschulz}, at iteration $42$ the FP32 residual is approximately $7.9\times10^{-2}$, while FP64 reaches $1.9\times10^{-9}$.
BF16 produces its first nonfinite iterate at $k=37$.

\begin{figure}[t]
    \centering
    \includegraphics[width=0.98\linewidth]{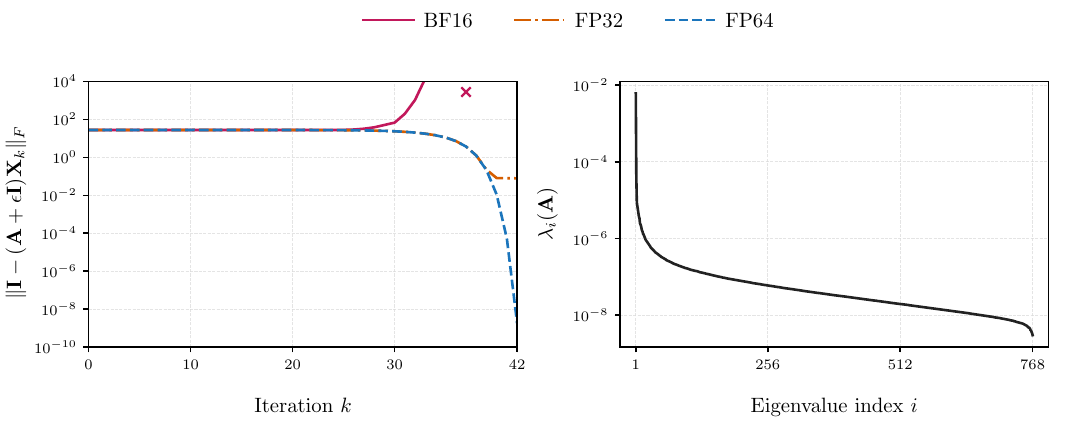}
    \caption{Newton--Schulz on the GPT-small Muon gradient-Gram EMA.
    Left: residual $\|\mI-(\mA+\varepsilon\mI)\mX_k\|_F$.
    Right: eigenvalues of $\mA$ in descending order.
    The cross marks the first nonfinite BF16 iterate ($k=37$).}
    \label{fig:newtonschulz}
\end{figure}

\paragraph{Ill-conditioned synthetic matrix.}
We construct a $768\times768$ matrix $\mA=\mQ\diag(\lambda_i)\mQ^\top$ in \text{FP64}, with eigenvalues logarithmically spaced from $1$ to $10^{-12}$ and an orthonormal basis obtained by QR decomposition of a Gaussian matrix.
The unregularized condition number is approximately $10^{12}$; with $\varepsilon=10^{-8}$, the shifted condition number is approximately $10^8$.
In~\Cref{fig:newtonschulz-synthetic}, BF16 and FP32 become nonfinite at iterations $35$ and $60$, respectively, whereas FP64 reaches a residual of approximately $2.1\times10^{-7}$ at iteration $65$.

\begin{figure}[t]
    \centering
    \includegraphics[width=0.98\linewidth]{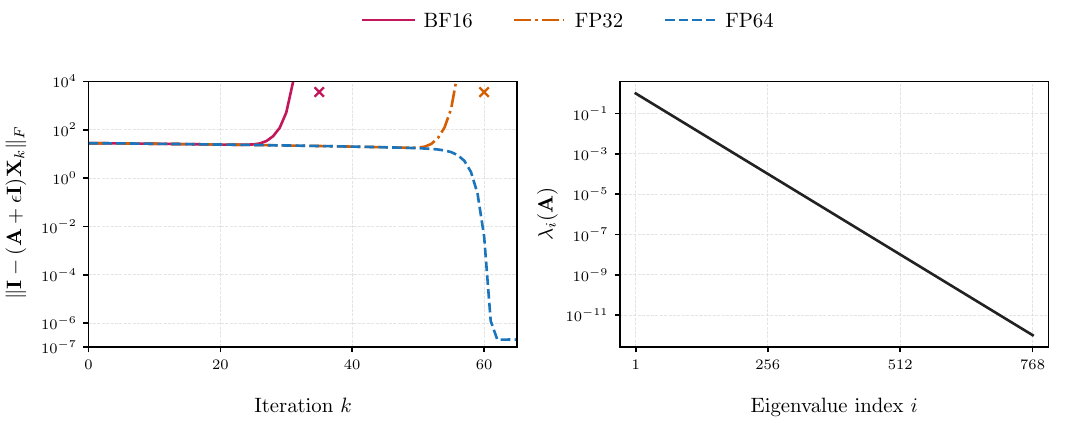}
    \caption{Newton--Schulz on the synthetic matrix.
    Left: residual $\|\mI-(\mA+\varepsilon\mI)\mX_k\|_F$.
    Right: eigenvalues of $\mA$ in descending order.
    Crosses mark the first nonfinite BF16 and FP32 iterates ($k=35$ and $k=60$).}
    \label{fig:newtonschulz-synthetic}
\end{figure}

These examples show that additional Newton--Schulz iterations need not improve the inverse approximation in finite precision. We also encountered divergent training runs when increasing the number of inversion steps. Reliable matrix inversion in low precision therefore remains a practical obstacle for the full MN method.

% \FloatBarrier

\section{Further MN method experiments}
\label{sec:further-mn-experiments}

\subsection{MND comparison protocol}
\label{app:mnd-comparison}

\paragraph{Training and selection.}
\Cref{fig:mn-d} uses GPT-small on FineWeb with effective batch 512,
microbatch 8, context 1024, BF16 autocast, compiled updates,
and highest FP32 matrix-multiplication precision.
The schedule has 10\% linear warmup, 80\% constant learning rate, and
10\% linear decay; weight decay is zero and the global gradient-norm clip
is 1. Each run has 4,743 optimizer updates, except KL-Shampoo,
whose declared token budget produces 4,744 updates. Final validation evaluates
all complete batches of the held-out split, as in \Cref{app:experimental-details}.
For each rate we run seeds 1, 2, and 42 to the full
budget and select the lowest mean final validation loss. The grids are
\begin{align*}
\text{MND}\text{ variants}:&\quad
 \{0.0001,0.0003,0.0006,0.001,0.003,0.006,0.01\},\\
\text{Distributed Shampoo}:&\quad
 \{0.00001,0.00003,0.0001,0.0003,0.0006,0.001,0.003\},\\
\text{KL-Shampoo}:&\quad
 \{0.0003,0.00045,0.0006,0.001,0.003,0.006,0.01,0.02\}.
\end{align*}
Training curves use a per-seed exponential moving average assigning
weight 0.05 to the current loss, followed by the seed mean and
minimum--maximum band. The horizontal axis shows cumulative recorded training hours, averaged across seeds at matching optimizer updates.

Each run uses one H100 GPU. GPT-small uses 64 gradient-accumulation
steps with microbatch 8 (effective batch 512), while GPT-medium uses
96 (effective batch 768). MND takes approximately 2.8 and 21 recorded
training hours, respectively.

\begin{table}[H]
\caption{Selected learning rates and final validation losses for
\Cref{fig:mn-d}. Losses are means $\mathbin{\pm}$ population standard
 deviations over seeds 1, 2, and 42.}
\label{tab:mnd-updated-results}
\centering
\small
\begin{tabular}{lcl}
\toprule
Method & Learning rate & Validation loss \\
\midrule
MND (uncentered) & 0.003 & $3.376\mathbin{\pm}0.00141$ \\
MND & 0.003 & $3.373\mathbin{\pm}0.00122$ \\
MND ($\beta_2$) & 0.003 & $3.409\mathbin{\pm}0.00280$ \\
KL-Shampoo & 0.003 & $3.444\mathbin{\pm}0.0105$ \\
Distributed Shampoo & 0.001 & $3.494\mathbin{\pm}0.00681$ \\
\bottomrule
\end{tabular}
\end{table}

\paragraph{Implemented diagonal recursion.}
All three variants use momentum decay $\beta=0.925$, $\varepsilon=10^{-9}$,
no shrinkage, and one row/column sweep. Momentum starts at zero and the
factor vectors at $\vone_m/m$ and $\vone_n/n$.
In the factor updates, centered and uncentered MND weight the
previous factors by $\beta$ and the sample terms by $\beta(1-\beta)$;
the $\beta_2=0.99$ variant instead uses $\beta_2$ and $1-\beta_2$.
Centered samples subtract the previous momentum; the other variants
use $\mX_t=\mG_t$.  

\subsection{MN comparison}

First we experiment with the full MN (Matrix Normal) method. Our focus here is on comparing the differences between~\Cref{alg:mn-method} and the KL-Shampoo method~\citep{lin2026klshampoo}. 

For both the matrix inverse and inverse square root in~\Cref{alg:mn-method} we use the Newton--Schulz method~\citep{highamfunctions}. The Newton--Schulz methods use polynomials of the matrix to approximate the (square root) inverse, and thus are GPU-friendly~\citep{amsel2025polar}.

To make the inversions more numerically stable, we apply the following shrinkage maps 
\begin{equation}
    \mathcal{P}_{\lambda}^{m}(\mA)
    =
    (1-\lambda)\mA
    +
    \lambda\frac{\tr(\mA)}{m}\mI_m,
    \qquad
    \mathcal{P}_{\lambda}^{n}(\mB)
    =
    (1-\lambda)\mB
    +
    \lambda\frac{\tr(\mB)}{n}\mI_n.
    \label{eq:mn-trace-preserving-shrinkage}
\end{equation}

Furthermore, to tackle the scale ambiguity in the factors $\hat{\mR}_t$ and $\hat{\mC}_t$, we will fix the scale of these factors through their trace
by applying the trace balancing operations
\begin{equation}
    \mathcal{B}(\mA,\mB)
    =
    \left(
    \sqrt{s_B/s_A}\;\mA,\;\;
    \sqrt{s_A/s_B}\;\mB
    \right),
    \quad
    s_A = \max(\tr(\mA),\varepsilon),\quad
    s_B = \max(\tr(\mB),\varepsilon).
    \label{eq:mn-transient-balance}
\end{equation}
This trace balancing prevents one factor, say $\hat{\mR}_t$, from gradually becoming very large while $\hat{\mC}_t$ becomes very small, which can cause floating-point issues.

Due to numerical instabilities of Newton--Schulz inversion, we replace the matrix products inside the trace terms of \eqref{eq:mn-R_update} and \eqref{eq:mn-C_update} by identity matrices, in particular $(\mR^+)^{-1}\hat{\mR}_{t-1}\approx\mI_m$, yielding the simplified factor updates in \Cref{alg:mn-method}.  

\begin{algorithm}[t]
\caption{The (approximate) MN method with full factors}
\label{alg:mn-method}
\begin{algorithmic}[1]
\Require Gradient matrix $\mG_t \in \R^{m \times n}$, step size $\eta > 0$,
momentum coefficient $\beta \in (0, 1)$,
$\lambda\in[0,1]$, and $\varepsilon>0$
\State Initialize $\hat{\mM}_0\gets0$,
$\hat{\mR}_0\gets\mI_m/m$, and $\hat{\mC}_0\gets\mI_n/n$
\For{$t=1,2,\ldots$}
    \State $\mX_t \gets \mG_t-\hat{\mM}_{t-1}$
    \Comment{centered residual}
    \State $\hat{\mM}_t \gets \beta\hat{\mM}_{t-1}+(1-\beta)\mG_t$
    \Comment{update momentum}
    \State $\mR^+ \gets
        \beta\hat{\mR}_{t-1}
        +
        \dfrac{\beta(1-\beta)}{n}\,
        \mX_t\,\mathcal{P}_{\lambda}^{n}(\hat{\mC}_{t-1})^{-1}\,\mX_t^\top$
    \State $\mC^+ \gets
        \beta\hat{\mC}_{t-1}
        +
        \dfrac{\beta(1-\beta)}{m}\,
        \mX_t^\top\,\mathcal{P}_{\lambda}^{m}(\mR^+)^{-1}\,\mX_t$
    \State $(\hat{\mR}_t,\hat{\mC}_t)\gets \mathcal{B}(\mR^+,\mC^+)$
    \Comment{balance factors \eqref{eq:mn-transient-balance}}
    \State $\Delta\mW_t \gets -\eta \mathcal{P}_{\lambda}^{m}(\hat{\mR}_t)^{-1/2}\,\hat{\mM}_t\,\mathcal{P}_{\lambda}^{n}(\hat{\mC}_t)^{-1/2}$
    \State $\mW_{t+1}\gets\mW_t+\Delta\mW_t$
\EndFor
\end{algorithmic}
\end{algorithm}

% \begin{minipage}{0.8\textwidth}
% \begin{algorithm}[H]
% \caption{The MN method  (Matrix Normal)}
% \label{alg:mn-methodsdsdsas}
% \begin{algorithmic}[1]
% \Require Gradient matrix $\mG_t \in \R^{m \times n}$, step size $\eta > 0$,
% momentum coefficient $\beta \in (0, 1)$
% \State Initialize $\hat{\mM}_0\gets0$,
% $\hat{\mR}_0\gets\mI_m/m$, and $\hat{\mC}_0\gets\mI_n/n$
% \For{$t=1,2,\ldots$}
%     \State $\mX_t \gets \mG_t-\hat{\mM}_{t-1}$
%     \State $\hat{\mM}_t \gets \beta\hat{\mM}_{t-1}+(1-\beta)\mG_t$
%     \Comment{update momentum}
%     \State $\hat{\mR}_{t}\gets
%         \beta\hat{\mR}_{t-1}
%         +
%         \dfrac{\beta(1-\beta)}{n}\,
%         \mX_t\,(\hat{\mC}_{t-1})^{-1}\,\mX_t^\top$
%     \State $\hat{\mC}_t \gets
%         \frac{\beta}{m}\tr\bigl((\hat{\mR}_t )^{-1}\hat{\mR}_{t-1}\bigr)\hat{\mC}_{t-1} 
%         +
%         \dfrac{\beta(1-\beta)}{m}\,
%         \mX_t^\top\,(\hat{\mR}_{t})^{-1}\,\mX_t$     \Comment{update factors}
%     \State $\Delta\mW_t \gets -\eta\, (\hat{\mR}_t)^{-1/2}\,\hat{\mM}_t\, (\hat{\mC}_t)^{-1/2}$  \Comment{Whitened Momentum update}
% \EndFor
% \end{algorithmic}
% % \end{algorithm}

%    \vspace{1cm}
% \end{minipage}

For all of our experiments in this section we train a 124M-parameter GPT-small model on FineWeb-10B~\citep{penedo2024fineweb}
for 2.5B tokens (about 20 tokens per parameter, Chinchilla-style). 
We use seeds 1, 2, and 42. For all three MN variants,
the learning rate grid is
\[
\{0.0001,0.0003,0.0006,0.001,0.003,0.006\}.
\]
For KL-Shampoo, it is
\[
\{0.0003,0.00045,0.0006,0.001,0.003,0.006,0.01,0.02\};
\]
for Distributed Shampoo, it is
\[
\{0.00001,0.00003,0.0001,0.0003,0.0006,0.001,0.003\}.
\]
We select each method's learning rate by its mean final validation loss
over the three seeds.

The main differences between our~\Cref{alg:mn-method} and the associated implementation\footnote{\url{https://github.com/yorkerlin/KL-Methods}} of KL-Shampoo~\citep{lin2026klshampoo} are: 1) using centered samples, 2) a shared covariance discount and the additional multiplicative factor $\beta$, and 3) a row-then-column sweep, shrinkage, Newton--Schulz inversion, and trace balancing.
We compare variants to gauge the importance of centering and the covariance discount. We use MN to denote~\Cref{alg:mn-method}, which contains all three of these differences. We use MN (uncentered) to denote the use of 2) and 3), that is,~\Cref{alg:mn-method} where $\mX_t = \mG_t.$  Finally, we use MN ($\beta_2$) to denote~\Cref{alg:mn-method} but where $\mX_t = \mG_t$, and replace $\beta$ in the factor prior terms by $\beta_2=0.99$ and $\beta(1-\beta)$ in the evidence terms by $1-\beta_2$.
That is, MN ($\beta_2$) uses KL-Shampoo-style factor updates, with a row-then-column sweep, shrinkage, trace balancing, and Newton--Schulz inversion.
See~\Cref{fig:mn-lr-sens} for the comparison between all of these variants, where we have also included the official implementation of Distributed Shampoo~\citep{shi2023distributedshampoo}  as a baseline.
Selected learning rates and corresponding mean final validation losses are: MN (uncentered) (0.0006, 3.458), MN (0.0006, 3.453), MN ($\beta_2$) (0.0006, 3.455), KL-Shampoo (0.003, 3.444), and Distributed Shampoo (0.001, 3.494).

On this learning rate sweep, each MN variant reaches a lower minimum validation loss as compared to Distributed Shampoo and has a wider basin of competitive learning rates.  But  KL-Shampoo achieves the lowest final validation loss.  All variants also remain above the tuned AdamW mean of 3.381 in~\Cref{fig:mn-muon-competitive}. 
Centered MN achieves the lowest mean validation loss
among the MN variants (3.453).
For uncentered gradients, using a shared $\beta$ gives a slightly
higher mean loss (3.458) than using a separate $\beta_2$ (3.455).

% Some full-matrix runs were numerically unstable under Newton--Schulz inversion, so the remaining experiments use diagonal factors.

\begin{figure}
    \centering
    \includegraphics[width=0.98\linewidth]{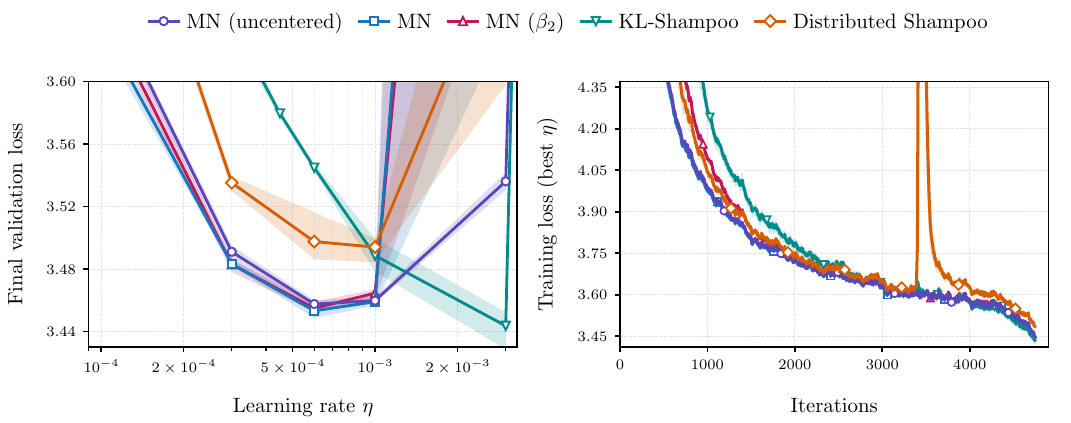}
    \caption{Comparing three full-factor MN variants with KL-Shampoo and Distributed Shampoo on GPT-small for 2.5B tokens. Left: final validation loss versus learning rate. Right: training loss at the selected learning rates, smoothed separately per seed by an EMA with weight 0.05 on the current loss. Lines show means over three seeds; shading: seed ranges. }
    \label{fig:mn-lr-sens}
\end{figure}

% \FloatBarrier
\section{Further MND-Muon experiments}
\label{app:further-mn-muon-experiments}

This appendix supplements the experiments in~\Cref{sec:mn-muon-experiments}
with training and optimizer settings, learning rate selection, and complete
numerical results. All reported MND-Muon experiments use the diagonal direction method
in~\Cref{alg:mn-muon-d-method}.

\subsection{Experimental setup}
\label{app:experimental-details}

\paragraph{Baselines.}
For GPT-small, we compare MND-Muon with Muon,
NorMuon, AdamW, Distributed Shampoo, and KL-Shampoo
on FineWeb at batch 512 (\Cref{fig:mn-muon-competitive}).
AdamW uses $(\beta_1,\beta_2)=(0.95,0.95)$;
Distributed Shampoo uses $(0.9,0.999)$, Adam grafting, and
preconditioning every 100 steps; KL-Shampoo uses $(0.9,0.98)$,
preconditioning every 10 steps, and no gradient normalization.
Muon uses Nesterov momentum with coefficient 0.95.
For AdamW, Distributed Shampoo, and KL-Shampoo, we sweep peak learning rates
\[
\{0.00001,\,0.00003,\,0.0001,\,0.0003,\,0.0006,\,0.001,\,0.003,\,0.006,\,0.01\};
\]
the learning rate grid for Muon, NorMuon, and
MND-Muon is given in~\Cref{app:mn-muon-final-comparison}.
For each method, we select the learning rate with the lowest validation loss after
1B tokens using seed 42, then train fresh runs for 2.5B tokens with seeds
42, 1, and 2 at that fixed learning rate. The right panel shows the mean
training curve and pointwise minimum and maximum across these three seeds,
after smoothing each seed with a 25M-token exponential moving average.
The selected rates for MND-Muon, Muon, and NorMuon
are all 0.03; those for AdamW, Distributed Shampoo, and
KL-Shampoo are 0.003, 0.0006, and 0.003, respectively.
The comparison across corpora and batch sizes includes MND-Muon,
Muon, and NorMuon (\Cref{fig:mn-muon-batch-size}).
Paired runs share initialization, token order, training budget,
schedule, numerical precision, and validation data, with learning rates
selected separately for each method.

\paragraph{Model.}
GPT-small has 12 layers, width 768, 12 attention heads, and 123,551,232
active parameters. Each layer uses width-2048 SwiGLU blocks, RMSNorm,
rotary-style positional transforms, and PyTorch scaled-dot-product attention.
The model has tied input/output embeddings, bias-free linear maps, and no
dropout. Linear and embedding weights are initialized from centered
Gaussians with standard deviation 0.02; residual projections use
$0.02/\sqrt{2L}$, where $L=12$. Normalization scales start at one.
The context length is 1024 and the vocabulary size is 50,257.

\paragraph{Data and loss.}
We use FineWeb-10B~\citep{penedo2024fineweb} and a sequential 10.1B-token
prefix of SlimPajama-627B~\citep{cerebras2023slimpajama}, denoted SlimPajama.
We tokenize documents with GPT-2, prepend an end-of-text token to
each document, and concatenate them in dataset order. Training
batches contain consecutive 1024-token sequences from this stream.
We reserve the first 100M tokens for
validation and evaluate final loss over all complete batches in this split.
Loss is cross-entropy in nats per token.

\begin{table}
\caption{GPT-small training settings. Batch-dependent entries follow the order 32, 64, 128, 512.}
\label{tab:mn-muon-training-config}
\centering
\small
\begin{tabular}{ll}
\toprule
Setting & Value \\
\midrule
Microbatch per GPU & 8 sequences \\
Effective batch & 32, 64, 128, 512 sequences \\
Accumulation steps & 4, 8, 16, 64 \\
Tokens per optimizer step & 32,768; 65,536; 131,072; or 524,288 \\
Learning rate selection budget & 1,000,341,504 tokens \\
Final training budget & 2,486,697,984 tokens \\
Final optimizer steps & 75,888; 37,944; 18,972; or 4,743 \\
Learning rate schedule & linear warmup (10\%), constant (80\%), linear decay (10\%) \\
Weight decay & 0 \\
Global gradient-norm clip & 1 \\
Implementation & BF16 autocast; \text{torch.compile} \\
FP32 matmul precision & \text{highest} \\
\bottomrule
\end{tabular}
\end{table}

Each paired GPT-small run uses one GPU. The Muon baselines use
80\,GB NVIDIA A100 SXM4 or H100 PCIe GPUs; the new MND-Muon
final runs use H100 or H200 GPUs. \Cref{tab:mn-muon-training-config} gives the common training settings.

\paragraph{Optimizers.}
MND-Muon and Muon use polar updates for hidden-layer matrices, with the settings
in~\Cref{tab:mn-muon-optimizer-config}. MND-Muon updates the row
factor once, then the column factor once, holding the previous direction fixed.
The tied embedding/output matrix and nonmatrix parameters use a
bias-corrected AdamW-style update with
$(\beta_1,\beta_2)=(0.95,0.95)$, $\varepsilon=10^{-8}$, and the same
base learning rate and schedule as the matrix updates.

NorMuon uses the authors' implementation with Nesterov momentum
0.95, five BF16 Newton--Schulz iterations, and a row mean-square EMA
with coefficient 0.95. It normalizes rows of the orthogonalized update,
restores its Frobenius norm, and applies the implementation's shape scaling.
The Newton--Schulz and row-normalization stabilizers are $10^{-7}$ and
$10^{-10}$, respectively. Its auxiliary Adam updates use the same
$(0.95,0.95)$ coefficients, $10^{-8}$ stabilizer, and swept learning rate
as the other methods.  

\begin{table}
\caption{Optimizer settings used for both model sizes.}
\label{tab:mn-muon-optimizer-config}
\centering
\small
\begin{tabular}{lll}
\toprule
Setting & MND-Muon & Muon \\
\midrule
Matrix momentum & whitened-gradient EMA, 0.9 & Nesterov, 0.95 \\
Factor discount & 0.9 & --- \\
Covariance factors & diagonal, direction-dependent & --- \\
Factor updates per step & one row, then one column & --- \\
Factor inverse floor & $\varepsilon=10^{-9}$ & --- \\
Polar Express iterations & 5 & 5 \\
\bottomrule
\end{tabular}
\end{table}

\subsection{GPT-small: learning rates and detailed results}
\label{app:mn-muon-final-comparison}

We compare the methods on FineWeb and SlimPajama at effective batch sizes
32, 64, 128, and 512. For each method, corpus, and batch size, we select
the learning rate with the lowest validation loss in the 1B-token sweep
with seed 42. The learning rate grid is
\[
\{0.0003,\,0.001,\,0.003,\,0.006,\,0.01,\,0.02,\,0.03,\,0.06,\,0.1\}.
\]
We then train from initialization for 2.5B tokens at the selected
learning rate using seeds 42, 1, and 2.
\Cref{tab:mn-muon-selected-lrs} lists the selected learning rates.

\paragraph{FineWeb training curves.}
Paired runs share initialization and token order.
At batch 64, mean final validation losses are 3.368 for MND-Muon,
3.362 for NorMuon, and 3.409 for Muon; at batch 128,
they are 3.338, 3.340, and 3.354, respectively.
For \Cref{fig:mn-muon-small-batch-curves}, we smooth each seed's training
loss with a 25M-token exponential moving average before computing the
pointwise mean and minimum--maximum band. The horizontal axis counts
optimizer updates: 37,944 at batch 64 and 18,972 at batch 128.
\Cref{tab:mn-muon-small-batch-results} gives the independently selected
learning rates and final validation losses for all three plotted methods.
The bands describe training-loss variation; the table reports sample
standard deviations of the final validation losses.

\begin{figure}
\centering
\includegraphics[width=0.98\linewidth]{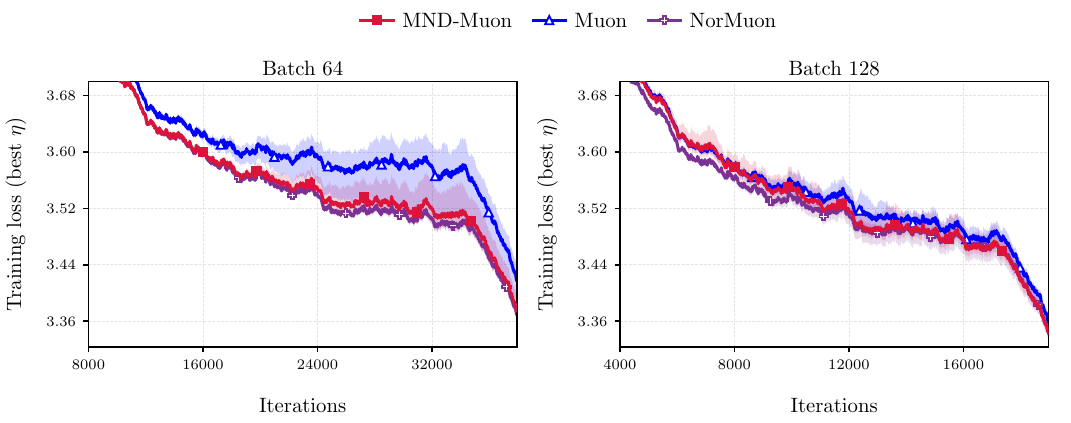}
\caption{FineWeb training loss versus iteration through 2.5B tokens.
Lines: three-seed means; shading: seed ranges.
Protocol and results: \Cref{app:mn-muon-final-comparison}.}
\label{fig:mn-muon-small-batch-curves}
\end{figure}

\begin{table}[H]
\caption{FineWeb results corresponding to \Cref{fig:mn-muon-small-batch-curves}.
Validation loss is the mean $\mathbin{\pm}$ sample standard deviation over
seeds 1, 2, and 42 after 2.5B tokens. Bold marks the lowest unrounded mean loss for each batch size.}
\label{tab:mn-muon-small-batch-results}
\centering
\small
\begin{tabular}{llcl}
\toprule
Batch & Method & Learning rate & Validation loss \\
\midrule
64 & MND-Muon & 0.01 & $3.368\mathbin{\pm}0.0157$ \\
64 & Muon & 0.01 & $3.409\mathbin{\pm}0.0285$ \\
64 & NorMuon & 0.006 & $\bm{3.362\mathbin{\pm}0.00592}$ \\
\midrule
128 & MND-Muon & 0.03 & $\bm{3.338\mathbin{\pm}0.0105}$ \\
128 & Muon & 0.02 & $3.354\mathbin{\pm}0.0104$ \\
128 & NorMuon & 0.02 & $3.340\mathbin{\pm}0.00932$ \\
\bottomrule
\end{tabular}
\end{table}

\begin{table}[H]
\caption{Learning rates selected using seed 42; columns give effective batch size.}
\label{tab:mn-muon-selected-lrs}
\centering
\small
\begin{tabular}{llrrrr}
\toprule
Dataset & Method & 32 & 64 & 128 & 512 \\
\midrule
FineWeb & MND-Muon & 0.03 & 0.01 & 0.03 & 0.03 \\
FineWeb & Muon & 0.03 & 0.01 & 0.02 & 0.03 \\
FineWeb & NorMuon & 0.01 & 0.006 & 0.02 & 0.03 \\
SlimPajama & MND-Muon & 0.02 & 0.02 & 0.02 & 0.03 \\
SlimPajama & Muon & 0.02 & 0.02 & 0.01 & 0.02 \\
SlimPajama & NorMuon & 0.006 & 0.01 & 0.01 & 0.02 \\
\bottomrule
\end{tabular}
\end{table}

\paragraph{Validation results.}
\Cref{fig:mn-muon-batch-size} compares all three methods across batch sizes,
showing mean validation losses and the full seed ranges.
MND-Muon and NorMuon each outperform Muon in all
eight corpus--batch means, and in 21 and 22 of the 24 seed pairs, respectively.
MND-Muon has the lowest mean at FineWeb batch 128 and all four
SlimPajama batches; NorMuon has the lowest mean at FineWeb batches
32, 64, and 512.
\Cref{tab:mn-muon-batch-size} reports mean final validation losses and
sample standard deviations for MND-Muon, Muon, and
NorMuon over seeds 1, 2, and 42. All 72 completed endpoints
are retained, including MND-Muon's FineWeb batch-32 seed-42
loss of 5.565 and Muon's seed-1 loss of 5.967.
At this batch, NorMuon attains $3.49238\mathbin{\pm}0.02290$.
Its only seed-level losses higher than Muon's are FineWeb
batch 32, seed 2, and SlimPajama batch 64, seed 1.

\begin{figure}
\centering
\includegraphics[width=0.98\linewidth]{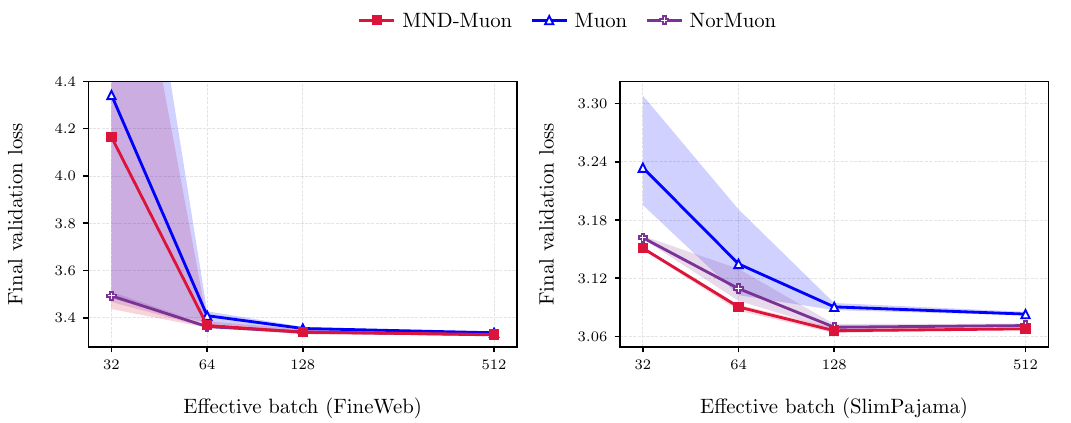}
\caption{Final validation loss versus effective batch size after 2.5B
 tokens on FineWeb (left) and SlimPajama (right). Lines show means over
 seeds 1, 2, and 42; shading spans their minimum and maximum. }
\label{fig:mn-muon-batch-size}
\end{figure}

\begin{table}
\caption{Final validation loss after 2.5B tokens. Entries are means
$\mathbin{\pm}$ sample standard deviations over seeds 1, 2, and 42.
Bold marks the lowest unrounded mean loss in each row.}
\label{tab:mn-muon-batch-size}
\label{tab:mn-muon-batch512}
\centering
\small
\begin{tabular}{lllll}
\toprule
Dataset & Batch & MND-Muon & Muon & NorMuon \\
\midrule
FineWeb & 32 & $4.165\mathbin{\pm}1.21$ & $4.343\mathbin{\pm}1.41$ & $\bm{3.492\mathbin{\pm}0.0229}$ \\
FineWeb & 64 & $3.368\mathbin{\pm}0.0157$ & $3.409\mathbin{\pm}0.0285$ & $\bm{3.362\mathbin{\pm}0.00592}$ \\
FineWeb & 128 & $\bm{3.338\mathbin{\pm}0.0105}$ & $3.354\mathbin{\pm}0.0104$ & $3.340\mathbin{\pm}0.00932$ \\
FineWeb & 512 & $3.329\mathbin{\pm}0.000796$ & $3.336\mathbin{\pm}0.000471$ & $\bm{3.327\mathbin{\pm}0.00149}$ \\
\midrule
SlimPajama & 32 & $\bm{3.151\mathbin{\pm}0.00150}$ & $3.234\mathbin{\pm}0.0643$ & $3.162\mathbin{\pm}0.00162$ \\
SlimPajama & 64 & $\bm{3.090\mathbin{\pm}0.00343}$ & $3.135\mathbin{\pm}0.0487$ & $3.109\mathbin{\pm}0.0177$ \\
SlimPajama & 128 & $\bm{3.066\mathbin{\pm}0.00278}$ & $3.090\mathbin{\pm}0.00369$ & $3.070\mathbin{\pm}0.00306$ \\
SlimPajama & 512 & $\bm{3.068\mathbin{\pm}0.000238}$ & $3.083\mathbin{\pm}0.00209$ & $3.071\mathbin{\pm}0.00221$ \\
\bottomrule
\end{tabular}
\end{table}

\subsection{GPT-medium: setup and full learning rate sweep}
\label{app:mn-muon-medium}

We use the same optimizer settings as for GPT-small
(\Cref{tab:mn-muon-optimizer-config}) and test six learning rates per
method, including NorMuon. All eighteen runs use seed 42 and share
initialization and token order.

\paragraph{Model.}
GPT-medium has 24 layers, width 1024, 16 attention heads, width-2730
SwiGLU blocks, and 353,454,080 active parameters. Other architecture and
initialization settings follow~\Cref{app:experimental-details}.

\paragraph{Data and loss.}
We use the FineWeb-10B data and validation split described in
\Cref{app:experimental-details}. The split is made in token order;
one document spans the training and validation shards.

\begin{table}
\caption{Training settings for the GPT-medium sweep.}
\label{tab:mn-muon-medium-config}
\centering
\small
\begin{tabular}{ll}
\toprule
Setting & Value \\
\midrule
Effective batch & 768 sequences of 1024 tokens \\
Tokens per optimizer update & 786,432 \\
Training budget & 7,090,470,912 tokens; 9,016 updates \\
Schedule & linear warmup, constant, linear decay toward zero \\
Warmup and decay & 901 updates each \\
Weight decay & 0 \\
Global gradient-norm clip & 1 \\
Numerical precision & BF16 autocast; FP32 parameters and state \\
Implementation & \text{torch.compile}; FP32 matmul precision \text{highest} \\
Seed & 42 \\
Validation & one evaluation at the end of each run \\
\bottomrule
\end{tabular}
\end{table}

\Cref{tab:mn-muon-medium-config} gives the training settings.
All runs use single-node DDP with 32 accumulation steps per update.

\paragraph{Learning rate selection.}
We test peak learning rates
\[
\{0.012,\,0.02,\,0.03,\,0.035,\,0.04,\,0.045\}
\]
for Muon and
\[
\{0.012,\,0.02,\,0.03,\,0.045,\,0.06,\,0.09\}
\]
for MND-Muon and NorMuon.
Each learning rate is tested for the full 7.09B-token budget with seed 42.

\paragraph{Complete results.}
\Cref{tab:mn-muon-medium-results} reports all eighteen final validation
losses. The best losses are 3.013852 for MND-Muon at 0.045,
3.017372 for Muon at 0.04, and 3.007364 for NorMuon at 0.045.
Muon's loss of 6.295429 at 0.045 and NorMuon's loss of
6.070581 at 0.09 exceed the displayed validation loss range.
\Cref{fig:mn-muon-medium} uses each method's best tested learning rate;
the stability comparison uses 0.09 for MND-Muon and NorMuon,
and 0.045 for Muon, its largest completed rate.

\begin{table}[htbp]
\caption{Final GPT-medium validation loss after 7.09B training
tokens, seed 42. Bold entries mark the best rate for each method before rounding; a dash denotes an untested combination.}
\label{tab:mn-muon-medium-results}
\centering
\small
\begin{tabular}{clll}
\toprule
Learning rate & MND-Muon & Muon & NorMuon \\
\midrule
0.012 & 3.026 & 3.033 & 3.015 \\
0.020 & 3.016 & 3.024 & 3.011 \\
0.030 & 3.015 & 3.018 & 3.008 \\
0.035 & --- & 3.018 & --- \\
0.040 & --- & \textbf{3.017} & --- \\
0.045 & \textbf{3.014} & 6.295 & \textbf{3.007} \\
0.060 & 3.016 & --- & 3.010 \\
0.090 & 3.024 & --- & 6.071 \\
\bottomrule
\end{tabular}
\end{table}

\Cref{fig:mn-muon-medium-stability} shows the validation loss sweep
and training at each method's largest tested rate. MND-Muon
remains stable at 0.09; Muon at 0.045 and NorMuon at 0.09
develop large loss spikes.

\begin{figure}[H]
\centering
\includegraphics[width=0.98\linewidth]{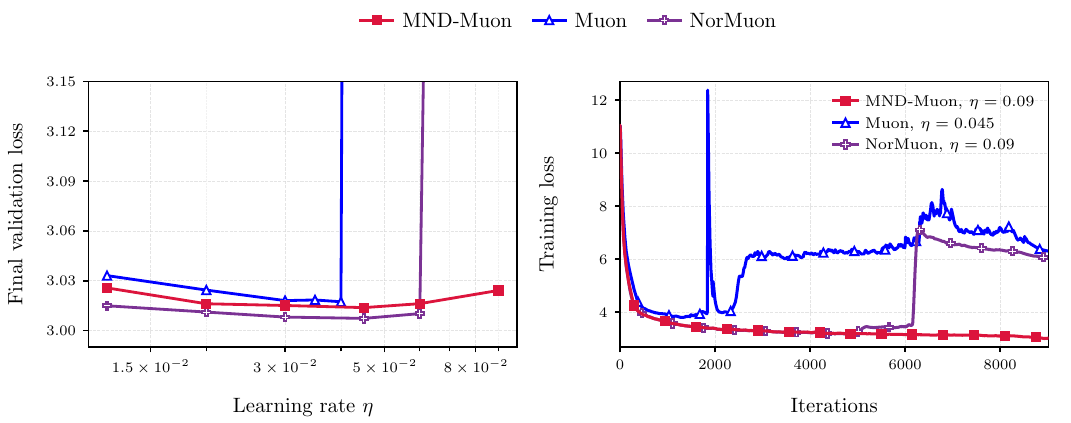}
\caption{GPT-medium on FineWeb, batch 768, seed 42.
Left: final validation loss after 7.09B training tokens versus peak learning
rate.
Right: training loss versus iteration at each method's largest tested rate,
smoothed with a 25M-token exponential moving average. MND-Muon
at 0.09 remains stable (final loss 3.024); Muon at 0.045 and
NorMuon at 0.09 develop large loss spikes (6.295 and 6.071).}
\label{fig:mn-muon-medium-stability}
\end{figure}

\end{document}